\documentclass{article}

\usepackage[main, final]{neurips_2026}

\usepackage[utf8]{inputenc} % allow utf-8 input
\usepackage[T1]{fontenc}    % use 8-bit T1 fonts
\usepackage{hyperref}       % hyperlinks
\usepackage{url}            % simple URL typesetting
\usepackage{booktabs}       % professional-quality tables
\usepackage{amsfonts}       % blackboard math symbols
\usepackage{nicefrac}       % compact symbols for 1/2, etc.
\usepackage{microtype}      % microtypography
\usepackage{xcolor}         % colors

\usepackage{algorithm}
\usepackage{algpseudocode}
\usepackage{algorithmicx}
\usepackage{enumitem}
\usepackage{multirow}
\usepackage{amsmath}
\usepackage{amssymb}
\usepackage{mathtools}
\usepackage{amsthm}
\usepackage[table]{xcolor}

\usepackage[most]{tcolorbox}

\theoremstyle{plain}
\newtheorem{theorem}{Theorem}[section]
\newtheorem{proposition}[theorem]{Proposition}
\newtheorem{lemma}[theorem]{Lemma}
\newtheorem{corollary}[theorem]{Corollary}
\newtheorem{definition}[theorem]{Definition}
\newtheorem{assumption}[theorem]{Assumption}
\newtheorem{remark}[theorem]{Remark}

\renewcommand{\Pr}[1]{\mathbb{P}\left(#1\right)}

\definecolor{navyblue}{RGB}{10, 10, 200}
\definecolor{bloodred}{RGB}{200, 10, 10}
\definecolor{darkgreen}{RGB}{10, 200, 10}
\newcommand{\red}[1]{\textcolor{bloodred}{#1}}
\newcommand{\blue}[1]{\textcolor{navyblue}{#1}}
\newcommand{\green}[1]{\textcolor{darkgreen}{#1}}

\definecolor{myred}{HTML}{F8766D}
\definecolor{mygreen}{HTML}{7CAE00}
\definecolor{mycyan}{HTML}{00BFC4}
\definecolor{mypurple}{HTML}{C77CFF}
\definecolor{mycolor}{HTML}{FF7AD0} %FF9E2F
\definecolor{myblack}{HTML}{7A7A7A}
\definecolor{myyellow}{HTML}{FFD700}

\usepackage{mysymbol}

\title{\textsc{Graph Cascades}: Contagion-Based Mesoscopic Rewiring for Structure-Aware Graph Machine Learning}

\author{
  Meher Chaitanya\thanks{Equal contribution.}\\
  KTH Royal Institute of Technology\\
  Stockholm, Sweden\\
  \texttt{mcpi@kth.se}\\
  \And
  My Le$^*$\\
  Johns Hopkins University\\
  Baltimore, MD\\
  \texttt{mle19@jh.edu}\\
  \And
  Luana Ruiz\\
  Johns Hopkins University\\
  Baltimore, MD\\
  \texttt{lrubini1@jh.edu}\\
}

\begin{document}

\maketitle

\begin{abstract} 
We introduce \textsc{Graph Cascades}, a mesoscopic rewiring strategy for Graph Neural Networks (GNNs) and Graph Transformers (GTs) that captures intermediate-scale graph structure beyond purely local edges or fully global attention. Using contagion-based diffusion with fixed cascade parameters, \textsc{Graph Cascades} constructs, in $O(|V|+|E|)$ time, a sparse auxiliary graph where node pairs supported by repeated multi-hop reinforcement are promoted to direct neighbors. We theoretically characterize when reinforcement-based rewiring helps: sufficient conditions under which  reinforcement-based edge selection is more label-aligned than direct adjacency, a stochastic block model (SBM) witness in which two-hop reinforcement is perfectly homophilic, and a formalization of mesoscopic connectivity via graph effective resistance. Empirically, across node-classification benchmarks, \textsc{Graph Cascades} improves multiple GNN and sparse-GT backbones, with the most reliable gains observed on heterophilic and moderate- to high-degree homophilic graphs. The theoretical conditions also identify regimes where mesoscopic rewiring is unlikely to be beneficial---low-degree regular graphs and graphs with structural bottlenecks---and these predictions match the observed failures. We additionally observe tight correlations between performance and structural properties in the rewired graphs, suggesting that rewired graphs carry predictive signals. %Building on these foundations, we further introduce CR-Graphormer, a cascade-rewired GT leveraging the auxiliary graph to generate sparse token sequences suitable for efficient training.
\end{abstract}
\section{Introduction}
Graph-structured data arises across many domains, with applications including anomaly detection, traffic forecasting, and recommendation \citep{deng2021graph,kong2024spatio,agrawal2024no}. A common strategy is to learn directly on the graph using Graph Neural Networks (GNNs), which capture local structure by propagating information along edges \citep{hamilton2017inductive}. This locality is a strength, but repeated aggregation can blur signals and compress long paths, leading to oversmoothing and oversquashing \citep{chen2020measuring,alon2020bottleneck}. Graph Transformers (GTs), motivated by the success of transformers in language modeling, offer an alternative \citep{dwivedi2020generalization,mialon2021graphit}. GTs treat nodes as tokens under global self-attention, but since graphs lack a canonical ordering, attention alone cannot distinguish adjacent from multi-hop neighbors. GTs therefore concatenate positional or structural encodings (e.g., Laplacian eigenvectors, shortest-path distances) to each node, which has been central to their empirical gains \citep{kreuzer2021rethinking,zhao2021gophormer,ying2021transformers}.

Still, a key challenge is structural resolution: distinguishing and preserving task-relevant neighborhoods and motifs. Dense attention can overshadow these patterns, overweight weak node pairs, and incur prohibitive cost \citep{ying2021transformers, chen2022structure, chennagphormer, zhou2025tokenphormer}. Sparse-attention GTs improve scalability by restricting attention to compact token sequences \citep{zhou2025tokenphormer}, but existing variants sacrifice granularity through coarse neighborhood aggregation or hand-crafted shortcuts \citep{chennagphormer, fu2024vcr}, an issue that is pronounced on heterophilic graphs, where same-label nodes are sparsely connected and multiple hops apart \citep{zhou2025tokenphormer}. To address these challenges, we introduce \textsc{Graph Cascades}, a mesoscopic rewiring strategy that captures \emph{intermediate-scale graph structure} via contagion dynamics \citep{granovetter1978threshold, centola_and_macy, centola-book}. Short threshold cascades are initiated at each node to reveal pairs that co-activate; these pairs are then promoted to direct connections in a sparse, weighted auxiliary graph that can be used either as a drop-in input for GNN and GT backbones or to instantiate \textsc{CR-Graphormer} (Section~\ref{sec:activation-rewire}), a lightweight standalone sparse transformer that directly acts on cascade-rewired neighborhoods.

\textsc{Graph Cascades} is grounded in a concrete structural mechanism: it is expected to help when task-relevant relationships are expressed through reinforced multi-hop neighborhoods rather than primarily through direct edges. This setting arises naturally in heterophilic graphs, where nodes sharing a label need not be adjacent, yet may be connected through multiple short, redundant paths. In contrast, on low-degree regular graphs or graphs with pronounced bottlenecks, such redundant support may be weak or absent, making contagion-based rewiring neutral or even detrimental. Accordingly, our evaluation goes beyond predictive accuracy: we also examine the structural regimes in which the theory predicts gains or failures, treating this alignment between mechanism and performance as a central part of the contribution.

Theoretically, we provide a mechanistic account of why reinforcement-based rewiring can benefit heterophilic networks. We establish sufficient conditions under which reinforcement events exhibit higher label agreement than direct adjacency (Proposition~\ref{prop:bayes-monotone-matrix}); identify an SBM regime in which two-hop profile reinforcement recovers same-class neighbors (Theorem~\ref{thm:sbm-profile-homophily}; and more generally multi-hop reinforcement in Appendix~\ref{app:sbm-raw-l-hop}) and characterize threshold activations through effective resistance (Theorem~\ref{thm:walklength}). These conditions also delineate regimes in which rewiring is unlikely to help, such as low-degree regular graphs and graphs with structural bottlenecks, and these predictions align with the observed failure cases. Empirically, cascade rewiring increases label homophily (Figure~\ref{fig:homo}) and improves classification performance across heterophilic and moderate- to high-degree homophilic graphs (Tables~\ref{tab:summary}--\ref{tab:numerical2}, Figure~\ref{fig:summary}), while structural properties of the rewired graph correlate strongly with downstream performance (Table~\ref{tab:correlation}).

\section{Related Work}\label{related_work_main}
\textit{\textbf{Graph Neural Networks and Transformers.}}
Our main experiments evaluate GCN, GraphGPS, NAGphormer, and VCR-Graphormer with and without cascade rewiring.
GraphGPS \citep{rampavsek2022recipe} interleaves message passing with global attention and replaces quadratic attention with linear mechanisms \citep{choromanski2021rethinking}.
NAGphormer \citep{chennagphormer} constructs fixed-length hop-aggregate token lists and performs localized self-attention within each list, improving mini-batch efficiency and neighborhood encoding.
VCR-Graphormer \citep{fu2024vcr} explicitly targets heterophily by introducing virtual connections and PPR-based token lists that encode both local and global dependencies while maintaining linear complexity per layer.
The broader comparison also includes GraphSAGE and GAT, together with GT \citep{dwivedi2020generalization}, Gophormer \citep{zhao2021gophormer}, Graphormer \citep{ying2021transformers}, SAN \citep{kreuzer2021rethinking}, Specformer \citep{bospecformer}, Exphormer \citep{shirzad2023exphormer}, SGFormer \citep{wu2023sgformer}, and PolyFormer \citep{ma2024polyformer}.

\textit{\textbf{Diffusion and Graph Rewiring.}}
PPRGo \citep{bojchevski2020scaling} is a scalable GNN that decouples local feature transformation from graph propagation and then propagates predictions using precomputed sparse approximations of personalized PageRank (PPR), avoiding expensive multi-hop message passing or per-step power iteration while keeping long-range information.
GRAND+ \citep{feng2022grand+} is a scalable version of GRAND that precomputes a (generalized) mixed-order propagation matrix using a generalized forward push (GFPush) algorithm, so it can do GRAND-style random propagation graph augmentations in mini-batches, and it adds a confidence-aware consistency loss to improve generalization.
Our mesoscopic rewiring conceptually falls under the same ``rewiring'' umbrella as diffusion-based methods (e.g., PPRGo/GRAND+), but it is explicit topology construction rather than an implicit propagation operator.
\textsc{Graph Cascades} is purely structural and label-free, like \textsc{FoSR} \citep{karhadkar2023fosr} and \textsc{SDRF} \citep{topping2022understanding}, but its objective is neighborhood realignment via reinforced multi-hop co-activation rather than bottleneck mitigation.
Additional rewiring comparisons and an extended discussion of \textsc{FoSR}, \textsc{SDRF}, and \textsc{GDC} are provided in Appendix~\ref{related_work_app}.

\section{Preliminary Definitions} \label{sec:contagion}
Let $G=(V,E)$ be an undirected, unweighted graph with $n=|V|$ nodes and $m=|E|$ edges. For $u\in V$, denote $N(u)=\{v\in V : (u,v)\in E\}$, $N[u]=N(u)\cup\{u\}$, and $d(u)=\lvert N(u)\rvert$. A \emph{contagion} or \emph{cascade} on \(G\) is a discrete-time process \(\{\mathbf{x}_t\}_{t=0}^T\), where \(\mathbf{x}_t \in \{0,1\}^n\) represents the instantaneous activation state at time \(t\), with \([\mathbf{x}_t]_i = 1\) if node \(i\) is activated at time step \(t\) and 0 otherwise. The \emph{active set} of a contagion at time \(t\) is defined as \(S_t = \{u\in V : [\mathbf x_{t'}]_u=1 \text{ for some } 0\le t'\le t\}\), storing all nodes activated in the contagion by time \(t\). Given an inactive node \(u\notin S_t\), its instantaneous \emph{active support} is the number of active nodes in its one-hop neighborhood, denoted \(\sigma_t(u)=\lvert N(u)\cap S_t\rvert\). 

In this paper, we leverage contagion dynamics to construct a mesoscopic rewiring by initiating a contagion with walk length $\ell$ from a neighborhood subset of each node \(v \in V\). Each contagion evolves according to one of two discrete-time activation rules as defined in \cite{chaitanyaadjacency}:

\textit{\textbf{Maximum Adjacency Search (MAS).}} For each node \(v \in V\), select a subset \(R \subseteq N(v)\) and initialize the cascade with active set \(S^v_0 = \{v\} \cup R\). At each step \(t \in [0,\ell-1]\), activate the inactive node with the strongest support
\(u_t^{\star}=\operatorname*{argmax}_{u\in V\setminus S_t^v}\sigma_t^v(u)\) and set \(S^{v}_{t+1}=S^v_{t}\cup\{u^{\star}_{t}\}\). The quantity \(\tau_t=\max_{u\in V\setminus S_t^v} \sigma_t^v(u)\) is called the activation threshold at time \(t\).

\textit{\textbf{Threshold Adjacency Search (TAS).}} Fix a threshold \(\tau \ge 1\). For each node \(v \in V\), select a subset \(R \subseteq N(v)\) and set \(S^v_0 = \{v\} \cup R\). At each step \(t \in [0, \ell-1]\), let
\(A_{t} = \bigl\{ u \in V \setminus S^v_t : \sigma^v_t(u) \ge \tau \bigr\}\)
be the set of inactive nodes with support at least \(\tau\). If \(A_{t}\neq\varnothing\), select one \(u^{\star}_{t}\in A_{t}\) (e.g., uniformly at random) and set \(S^{v}_{t+1}=S^v_{t}\cup\{u^{\star}_{t}\}\). If \(A_{t}=\varnothing\), the contagion process terminates. When $\tau = 1$ and activated nodes are processed in queue order, the process reduces to breadth-first search; larger \(\tau\) restrict activation to nodes with multi-neighbor reinforcement \citep{granovetter1978threshold, centola-book}.

Both MAS and TAS define higher-order activation processes. The next activation depends only on the current active set, and the walk length $\ell$ sets the spatial scale of propagation. MAS activation favors nodes with the strongest connectivity to the current active set, leading to early activation of densely connected neighborhoods, while nodes with weak ties are activated only as the threshold decreases. Thus, MAS tends to exhaustively explore cohesive regions before crossing weak cuts, causing the contagion to propagate preferentially within clusters.
TAS relaxes this behavior by allowing nodes with lower instantaneous support to activate earlier, promoting a more balanced graph exploration that is less constrained by local density. Further preliminaries on GNNs, GTs, and sparse variants are provided in Appendix~\ref{app:prelim}.

\section{\textsc{Graph Cascades}: Mesoscopic Rewiring}
\label{sec:activation-rewire}
We construct the mesoscopic rewiring by leveraging the activation sets generated by the contagion dynamics in Section~\ref{sec:contagion}. For each seed \(v \in V\), we run multiple contagions of fixed length \(\ell\), each initialized at \(v\) with a random subset \(R \subseteq N(v)\). Each realization yields an activation set \(S_\ell^v\) containing \(v\), \(R\), and up to \(\ell\) nodes activated during the process. Aggregating activation sets across realizations defines a multiset \(\mathcal{M}_v\) that records how frequently other nodes co-activate with \(v\).

Let \(f_v(u)\) denote the multiplicity of node \(u\) in \(\mathcal{M}_v\) and retain the top-\(k\) co-activated nodes for each \(v\):
\begin{equation}
    F_k(v)=\operatorname*{argtop-k}_{u\in V\setminus \{v\}}\{f_v(u)\}. \label{eq:topk}
\end{equation}
If fewer than $k$ nodes have positive co-activation counts $f_v(u)>0$, then $F_k(v)$ contains only those positive-count nodes, so $|F_k(v)|\le k$. The corresponding activated edges are \(E_{\mathrm{act}}^{v}=\{(v,u)\mid u\in F_k(v)\}\). Repeating this for all \(v \in V\), we define the undirected weighted \emph{cascade-rewired auxiliary graph}:
\begin{align}
&G^{\ast} = (V, E^{\ast}, W^{\ast}), \qquad 
E^{\ast} = \bigcup_{v \in V} E_{\mathrm{act}}^{v}, \qquad W^{\ast}_{uv} =
\begin{cases}
\frac{f_v(u) + f_u(v)}{2} & \text{if } (u,v) \in E^{\ast}, \\
0 & \text{otherwise}.
\end{cases}  \label{eqn:aux_graph}
\end{align}

The multiplicity \(f_v(u)\) is the empirical cascade reinforcement score used for
the directed top-\(k\) selection in~\eqref{eq:topk}; larger values indicate that
\(u\) is repeatedly co-activated with seed \(v\) across cascade realizations. The undirected weight \(W^*_{uv}\) in~\eqref{eqn:aux_graph} symmetrizes the two directed scores. In the theoretical analysis below, we use idealized short-walk reinforcement events as population-level proxies for this empirical co-activation score.

\begin{remark}[Normalization] \label{remark1}
We may optionally normalize multiplicities before incorporating them into the weight matrix in \eqref{eqn:aux_graph}. This can be done either globally
\(\left(\smash{f_v(u) \leftarrow f_v(u)/{\max_{s,t\in V} f_s(t)}}\right)
\)
or locally
\(\left(\smash{f_v(u) \leftarrow f_v(u)/{\max_{t \in N(v)} f_v(t)}}\right)\).
The former preserves relative strength across the entire graph and emphasizes long-range structural effects, while the latter preserves relative importance within each node's neighborhood and ensures more uniform participation across nodes.
\end{remark}

%For later comparison with the theoretical abstractions (in~\ref{sbs:theory_hom},~\ref{sec:walk_length}), it is useful to name the empirical score underlying~\eqref{eq:topk}. Generalizing the notation of Section~\ref{sec:contagion}, let $\mathcal{C}(v)$ denote the collection of cascade realizations run with seed $v$, indexed over initialization chunks, permutations, and (for TAS) thresholds, and write $S_\ell^{v,c}$ for the terminal active set of realization $c \in \mathcal{C}(v)$. The implemented co-activation multiplicity is then $f_v(u) \;=\; \sum_{c \in \mathcal{C}(v)} \mathbf{1}\bigl\{u \in S_\ell^{v,c} \setminus \{v\}\bigr\}$. Thus, $f_v(u)$ is the empirical cascade reinforcement score used for the directed top-$k$ selection in~\eqref{eq:topk}. Equivalently, $\bar f_v(u) = f_v(u)/|\mathcal{C}(v)|$ is a seed-wise normalized co-activation frequency; for fixed $v$, it induces the same ranking as $f_v(u)$. The final undirected weight $W^*_{uv}$ in~\eqref{eqn:aux_graph} is obtained by symmetrizing the two directed multiplicities.

By construction, the rewired graph \(G^{\ast}\) emphasizes cohesive communities and multi-path reinforcement, as both MAS and TAS contagion modes privilege multi-hop links in densely connected regions, where early activations correspond to nodes tightly coupled to the seed. Later activations identify peripheral or weakly connected nodes; consequently, nodes separated by weak cuts rarely appear in contagions with restricted walk length \(\ell\). This naturally reveals higher-order motifs and improves label homophily in heterophilic networks. We formalize this effect theoretically and validate it empirically in Section~\ref{sbs:theory_hom} and Appendix~\ref{exp1}.

\textit{\textbf{The cascade-rewired auxiliary graph is the central contribution of this paper and can be used directly by GNN and GT backbones.}} To show that this structure also enables a standalone sparse GT, we introduce \textsc{CR-Graphormer}, a lightweight model that constructs node-specific token sequences from cascade-rewired neighborhoods, making it suitable for settings in which the backbone cannot directly exploit edge weights. The model leverages the auxiliary graph \(G^{\ast}\) constructed from contagion dynamics to define compact, node-specific attention neighborhoods that capture higher-order structure beyond immediate adjacency.

\begin{tcolorbox}[
    colback=gray!5!white,
    colframe=gray!60!black,
    title=\textbf{\textsc{CR-Graphormer}},
    fonttitle=\bfseries,
    arc=1.5mm,
    boxrule=0.8pt
]

Given an undirected, unweighted graph \(G = (V, E)\) with node features \(\mathbf{X} \in \mathbb{R}^{n \times d}\), we construct a \emph{cascade-rewired} auxiliary graph \(G^{\ast} = (V, E^{\ast}, W^{\ast})\) following Section \ref{sec:activation-rewire}. To enable efficient mini-batch processing, we derive token sequences of length at most $k+1$ from \(G^{\ast}\) for each target node \(v\) (see \eqref{eq:topk}), combining its original features \(\mathbf{x}_v\) with those of its neighbors. These tokens are processed by a standard transformer \citep{vaswani2017attention} to produce a node representation vector. Formally, the token sequence is 
$$\mathbf{T}_v :=
\Bigl(
\underbrace{\mathbf{x}_v \Vert 1}_{\text{self}},\ 
\underbrace{\{W^\ast_{uv}\mathbf{x}_u \Vert W^\ast_{uv}\}_{u\in N^\ast(v)}}_{\text{neighbors in }G^{\ast}\text{, weighted}}
\Bigr),$$ 
where \(\mathbf{T}_v\in\mathbb{R}^{(k+1)\times(d+1)}\) and \(\Vert\) denotes concatenation. Using \(\mathbf{Z}^{(0)}_v = \mathbf{T}_v\) as the node-specific input sequence, each encoder layer applies standard transformer updates:
$$\tilde{\mathbf{Z}}^{(t)}_v = \operatorname{MHA}(\operatorname{LN}(\mathbf{Z}^{(t-1)}_v)) + \mathbf{Z}^{(t-1)}_v, \quad 
\mathbf{Z}^{(t)}_v = \operatorname{FFN}(\operatorname{LN}(\tilde{\mathbf{Z}}^{(t)}_v)) + \tilde{\mathbf{Z}}^{(t)}_v,$$
where $\operatorname{MHA}$ is multi-head self-attention, $\operatorname{FFN}$ is a position-wise feed-forward network, and $\operatorname{LN}$ denotes layer normalization. After \(T\) layers, a readout (e.g., mean pooling) over \(\mathbf{Z}^{(T)}_v\) produces the final node embedding. Since all tokens are pre-selected, training is mini-batchable with a sequence length of at most \(k+1\), independent of the original graph size, where \(k\) controls the node degrees in the auxiliary graph \(G^{\ast}\).
\end{tcolorbox}

\subsection{Effect of Cascade Rewiring on Label Homophily} \label{sbs:theory_hom}
In heterophilic graphs, label information is often not carried by individual
edges but by collections of multi-hop paths. Nodes sharing the same label may
be non-adjacent, yet they can be repeatedly connected through higher-order
multi-hop structures. This observation motivates replacing edge-level
connectivity with cascade rewiring.

%To analyze the structural mechanism behind this selection, we use the idealized short-walk reinforcement event $F_L := \{(A^2+\cdots+A^L)_{uv}\ge \kappa\}$, for a fixed walk horizon \(L\ge 2\) and \(\kappa\ge 1\). For an \(\ell\)-step cascade initialized from \(\{v\}\cup R\), the corresponding seed-to-target horizon is \(L=\ell+1\), since the initialized neighbors of \(v\) already contribute one hop. When \(L\) is clear, we write \(F=F_L\) and define \(h_F:=\Pr[S\mid F]\).

\textbf{Setting and Theoretical Proxy.} Let \(G=(V,E)\) be an undirected graph with adjacency matrix \(A\) and labels \(y:V\to\{1,\ldots,C\}\). The standard edge-level homophily is $h_G=\sum_{(u,v)\in E}\mathbf{1}\{y_u=y_v\}/|E|$. Equivalently, for a uniformly random unordered pair of distinct nodes \((u,v)\), let
\(S:=\{y_u=y_v\}\) and \(\mathcal E:=\{A_{uv}=1\}\). Then \(h_G=\Pr{S\mid \mathcal E}\). The implemented rewiring selects, for each seed \(v\), nodes with large empirical
co-activation multiplicity \(f_v(u)\). To analyze the structural mechanism behind
this selection, we use the idealized short-walk reinforcement event $\mathcal F :=\{(A^2+\cdots+A^\ell)_{uv}\ge \kappa\}$, for fixed \(\ell\ge 2\) and \(\kappa\ge 1\). This event is not the exact finite-sample event defining \(G^*\). Rather, it is a population-level proxy for the same reinforcement principle: node pairs supported by many short walks are more likely to receive high cascade co-activation scores. Define the corresponding reinforcement homophily as \(h_\mathcal F:=\Pr{S\mid \mathcal F}\). We next give a Bayes-rule sufficient condition under which this idealized reinforcement event is more homophilic than direct
adjacency, i.e., \(h_\mathcal F\ge h_G\).

\iffalse
\textit{\textbf{Setting.}}
Let \(G=(V,E)\) be an undirected graph with adjacency matrix \(A\) and node
labels \(y:V\to\{1,\dots,C\}\). The standard edge-level homophily is $ h_G= \frac{\sum_{(u,v)\in E}\mathbb I(y_u=y_v)}{|E|}$. Equivalently, for a uniformly random unordered pair of distinct nodes
\((u,v)\), if  $\mathcal S:=\{y_u=y_v\}, \;\mathcal E:=\{A_{uv}=1\}$,
then $h_G=\Pr{\mathcal S\mid \mathcal E}$. We analyze an idealized short-walk reinforcement event $\mathcal F := \left\{  (A^2+\cdots+A^\ell)_{uv}\ge \kappa \right\}$, for fixed walk length \(\ell\ge 2\) and threshold \(\kappa\ge 1\). This event captures the intuition that two nodes should be connected in the auxiliary
graph when they are supported by sufficiently many short walks. In the
implemented cascade algorithm, this reinforcement effect is approximated by
repeated MAS/TAS activation frequencies. The implementation additionally
includes bootstrap neighbors through the initial random set
\(R_v\subseteq N(v)\) and then applies finite-sample top-\(k\) selection.
Thus, \(\mathcal F\) should be viewed as an analytically tractable population
proxy for the multi-hop reinforcement component of the cascade, rather than as
the exact finite-sample event defining \(G^*\). Define the population homophily induced by this reinforcement predicate as $h_\mathcal F:=\Pr{\mathcal S\mid \mathcal F}$. We now give a simple Bayes-rule sufficient condition under which reinforcement is more homophilic than direct adjacency.
\fi

\begin{assumption}[Nondegeneracy]
\label{as1}
\(\Pr{\mathcal S}>0\), \(\Pr{\mathcal E}>0\), and \(\Pr{\mathcal F}>0\).
This ensures that the label-agreement, edge-based, and reinforcement-based
conditional probabilities are well-defined.
\end{assumption}

\begin{assumption}[Label-Reinforcement Alignment]
\label{as2}
$\Pr{\mathcal F\mid \mathcal S} \ge \Pr{\mathcal E\mid \mathcal S}$. That is, same-label node pairs are at least as likely to be short-walk
reinforced as they are to be directly adjacent.
\end{assumption}

\begin{assumption}[Selectivity of Reinforcement]
\label{as3} 
$\Pr{\mathcal F}\le \Pr{\mathcal E}$. This rules out saturated regimes in which the reinforcement event occurs for most node pairs and therefore becomes less selective than direct adjacency.
\end{assumption}

\begin{proposition}[Bayes Sufficient Condition for Reinforcement Homophily]
\label{prop:bayes-monotone-matrix}
Under Assumptions~\ref{as1} and~\ref{as2}, $h_\mathcal F \ge h_G\cdot \Pr{\mathcal E}/\Pr{\mathcal F}$. If Assumption~\ref{as3} also holds, then $h_\mathcal F\ge h_G$. The inequality  $h_\mathcal{F} \geq h_G$ is strict if Assumption~\ref{as2} holds strictly, or if Assumption~\ref{as3} holds strictly and \(h_G>0\).
\end{proposition}

\begin{proof}
See Appendix~\ref{proof1}.
\end{proof}

Proposition~\ref{prop:bayes-monotone-matrix} is algorithm-agnostic: the event $\mathcal{F}$ is an idealized reinforcement-based selection rule, and the analogous selection in the implemented method is induced by retaining high-$f_v(u)$ cascade neighbors. The proposition identifies a sufficient condition under which any such rule has at least the label agreement of the original adjacency, namely that it be sufficiently label-aligned (Assumption~\ref{as2}) and sufficiently selective (Assumption~\ref{as3}). Appendix~\ref{exp1} verifies this empirically: cascade rewiring is most beneficial on heterophilic graphs, where label information is carried by reinforced multi-hop structure rather than direct edges, and less reliable on low-degree homophilic graphs, where sparse local structure provides limited same-label reinforcement.

\begin{remark} [Finite-Sample Cascade Proxy]
The rewired graph \(G^\ast\) is the finite-sample top-\(k\) graph induced by empirical scores \(\widehat p_v(u)=f_v(u)/N_v\), where \(N_v\) is the number of cascade realizations used for seed \(v\). Let $p_v(u)=\mathbb{E}[\widehat p_v(u)\mid G]$.
Under the independent-block sampling assumptions of Theorem~B.6,
for any $\delta\in(0,1)$, with probability at least $1-\delta$,
simultaneously for all seeds $v$ with $N_v\geq1$ and all $u\neq v$, $|\widehat p_v(u)-p_v(u)|
\leq \varepsilon_v(\delta) := \sqrt{\frac{\log(2n(n-1)/\delta)}{2N_{\mathrm{eff}}(v)}}$. Here $N_{\mathrm{eff}}(v)=N_v^2/\sum_j a_{v,j}^2$, where
$a_{v,j}$ are the sizes of the independent blocks of realizations
for seed $v$; dependence within each block is permitted.
Consequently, the empirical and population top-$k$ sets coincide
whenever the gap between the $k$-th and $(k+1)$-th population
scores is strictly greater than $2\varepsilon_v(\delta)$. Moreover, for fixed-threshold TAS, every nonlocal positive-count selection \(v\to u\) certifies \((A^2+\cdots+A^{\ell+1})_{vu}\ge \tau\) (see Theorem~\ref{thm:finite-sample-cascades} for a precise statement). If counts are aggregated over thresholds \(\mathcal T\), the unconditional certificate uses \(\tau_{\min}=\min\mathcal T\); threshold-specific counters certify their corresponding threshold.
%The rewired graph \(G^*\), constructed by retaining the top-\(k\) nodes under
 Thus, $G^\ast$ should be viewed as a finite-sample cascade analogue of the idealized reinforcement graph induced by \(\mathcal{F}\) rather than an exact realization of it. We measure homophily on the rewired graph via \(h_{G^*}=\Pr{y_u=y_v\mid (u,v)\in E^*}\), and compare this estimate with the population proxy \(h_{\mathcal{F}}\) in Appendix~\ref{exp1}.
\end{remark}
%Let \(p_v(u)=\mathbb E[\widehat p_v(u)]\). Appendix~\ref{app:finite-sample-cascades} shows that, for fixed graph and cascade hyperparameters, \(\widehat p_v(u)\) concentrates uniformly around \(p_v(u)\) at rate \(O(\sqrt{\log n/N_{v}})\), so the empirical top-$k$ set coincides with the population top-$k$ set whenever the $k$-th and $(k+1)$-th population scores are separated by twice this rate. 
%%%%%%%%%%%%%%%%%
\subsubsection*{\textbf{Case Study: Stochastic Block Model (SBM) as a Witness for the Reinforcement Mechanism}}

%%%New Analysis for SBM (Modified for Neurips) %%%%

To make the reinforcement mechanism concrete, we analyze a minimal SBM setting in
which direct adjacency and short-walk reinforcement can disagree. Direct edges are governed by the block-affinity matrix \(B\), whereas expected two-hop reinforcement $\mathbb E[A^2_{uv}]$ is governed by the profile-similarity kernel \(BD_\pi B^\top\). Thus, the SBM construction serves as a witness that a graph may be edge-heterophilic while its reinforced two-hop neighborhoods are homophilic. This abstraction removes finite sampling, initialization, and MAS/TAS-specific effects while preserving the core reinforcement principle used by \textsc{Graph Cascades}: pairs supported by many short paths are promoted to direct neighbors.

\textbf{Setting.} Let $G\sim SBM(n,\pi,B)$ be a $C$-block stochastic block model with node set $V$, labels $y_v\in[C]$, block sizes $n_a=\pi_a n$, and symmetric block-affinity matrix $B\in[0,1]^{C\times C}$ where $C$, $\pi$, and $B$ are fixed and are independent of $n$. Conditional on the labels, edges are independent and $\Pr{A_{uv}=1\mid y_u=a,\ y_v=b}=B_{ab}$. For two distinct nodes $u,v$, define the two-hop cascade reinforcement score $s_2(u,v):=(A^2)_{uv}=|N(u)\cap N(v)|$. For each node $u$, define the directed top-$k$ two-hop cascade graph by connecting $u$ to the $k$ nodes $v\neq u$ with the largest $s_2(u,v)$, and let $G_2^\star(k)$ denote its symmetrization. For labels $a,b\in[C]$, define $\mu_{ab}^{(n)} := \E{s_2(u,v)\mid y_u=a, y_v=b}$ with $u\neq v$. The following theorem assumes that same-class pairs are uniformly separated from different-class pairs.

%Define the finite-sample profile margin $\gamma_n := \min_{a\in[C]} \left( \mu_{aa}^{(n)}-\max_{b\neq a}\mu_{ab}^{(n)} \right),  \; \mu_{\max}:=\max_{a,b\in[C]}\mu_{ab}^{(n)}$. Assume that the profile margin dominates the uniform concentration scale: $\gamma_n \ge  12\sqrt{\mu_{\max}\log n}+16\log n$.

\begin{theorem}[Profile Homophily of Two-Hop Top-$k$ Selection]
\label{thm:sbm-profile-homophily}
Under the SBM setting above with $D_\pi := \operatorname{diag}(\pi_1,\ldots,\pi_C)$,  $\frac{1}{n}\mu_{ab}^{(n)} = (BD_\pi B^\top)_{ab} + O\!\left(\frac{1}{n}\right) \text{ for every } a,b \in [C]$. Assume the profile kernel $BD_\pi B^\top$ has a positive same-class margin
$\Delta_2 := \min_{a \in [C]} \big( (BD_\pi B^\top)_{aa} - \max_{b \ne a}(BD_\pi B^\top)_{ab} \big) > 0$.
Then for all sufficiently large $n$, with probability at least $1 - 2n^{-2}$, $s_2(u,v) > s_2(u,w)$ for every seed node $u \in V$ and any other nodes
$v, w \in V \setminus \{u\}$ with $y_v = y_u$ and $y_w \ne y_u$. Consequently, for every $k \le \min_{a \in [C]}(n_a - 1)$,
the symmetrized top-$k$ two-hop cascade graph $G_2^*(k)$ is perfectly
homophilic: $h_{G_2^*(k)} = 1$.
\end{theorem}

\iffalse
\begin{theorem}[Profile homophily of two-hop top-$k$ selection]
\label{thm:sbm-profile-homophily}
Under the SBM setting above, with $D_\pi:=\operatorname{diag}(\pi_1,\ldots,\pi_C)$,  $ \frac{1}{n}\mu_{ab}^{(n)} = (BD_\pi B^\top)_{ab} +  O\!\left(\frac{1}{n}\right) \; \text{for every } a,b\in[C]$. If the profile-margin condition above holds, then with probability at least $1-2n^{-2}$, for every seed node $u\in V$ and every pair of distinct nodes $v,w\in V\setminus\{u\}$ with $y_v=y_u$ and $y_w\neq y_u$, we have
$s_2(u,v)>s_2(u,w)$. Consequently, for every $k\le \min_{a\in[C]}(n_a-1)$, the
symmetrized top-$k$ two-hop cascade graph $G_2^*(k)$ is perfectly
homophilic: $h_{G_2^*(k)}=1$.
\end{theorem}
\fi
\begin{proof}
    See Appendix \ref{profile_proof}.
\end{proof}

Theorem~\ref{thm:sbm-profile-homophily} isolates the reinforcement mechanism rather than analyzing the finite-sample MAS/TAS pipeline. Even when \(B\) favors cross-label edges, the profile kernel  \(BD_\pi B^\top\) can separate same- from different-label pairs, so the top-$k$ two-hop rewired graph is perfectly homophilic with high probability. The witness can still be edge-heterophilic before rewiring: $h_G \xrightarrow{p} \sum_{a=1}^{C} \pi_a^2 B_{aa} / \sum_{a,b=1}^{C} \pi_a \pi_b B_{ab}$, which falls below $1/2$ whenever cross-class edge mass dominates (Corollary~\ref{cor:original-sbm-homophily}). Corollary~\ref{cor:learnability} shows that, in a contextual SBM with Gaussian class-conditional features, one-hop aggregation over the top-$k$ two-hop rewired graph followed by a nearest-centroid linear classifier has expected misclassification at most $2n^{-2} + (C-1)\exp\!\left(-\frac{(k+1)\Delta^2}{8\sigma^2}\right)$;  where $\Delta$ is the smallest distance between class means and $\sigma^2$ is the feature noise variance; the $k+1$ factor links the top-$k$ parameter directly to the downstream signal-to-noise ratio.
The multi-hop extension of Theorem~\ref{thm:sbm-profile-homophily} appears in Appendix~\ref{app:sbm-raw-l-hop}, Theorem~\ref{thm:sbm-raw-l-hop}. Corollary~\ref{cor:one-round-tas-end-to-end} provides an end-to-end guarantee for a one-round, parallel, nonlocal-only TAS variant. Under the stated SBM and sampling assumptions, its empirical top-$k$ graph is perfectly homophilic with high probability and supports a Gaussian-feature classification bound using uniform averaging and the true class means. Extending this guarantee to the sequential multi-step MAS/TAS schedule remains open.

\subsection{Choice of Walk Length $\ell$}\label{sec:walk_length}

The walk length \(\ell\) sets the spatial scale at which \(f_v(u)\) accumulates support. We use effective resistance \citep{spielman-srivastava} as a structural certificate for threshold activations: if a node is activated within \(\ell\) steps by a threshold-\(\kappa\) cascade, then its activation is supported by multiple short paths from the initial active set. This is a one-way statement about activated nodes, not a guarantee that all low-resistance nodes are selected. It directly describes TAS-style thresholding;  for MAS, an analogous bound holds whenever the activation threshold $\tau_t$ is at least $\kappa$.

%The walk length \(\ell\) determines the spatial scale at which the rewiring operates. Restricting each contagion to \(\ell\) propagation steps limits reinforcement to paths of length at most \(\ell\), preventing the accumulation of weak long-range connections while still capturing multi-hop structure beyond immediate adjacency. We formalize this mesoscopic notion of reinforcement using \textit{effective resistances} . Nodes with low effective resistance to \(S_0^v\) exhibit strong multi-path connectivity and are preferentially activated within \(\ell\) steps.

\textit{\textbf{Setting.}} Let \(G=(V,E)\) be an undirected, unweighted graph. Fix a nonempty initial active set \(S_0^v\subseteq V\), and let $S_0^v \subseteq S_1^v \subseteq \cdots \subseteq S_\ell^v$ be the active sets generated by a uniform threshold-\(\kappa\) contagion, where \(\kappa\in \mathbb{N}\). That is, whenever a vertex \(x\) first activates at time \(t\ge 1\), it satisfies $\bigl|N(x)\cap S_{t-1}^v\bigr| \ge \kappa$. Let $A^v := \bigcup_{t=0}^{\ell} S_t^v = S_\ell^v$. For any \(u\in S_\ell^v\setminus S_0^v\), let $T(u):=\min\{t\in\{0,1,\dots,\ell\}:u\in S_t^v\}$ be its activation time. Let $G_v$ be the subgraph induced by the activated nodes $\bigcup_{t=0}^{\ell} S_t^v$. We write \(R_{\mathrm{eff}}^G(S,u)\) for the effective resistance between \(u\) and the terminal obtained by contracting all vertices of \(S\) to a single vertex.

\begin{theorem}[Effective-Resistance Bound for Threshold Activations]
 \label{thm:walklength}
 $R_{\mathrm{eff}}^{G_v}(S_0^v,u) \le \frac{T(u)}{\kappa} \le \frac{\ell}{\kappa}$.
\end{theorem}
\begin{proof}
   The proof follows directly from the two lemmas below (see Appendix \ref{proof5}).  
   \pushQED{}
\end{proof}

\begin{lemma}
\label{lem:paths}
Following the setting of Theorem~\ref{thm:walklength}, if
\(u\in S_\ell^v\setminus S_0^v\), then \(G_v\) contains \(\kappa\) pairwise edge-disjoint
paths from \(S_0^v\) to \(u\). Moreover, each of these paths has length at most \(T(u)\), and hence at most \(\ell\).
\end{lemma}
\begin{proof}
See Appendix \ref{proof3}. \pushQED{}
\end{proof}
\begin{lemma}
\label{lem:res}
Let \(H=(W,F)\) be an undirected, unweighted graph, let \(S\subseteq W\) be nonempty, and let \(u\in W\setminus S\). Suppose \(H\) contains \(\kappa\) pairwise edge-disjoint paths \(P_1,\dots,P_\kappa\) of lengths \(\ell_1,\dots,\ell_\kappa\) from \(S\) to \(u\). Then $R_{\mathrm{eff}}^H(S,u) \le \frac{1}{\kappa^2}\sum_{i=1}^{\kappa}\ell_i$. In particular, if each \(\ell_i\le L\), then $R_{\mathrm{eff}}^H(S,u)\le \frac{L}{\kappa}$.
\end{lemma}
\begin{proof}
See Appendix \ref{proof4}. 
\pushQED{}
\end{proof}

Theorem~\ref{thm:walklength} implies that activated nodes are not arbitrary long-range shortcuts: they admit $\kappa$ edge-disjoint short paths from the initial active set, and the resulting $R_\mathrm{eff}$ bound scales with $\ell$. While this is a one-way certificate rather than a global optimization guarantee, it motivates using a small $\ell$ in our experiments to concentrate $f_v(u)$ on mesoscopically well-supported neighbors.

%characterizes the structural bias of thresholded cascade selection: an activated node is not an arbitrary long-range shortcut, since activation within $\ell$ steps certifies $\kappa$ short edge-disjoint supports and yields $R_{\mathrm{eff}} \le \ell/\kappa$. The result does not claim that the final top-$k$ graph $G^*$ globally optimizes effective resistance, but it provides a structured way to choose $\ell$: since this bound scales with $\ell$, our experiments use smaller $\ell$ to favor stronger effective-resistance connectivity, yielding scores $f_v(u)$ concentrated on mesoscopically well-supported neighbors.

\textit{\textbf{Edge-Connectivity Certificate.}}
A complementary certificate follows from the same activation-path argument. Appendix~\ref{app:edge-connectivity} shows that any nonlocal directed edge \(v\to u\) selected by a fixed-threshold TAS-\(\kappa\) run satisfies \(\lambda_G(N[v],u)\ge \kappa\) in the original graph, where $\lambda_G(N[v],u)$ denotes the set-to-vertex edge connectivity (the maximum number of edge-disjoint paths from $N[v]$ to $u$). Intuitively, fixed-threshold TAS rewiring reinforces redundant mesoscopic connectivity and cannot introduce propagation-selected shortcuts across cuts of size smaller than \(\kappa\). Together, these results delineate the regime where cascade rewiring is expected to help—graphs whose label-relevant structure is carried by reinforced multi-hop neighborhoods with sufficient local redundancy—and identify bottlenecked sparse graphs as the predicted failure mode.

\section{Experiments}\label{sec:empirical}

In this section, we show how our cascade-based mesoscopic rewiring improves homophily in heterophilic graphs and boosts the performance of GNNs and GTs.  Additionally, we show that our proposed \textsc{CR-Graphormer} is a competitive alternative to existing graph models.

\textit{\textbf{Datasets.}} We consider $16$ benchmarks from the Deep Graph Library (DGL)\footnote{\url{https://www.dgl.ai/dgl_docs/api/python/dgl.data.html}} including both homophilic and heterophilic settings. Our code is publicly available.\footnote{GitHub: \url{https://github.com/ruiz-lab/Graph-Cascades}} Dataset information is given in Table \ref{tab:dataset_stats} (Appendix \ref{app:exp}). To study the effects of mesoscopic cascades, we generate 125 balanced-community SBM networks with a fixed intra-community edge probability \(p_{\mathrm{in}} = 0.005\). We vary the number of nodes \(n \in \{500, 1000, 1500, 2000, 2500\}\), the number of communities \(c \in \{2, 4, 6, 8, 10\}\), and the inter-community edge probability \(p_{\mathrm{out}} \in \{10^{-5}, 10^{-4}, 10^{-3}, 10^{-2}, 10^{-1}\}\) to cover a range of network structures. 
Node features are sampled i.i.d. as \(\mathbf{X} \sim \mathcal{N}(0,1)^{n \times 5}\). We use these synthesized graphs to examine how mesoscopic cascades affect label homophily (Appendix \ref{exp1}) and assess the correlation between model performance and graph structural statistics (Table~\ref{tab:correlation}). Additional experiments on long-range graph benchmarks (\textit{PascalVOC-SP}, \textit{COCO-SP}, \textit{Questions}, \textit{Roman-empire}, and \textit{Amazon-ratings}) and non-duplicate heterophilic datasets (\textit{chameleon-filtered} and \textit{squirrel-filtered}) are reported in Appendix~\ref{app:longrange} with the corresponding standard metrics (accuracy, macro-F1, and AUROC).%, using the standard metrics for each setting, namely macro-F1 for class-imbalanced tasks and AUROC for Questions.

\textit{\textbf{Mesoscopic Cascades.}} We rewire graphs using the MAS and TAS cascades from Section~\ref{sec:contagion}, with initial active-set size $|R|=5$ and cascade length $\ell=10$, each with $P=5$ permutations and $k$ in \eqref{eq:topk} set to the rounded average degree to preserve sparsity. For TAS, we use thresholds $\tau \in \{1,2,3,4,5\}$ and select the frequency normalization by validation performance (Remark~\ref{remark1}). \textit{This is our default setting.} Methods with MAS/TAS in their names use the corresponding weighted rewired graph $(G^\ast, W^\ast)$; all others use the original unweighted graph $G$. From Appendix~\ref{ablation}, TAS is generally preferable on heterophilic graphs because thresholding filters weakly supported activations, while MAS gives tighter structural alignment but is more sensitive to walk length. On skewed-degree graphs, setting $k$ to the median degree rather than the average may reduce over-inclusion.

\textit{\textbf{Experimental Details.}} We compare against sixteen representative methods, comprising five GNNs and eleven GTs, with details provided in Appendix~\ref{app:baselines}. Each model is trained for up to 2000 epochs. All experiments use a 50\%–25\%–25\% train--validation--test split and are repeated over 20 random splits per graph for robustness. Table~\ref{tab:numerical1} reports the within-backbone comparison for GCN, GraphGPS (GPS), NAGphormer (NAG), and VCR-Graphormer (VCR) along with our \textsc{CR-Graphormer}. Table~\ref{tab:numerical2} provides the broader cross-architecture context in comparison with our best validation selected cascade models (the configuration with the highest mean validation accuracy). We report Tables~\ref{tab:numerical1} and~\ref{tab:numerical2} without standard deviations for readability; complete numerical results with standard deviations are in Appendix~\ref{app:numerical1_appendix} (Tables~\ref{std2} and \ref{std1}). Complete experimental details, the empirical validation of Assumptions~\ref{as1}--\ref{as3} and Proposition~\ref{prop:bayes-monotone-matrix}, ablation studies, and comparisons with other rewiring methods and heterophily-specific architectures are provided in Appendices~\ref{app:exp}, \ref{exp1}, \ref{ablation}, \ref{app:rewiring} and \ref{app:heterophily}, respectively.

%Results across six additional backbones (GAT, GT, Gophormer, ACMGCN, H2GCN, and GloGNN) evaluated on nine graphs within the target regimes (heterophilic and moderate- to high-degree homophilic) are summarized in Appendix~\ref{10-backbones}. 

\begin{table*}[!ht]
\begin{minipage}[t]{0.54 \textwidth}
\centering
\setlength{\tabcolsep}{2.9pt}
\renewcommand{\arraystretch}{1.05}
\caption{Number of DGL graphs where \textsc{Graph Cascades} improves test accuracy (GraphGPS counts exclude $3$ OOM target datasets). Colors denote the graph type with the largest gain: \red{\textbf{red}} (heterophilic, target regime), \blue{\textbf{blue}} (moderate- to high-degree homophilic, target regime), and \black{\textbf{black}} (low-degree homophilic). Complete results are in Table~\ref{tab:numerical1}.}
\tiny
\begin{tabular}{lcccc}
\hline
Baseline & \shortstack{Cascade \\ Model} & \shortstack{Largest \\ Increase} & \shortstack{\# Graphs Improved \\ Out of All 16 Graphs} & \shortstack{\# Graphs Improved \\ Out of 12 Target Graphs} \\
\hline
GCN & GCN-MAS & \red{\textbf{+17.11\%}} & 8 (avg. +0.39\%) & 8 (avg. +1.06\%) \\
 & GCN-TAS & \red{\textbf{+20.15\%}} & 10 (avg. +2.87\%) & 10 (avg. +4.03\%) \\
\hline
GraphGPS & GPS-MAS & \blue{\textbf{+8.15\%}} & 8 (avg. +2.00\%) & 7 (avg. +3.29\%) \\
 & GPS-TAS & \blue{\textbf{+9.52\%}} & 9 (avg. +2.23\%) & 7 (avg. +3.08\%) \\
\hline
NAG & NAG-MAS & \red{\textbf{+19.77\%}} & 11 (avg. +4.67\%) & 11 (avg. +8.85\%) \\
 & NAG-TAS & \blue{\textbf{+34.54\%}} & 13 (avg. +6.94\%) & 12 (avg. +11.66\%) \\
\hline
VCR & VCR-MAS & \red{\textbf{+32.74\%}} & 13 (avg. +10.41\%) & 12 (avg. +15.90\%) \\
 & VCR-TAS & \red{\textbf{+32.33\%}} & 13 (avg. +13.63\%) & 12 (avg. +19.79\%) \\
\hline
\end{tabular}
\label{tab:summary}
\end{minipage}
\hfill
\begin{minipage}[t]{0.44\textwidth}
\tiny
\centering
\setlength{\tabcolsep}{4.4pt}
\renewcommand{\arraystretch}{1.05}
\caption{Regression coefficients for $A(M) \sim H \ast L(D) + H \ast L(D) \ast C$, where $H$ is label homophily, $L(D)$ is log average degree, and $C \in \{0,1\}$ indicates original graph connectivity. $A(M)$ denotes test accuracy, evaluated on the original graph unless MAS/TAS specifies an auxiliary one.}\label{tab:correlation}
\begin{tabular}{lcccc}
    \hline
    M & $\beta_0$ & $\beta_1$ & $\beta_2$ & Adjusted $R^2$ \\
    \hline
    NAG & 0.193*** & 0.197*** & 0.170**\phantom{*} & 0.454 \\
    NAG-MAS & 0.143*** & 0.258*** & 0.173*** & 0.722 \\
    NAG-TAS & 0.152*** & 0.265*** & 0.237*** & 0.695\\
    \hline
    VCR & 0.253*** & 0.233*** & -0.005 \phantom{***} & 0.426 \\
    VCR-MAS & 0.155*** & 0.306*** & 0.157*** & 0.704 \\
    VCR-TAS & 0.172*** & 0.301*** & 0.201*** & 0.607 \\
    \hline
    CR-MAS & 0.091*** & 0.238*** & 0.261*** & 0.760 \\
    CR-TAS & 0.132*** & 0.219*** & 0.421*** & 0.639\\
    \hline
\end{tabular}
\makebox[\linewidth][c]{%
     $^{***}p<0.001$; $^{**}p<0.01$; none: $p\geq0.1$%
}
\end{minipage}\\

\centering
\tiny
\setlength{\tabcolsep}{6.13pt}
\renewcommand{\arraystretch}{1.05}
\caption{Alphabetical dataset labels and names. Colors indicate graph type, aligned with Table~\ref{tab:summary}. See Table~\ref{tab:dataset_stats} for exact label homophily and average degree.}
\begin{tabular}{c l c l c l c l c l c l c l c l}
\hline
\red{\textbf{A}} & \red{\textbf{\textit{actor}}} &
\red{\textbf{B}} & \red{\textbf{\textit{chameleon}}} &
\black{\textbf{C}} & \black{\textbf{\textit{citeseer}}} &
\blue{\textbf{D}} & \blue{\textbf{\textit{community}}} &
\blue{\textbf{E}} & \blue{\textbf{\textit{computer}}} &
\black{\textbf{F}} & \black{\textbf{\textit{cora}}} &
\red{\textbf{G}} & \red{\textbf{\textit{cornell}}} &
\black{\textbf{H}} & \black{\textbf{\textit{cycle}}} \\

\black{\textbf{I}} & \black{\textbf{\textit{grid}}} &
\blue{\textbf{J}} & \blue{\textbf{\textit{photo}}} &
\blue{\textbf{K}} & \blue{\textbf{\textit{pubmed}}} &
\blue{\textbf{L}} & \blue{\textbf{\textit{shape}}} &
\red{\textbf{M}} & \red{\textbf{\textit{squirrel}}} &
\red{\textbf{N}} & \red{\textbf{\textit{texas}}} &
\blue{\textbf{P}} & \blue{\textbf{\textit{wiki}}} &
\red{\textbf{Q}} & \red{\textbf{\textit{wisconsin}}} \\
\hline
\end{tabular}
\label{tab:dataset-labels-compact}
\tiny
\centering
\setlength{\tabcolsep}{2.72pt}
\renewcommand{\arraystretch}{1.05}
\caption{Test accuracy (mean, in \%) across DGL datasets, summarized in Table \ref{tab:summary}. OOM indicates out-of-memory errors. Accuracy is encoded by shading, with darker cells indicating better performance. Table~\ref{tab:dataset-labels-compact} provides dataset names. Table~\ref{std1} reports the corresponding standard deviations.}
\begin{tabular}{llcccccccccccccccc}
\hline
& Model & \red{\textbf{A}} & \red{\textbf{B}} & \black{\textbf{C}} & \blue{\textbf{D}} & \blue{\textbf{E}} & \black{\textbf{F}} & \red{\textbf{G}} & \black{\textbf{H}} & \black{\textbf{I}} & \blue{\textbf{J}} & \blue{\textbf{K}} & \blue{\textbf{L}} & \red{\textbf{M}} & \red{\textbf{N}} & \blue{\textbf{P}} & \red{\textbf{Q}} \\
\hline

GNN-based & GCN & 29.59 & 57.30 & \cellcolor{myred!60} 75.70 & \cellcolor{myred!60} 55.63 & \cellcolor{myred!30} 90.30 & \cellcolor{myred!60} 87.96 & 45.11 & \cellcolor{myred!60} 60.73 & \cellcolor{myred!60} 59.34 & 93.90 & 87.16 & \cellcolor{myred!30} 89.46 & 39.10 & 54.57 & \cellcolor{myred!60} 83.89 & 51.41 \\
Architectures & GCN-MAS & \cellcolor{myred!30} 34.28 & \cellcolor{myred!30} 57.67 & 74.71 & \cellcolor{myred!30} 51.76 & 89.34 & 84.93& \cellcolor{myred!60} 60.85 & 59.27 & \cellcolor{myred!30} 58.30 & \cellcolor{myred!30} 93.91 & \cellcolor{myred!30} 87.87 & 57.69 & \cellcolor{myred!30} 40.01 & \cellcolor{myred!60} 66.49 & 81.79& \cellcolor{myred!30} 68.52 \\
& GCN-TAS & \cellcolor{myred!60} 34.62 & \cellcolor{myred!60} 58.75 & \cellcolor{myred!30} 75.46 & 50.11 & \cellcolor{myred!60} 90.87 & \cellcolor{myred!30} 87.29 & \cellcolor{myred!30} 57.02 & \cellcolor{myred!30} 60.39 & 58.19 & \cellcolor{myred!60} 94.17 & \cellcolor{myred!60} 88.15 & \cellcolor{myred!60} 89.57 & \cellcolor{myred!60} 42.38 & \cellcolor{myred!30} 65.00 & \cellcolor{myred!30} 83.54& \cellcolor{myred!60} 71.56\\
\cline{2-18}
& GraphGPS & 35.19 & \cellcolor{mygreen!60} 60.78 & \cellcolor{mygreen!30}75.23 & 44.00 & OOM & \cellcolor{mygreen!60} 86.65 & 61.38 & 55.11 & \cellcolor{mygreen!60} 57.41 & 94.64 & OOM & 46.74 & \cellcolor{mygreen!60} 40.04 & 69.68 & OOM & 72.66 \\
& GPS-MAS & \cellcolor{mygreen!30} 36.39 & \cellcolor{mygreen!30} 59.90 & 74.09 & \cellcolor{mygreen!60} 47.43 & OOM & 85.49 & \cellcolor{mygreen!60} 68.72 & \cellcolor{mygreen!30} 55.78& 55.44 & \cellcolor{mygreen!30} 94.71& OOM & \cellcolor{mygreen!30} 54.89 & \cellcolor{mygreen!30} 39.20 & \cellcolor{mygreen!30} 75.00 & OOM & \cellcolor{mygreen!60} 78.44 \\
& GPS-TAS & \cellcolor{mygreen!60} 36.46 & 59.66 & \cellcolor{mygreen!60} 75.28 & \cellcolor{mygreen!30} 45.26 & OOM & \cellcolor{mygreen!30} 86.53 & \cellcolor{mygreen!30} 67.23 & \cellcolor{mygreen!60} 57.51 & \cellcolor{mygreen!30} 56.33 & \cellcolor{mygreen!60} 94.84 & OOM & \cellcolor{mygreen!60} 56.26 & 39.18 & \cellcolor{mygreen!60} 75.64 & OOM & \cellcolor{mygreen!30} 78.28 \\

\hline
Tokenized & NAG & 31.68 & 40.38 & \cellcolor{mycyan!60} 74.72 & 46.17 & 88.66 & \cellcolor{mycyan!30} 86.75 & 58.94 & \cellcolor{mycyan!60} 77.24 & \cellcolor{mycyan!60} 70.99 & 93.94 & 87.26 & 49.89 & 25.80 & 63.30 & \cellcolor{mycyan!30} 82.40 & 61.25 \\
Sparse- & NAG-MAS & \cellcolor{mycyan!60} 35.10 & \cellcolor{mycyan!60} 59.15 & 72.17 & \cellcolor{mycyan!30} 52.59 & \cellcolor{mycyan!30} 89.10 & 84.84 & \cellcolor{mycyan!60} 74.36 & 59.00 & \cellcolor{mycyan!30} 62.25 & \cellcolor{mycyan!30} 94.27 & \cellcolor{mycyan!30} 88.19 & \cellcolor{mycyan!30} 57.54 & \cellcolor{mycyan!30} 40.89 & \cellcolor{mycyan!60} 81.28 & 82.38 & \cellcolor{mycyan!60} 81.02 \\
Attention & NAG-TAS & \cellcolor{mycyan!60} 35.10 & \cellcolor{mycyan!30} 58.41 & \cellcolor{mycyan!30} 74.15 & \cellcolor{mycyan!60} 58.21 & \cellcolor{mycyan!60} 89.19 & \cellcolor{mycyan!60} 86.82 & \cellcolor{mycyan!60} 74.36 & \cellcolor{mycyan!30} 59.20 & 60.60 & \cellcolor{mycyan!60} 94.61 & \cellcolor{mycyan!60} 88.26 & \cellcolor{mycyan!60} 84.43 & \cellcolor{mycyan!60} 42.13 & \cellcolor{mycyan!60} 81.28 & \cellcolor{mycyan!60} 82.59 & \cellcolor{mycyan!60} 81.02 \\
\cline{2-18}
GTs & VCR & 26.41 & 27.28 & 71.15 & 28.13 & 82.45 & \cellcolor{mypurple!60} 85.41 & 44.89 & \cellcolor{mypurple!60} 61.64 & \cellcolor{mypurple!60} 58.32 & 88.45 & 83.25 & 44.66 & 22.48 & 54.68 & 78.46 & 51.33 \\

& VCR-MAS & \cellcolor{mypurple!30} 36.18 & \cellcolor{mypurple!60}  60.02 & \cellcolor{mypurple!30} 71.47 & \cellcolor{mypurple!30}47.22 & \cellcolor{mypurple!60}89.43 & 85.08 & \cellcolor{mypurple!30} 73.85 & \cellcolor{mypurple!30} 51.70 & 44.05 & \cellcolor{mypurple!30} 93.88 & \cellcolor{mypurple!30}87.59 & \cellcolor{mypurple!30}52.12 & \cellcolor{mypurple!30} 39.37 & \cellcolor{mypurple!30} 80.99 & \cellcolor{mypurple!30} 82.68 & \cellcolor{mypurple!30} 79.93 \\

& VCR-TAS & \cellcolor{mypurple!60} 36.28 & \cellcolor{mypurple!30} 59.61 & \cellcolor{mypurple!60} 74.14 & \cellcolor{mypurple!60} 57.97 & \cellcolor{mypurple!30} 88.95 & \cellcolor{mypurple!30} 85.11 & \cellcolor{mypurple!60} 75.02 & 50.09 & \cellcolor{mypurple!30} 47.84 & \cellcolor{mypurple!60} 94.22 & \cellcolor{mypurple!60} 88.70 & \cellcolor{mypurple!60} 76.94 & \cellcolor{mypurple!60} 46.06 & \cellcolor{mypurple!60} 82.31 & \cellcolor{mypurple!60} 82.84 & \cellcolor{mypurple!60} 81.00 \\
\cline{2-18}
& CR-MAS & 31.52 & 62.82 & 50.63 & 54.04 & \cellcolor{myblack!45} 89.44 & 79.81 & 54.04 & 59.18 & \cellcolor{myblack!45} 61.80 & 93.73 & 87.53 & 54.37 & 45.72 & \cellcolor{myblack!45} 65.43 & 82.65 & 66.56 \\
& CR-TAS & \cellcolor{myblack!45} 34.64 & \cellcolor{myblack!45} 63.65 & \cellcolor{myblack!45} 51.84 & \cellcolor{myblack!45} 78.60 & 89.33 & \cellcolor{myblack!45} 81.02 & \cellcolor{myblack!45} 55.74 & \cellcolor{myblack!45} 61.32 & 60.74 & \cellcolor{myblack!45} 93.95 & \cellcolor{myblack!45} 87.75 & \cellcolor{myblack!45} 85.94 & \cellcolor{myblack!45} 47.31 & \cellcolor{myblack!45} 65.43 & \cellcolor{myblack!45} 83.04 & \cellcolor{myblack!45} 70.39 \\
\hline
\end{tabular}
\label{tab:numerical1}
\end{table*}

\textbf{Model-Agnostic Gains from Cascades.} From Tables~\ref{tab:summary} and~\ref{tab:numerical1}, cascade rewiring with just default parameter selections already improves classification on a substantial portion of benchmarks for both GNNs and GTs. The gains are most pronounced and frequent for NAG and VCR, where cascades improve performance on up to $13/16$ graphs, indicating that these architectures benefit particularly strongly from cascade rewiring. Notably, across the $12$ graphs in the target regimes, MAS improves NAG and VCR performance by an average of \textbf{+8.85\%} and \textbf{+15.90\%}, while TAS increases performance by \textbf{+11.66\%} and \textbf{+19.79\%}, respectively. Improvements are also widespread for GCN ($10/12$ target graphs) and consistent for GraphGPS ($7/9$ target graphs; the remaining $3$ are OOM). Overall, the strongest single improvements reach \textbf{+32.74\%} under MAS and \textbf{+34.54\%} under TAS. %Notably, these classification gains are achieved using a fixed hyperparameter setting per GNN/GT model, applied unchanged when training on both the original and rewired graphs. 

%\textbf{Cascade Weights.} The cascade weights $W^\ast$ are not a post-hoc embellishment: they are the empirical reinforcement scores induced by the contagion dynamics. Cascade rewiring is most naturally viewed as a single weighted operator, rather than as a topology followed by a separate weighting step. The ablations in Appendix~\ref{ablation} (Tables~\ref{tab:summary_unweighted}--\ref{tab:numerical3}) compare the unweighted auxiliary graph $G^\ast$ with the full weighted operator, and reveal that the contribution of $W^\ast$ is architecture-dependent. For sparse-attention GTs, $G^\ast$ alone already yields broad gains, with VCR improving on 13 of 16 graphs, indicating that the rewired topology carries meaningful structural signal. For GNN-based backbones, whose aggregation is more sensitive to degree imbalance, the frequency weights yield substantially larger and more consistent gains. We therefore report \textsc{Graph Cascades} with $W^\ast$ as the recommended configuration. The topology-only ablation shows that both the rewired structure and the reinforcement weights contribute to performance, with their relative contributions depending on the backbone.

We observe pronounced improvements on all $6$ heterophilic datasets (\textit{actor}: +9.87\%; \textit{chameleon}: +32.74\%; \textit{cornell}: +30.13\%; \textit{squirrel}: +23.58\%; \textit{texas}: +27.63\%; \textit{wisconsin}: +29.67\%). A possible explanation for this—and for the performance declines in homophilic networks such as \textit{citeseer}, \textit{cora}, \textit{cycle}, and \textit{grid}—is that cascades primarily boost homophily in heterophilic and moderate- to high-degree homophilic graphs but decrease homophily in low-degree homophilic graphs due to the violation of Assumption~\ref{as2}. This trend aligns with Proposition~\ref{prop:bayes-monotone-matrix} and is consistent across GNNs and GTs (Figure~\ref{fig:summary} in Appendix~\ref{app:numerical1_appendix}, Table~\ref{tab:reinforcement-homophily-full} in Appendix~\ref{exp1}). Table~\ref{tab:regimes} summarizes the regimes induced by Assumptions~\ref{as2} and~\ref{as3} and the empirical behavior they predict; \emph{the assumptions are evaluated directly on each DGL benchmark} in Appendix~\ref{exp1}.

\begin{table}[ht]
\centering
\small
\setlength{\tabcolsep}{2.5pt}
\renewcommand{\arraystretch}{1.15}
\caption{Regimes identified by alignment (Assumption~\ref{as2}) and selectivity (Assumption~\ref{as3}) conditions of Proposition~\ref{prop:bayes-monotone-matrix}, together with their empirical behavior. Full per-dataset verification is provided in Appendix~\ref{exp1}.}
\label{tab:regimes}

\begin{tabular}{p{2.3cm} c c p{5.1cm} p{2.9cm}}
\hline
\textbf{Regime} & \textbf{Alignment} & \textbf{Selectivity} & \textbf{Example Datasets} & \textbf{Observed behavior} \\
\hline

Heterophilic
  & Holds & Holds
  & \textit{actor}, \textit{chameleon}, \textit{cornell}, \textit{squirrel}, \textit{texas}, \textit{wisconsin}, \textit{Amazon-ratings}, \textit{Questions}
  & Largest gains ($+9.87\%$ to $+32.74\%$) \\
\hline
Moderate to high-degree homophilic
  & Holds & Fails
  & \textit{community}, \textit{computer}, \textit{photo}, \textit{pubmed}, \textit{shape}, \textit{wiki}, \textit{COCO-SP}, \textit{PascalVOC-SP}
  & Often positive gains via Theorem~\ref{thm:walklength} \\
\hline
Low-degree homophilic
  & Fails & Holds
  & \textit{citeseer}, \textit{cora}
  & Small or unstable changes \\
\hline
Bottlenecked or highly regular
  & Fails & N/A
  & \textit{cycle}, \textit{grid}, \textit{Roman-empire}
  & Predicted failure mode, unstable changes \\
\hline
\end{tabular}
\end{table}

Furthermore, cascade rewiring improves macro-F1 on the cohesive mesoscopic long-range datasets \textit{COCO-SP} and \textit{PascalVOC-SP}, improves AUROC on \textit{Amazon-ratings} and \textit{Questions}, and yields gains on \textit{chameleon-filtered}. On \textit{Roman-empire}, the effect is backbone-dependent: GNN-based architectures degrade, consistent with the scarcity of the $\kappa$ edge-disjoint short paths required by Theorem~\ref{thm:walklength} and Appendix~\ref{app:edge-connectivity}, whereas sparse-attention GT backbones gain, modestly for NAG (up to $+4.09\%$) and strongly for VCR (up to $+39.05\%$). See Appendix~\ref{app:longrange} for numerical results. Results across six additional backbones (GAT, GT, Gophormer, ACMGCN, H2GCN, and GloGNN) on nine graphs in the target regimes (heterophilic and moderate- to high-degree homophilic) are reported in Table~\ref{tab:cascade_summary} (Appendix~\ref{10-backbones}) and are consistent with the trends in Tables~\ref{tab:summary} and~\ref{tab:numerical1}. % Results on six other baseline backbones are summarized in Table~\ref{tab:cascade_summary}. 

\begin{remark}[A Note on Cascade Weights]
    The weights $W^\ast$ are intrinsic to cascade rewiring: they are the empirical reinforcement scores produced by the contagion dynamics, not a post-hoc addition. Thus, \textsc{Graph Cascades} is best viewed as a single weighted operator. The ablations in Appendix~\ref{weights} (Tables~\ref{tab:summary_unweighted} and \ref{tab:numerical3}) show that the role of $W^\ast$ depends on the backbone. For sparse-attention GTs, the unweighted graph $G^\ast$ already gives broad gains, indicating that the rewired topology carries meaningful structural signal. For GNN backbones, whose aggregation is more sensitive to degree imbalance, the frequency weights provide larger and more consistent improvements. We therefore use the weighted operator as the default configuration. The topology-only ablation confirms that both structure and reinforcement strength contribute to performance.
\end{remark}

\textbf{How Accuracy Relates to Homophily, Degree, and Connectivity.} To further assess the impact of cascade rewiring on graph learning performance, we conduct a regression study relating test accuracy to key graph characteristics, including label homophily, average degree, and graph connectivity. These statistics are computed on the input graph: the original graph with unit edge weights for baseline methods, and the MAS/TAS-rewired graphs with frequency-based weights.

Table \ref{tab:correlation} reports regressions of test accuracy on topological metrics across 125 synthesized graphs and 16 benchmark datasets. Overall, rewired structural properties explain a larger fraction of performance variation than those derived from the original graphs. While TAS typically yields larger absolute performance gains across datasets and architectures (Tables \ref{tab:summary} and \ref{tab:numerical1}), MAS shows stronger alignment with the topological structure, attaining higher adjusted $R^2$ values for NAG, VCR, and the highest overall for \textsc{CR-Graphormer}. Notably, despite synthesized graphs using random node features with no semantic content, we observe consistent correlations between accuracy and structural metrics across both synthesized and real-world benchmarks (Figure \ref{fig:correlation} in Appendix \ref{app:correlation}). This indicates that the observed effects are primarily driven by graph topology rather than feature semantics.

\textbf{Comparison to GNN and GT Baselines.} We compare our best validation selected cascade models against the five GNNs and eleven GTs. %Our top results consistently rank within the top 5 across all datasets: 1st for \textit{community}, \textit{cornell}, \textit{pubmed}, \textit{shape}, \textit{squirrel}, and \textit{texas}; 2nd for \textit{actor}, \textit{computer}, and \textit{wisconsin}; 3rd for \textit{chameleon}; 4th for \textit{citeseer}, \textit{cora}, and \textit{grid}; and 5th for \textit{cycle}, \textit{photo}, and \textit{wiki} 
The best validation-selected cascade model ranks among the top five on all $16$ datasets and first on six  (Table \ref{tab:numerical2}). While the test accuracy of our best model is lower than the highest accuracy from baseline GNNs and GTs by $20.15\%$ on \textit{cycle} and 9.89\% on \textit{grid}, it remains competitive on the other graphs (\textit{\textbf{at most 1.18\% lower than the observed best accuracy of all GNN and GT baselines}}). The declines on \textit{cycle} and \textit{grid} may reflect their  low degree and regular structure, which limit the benefits of cascade rewiring.
\begin{table*}[ht!]
\tiny
\setlength{\tabcolsep}{2.55pt}
\renewcommand{\arraystretch}{1.05}
\caption{Test accuracy (mean, \%) on DGL datasets. OOM denotes out-of-memory. ``Best-Validated'' indicates the best validation-selected cascade models from Table~\ref{tab:numerical1}. Top-5 models are shaded (rows 1-17 only, darker indicates better performance). See Table~\ref{tab:dataset-labels-compact} for dataset names and Tables~\ref{std2} and \ref{std1} for standard deviations.}
\vspace{2pt}
\begin{tabular}{llcccccccccccccccc}
\hline
 & & \red{\textbf{A}} & \red{\textbf{B}} & \black{\textbf{C}} & \blue{\textbf{D}} & \blue{\textbf{E}} & \black{\textbf{F}} & \red{\textbf{G}} & \black{\textbf{H}} & \black{\textbf{I}} & \blue{\textbf{J}} & \blue{\textbf{K}} & \blue{\textbf{L}} & \red{\textbf{M}} & \red{\textbf{N}} & \blue{\textbf{P}} & \red{\textbf{Q}} \\
\hline

GNNs & GCN
& 29.59 & 57.30 & \cellcolor{mycolor!54} 75.70 & 55.63 & 90.30 & \cellcolor{mycolor!72} 87.96 & 45.11
& 60.73 & \cellcolor{mycolor!18} 59.34 & 93.90 & \cellcolor{mycolor!36} 87.16 & \cellcolor{mycolor!72} 89.46 & 39.10 & 54.57 & \cellcolor{mycolor!54} 83.89 & 51.41 \\

& GraphSAGE
& 34.57 & \cellcolor{mycolor!72} 64.36& \cellcolor{mycolor!90} 75.88 & 52.09 & \cellcolor{mycolor!18} 90.39 & \cellcolor{mycolor!90} 88.09 & \cellcolor{mycolor!18} 64.47
& 59.27 & 58.19 & \cellcolor{mycolor!90} 95.31 & \cellcolor{mycolor!72} 87.33 & 43.11 & \cellcolor{mycolor!72} 45.41 & \cellcolor{mycolor!36} 79.79 & \cellcolor{mycolor!90} 84.57 & 72.89 \\

& GAT
& 28.85 & 58.63 & \cellcolor{mycolor!72} 75.77 & \cellcolor{mycolor!18} 58.99 & OOM & \cellcolor{mycolor!54} 87.38 & 44.15
& 59.27 & 58.19 & 94.49 & 86.88 & 43.11 & OOM & 56.70 & OOM & 49.45 \\

& PPRGo & 31.36 & 45.83 & 74.33 & 39.43 & \cellcolor{mycolor!90} 91.13 & 84.30 & 46.91 & 59.27 & 58.19 & \cellcolor{mycolor!72} 95.09 & 86.72 & 43.11 & 30.22 & 58.51 & \cellcolor{mycolor!72} 84.16 & 55.63\\

& GRAND+ & 30.20 & 45.54 & 73.80 & 39.37 & \cellcolor{mycolor!54}90.85 & 83.70 & 45.11 & 59.27& 58.19 & 94.82& 85.95 & 43.11 & 30.23 & 56.49 & 83.38 & 53.20\\
\hline
GTs & GT
& 34.19 & \cellcolor{mycolor!90} 64.83 & 68.60 & 54.00 & OOM & 85.41 & 61.91
& 56.32 & 54.64 & OOM & 86.59 & 23.23 & OOM & 77.77 & OOM & \cellcolor{mycolor!18} 73.52 \\

& Gophormer
& 34.90 & 53.16 & 31.48 & \cellcolor{mycolor!54} 71.00 & \cellcolor{mycolor!36} 90.60 & 49.18 & \cellcolor{mycolor!36} 67.13
& 54.98 & \cellcolor{mycolor!54} 62.62 & \cellcolor{mycolor!54}  94.98 & 40.40 & \cellcolor{mycolor!36} 79.83 & 37.44 & \cellcolor{mycolor!72} 81.49 & 83.34 & \cellcolor{mycolor!90} 81.95 \\

& Graphormer
& \cellcolor{mycolor!18} 35.01 & 50.70 & 33.18 & \cellcolor{mycolor!72} 73.20 & OOM & 35.89 & 51.49
& \cellcolor{mycolor!90} 81.48 & \cellcolor{mycolor!90} 72.14 & 89.02 & OOM & \cellcolor{mycolor!54} 85.57& \cellcolor{mycolor!36} 40.55 & 68.62 & OOM & 71.33 \\

& SAN
& \cellcolor{mycolor!54} 36.04 & 59.22 & 73.37 & 46.00 & OOM & 85.07 & \cellcolor{mycolor!54} 67.55
& 59.27 & 57.27 & \cellcolor{mycolor!36} 94.85 & OOM & 29.40 & 39.00 & \cellcolor{mycolor!54} 80.00 & OOM & \cellcolor{mycolor!54} 80.63\\

& GraphGPS
& \cellcolor{mycolor!36} 35.19 & \cellcolor{mycolor!36} 60.78 & \cellcolor{mycolor!18} 75.23 & 44.00 & OOM & 86.65 & 61.38
& 55.11 & 57.41 & 94.64 & OOM & 46.74 & \cellcolor{mycolor!18} 40.04 & 69.68 & OOM & 72.66 \\

& NAG
& 31.68 & 40.38 & 74.72 & 46.17 & 88.66 & \cellcolor{mycolor!18} 86.75 & 58.94
& \cellcolor{mycolor!72} 77.24 & \cellcolor{mycolor!72} 70.99 & 93.94 & \cellcolor{mycolor!54} 87.26 & 49.89 & 25.80 & 63.30 & 82.40 & 61.25\\

& Specformer
& 30.74 & 42.29 & 70.33 & 47.69 & 88.05 & 82.49 & 54.00 & 54.32 & 53.77 & 89.22 & 85.58 & 45.77 & 31.44 & 63.26 & 71.82 & 68.90 \\

& Exphormer
& \cellcolor{mycolor!90} 36.62 & \cellcolor{mycolor!18} 59.36 & 68.34 & \cellcolor{mycolor!36} 65.63 & 89.22 & 78.17 & \cellcolor{mycolor!72} 68.19
& \cellcolor{mycolor!54} 61.80 & 57.31 & 94.25 & \cellcolor{mycolor!18} 86.97 & \cellcolor{mycolor!18} 76.31 & \cellcolor{mycolor!54} 42.31 & \cellcolor{mycolor!18} 79.47 & \cellcolor{mycolor!36}  83.84 & \cellcolor{mycolor!36} 80.39\\

& SGFormer
& 31.06 & 42.37 & 69.91 & 44.61 & 85.40 & 81.41 & 55.19 & 55.19 & 52.53 & 87.93 & 85.07 & 42.80 & 32.86 & 61.61 & 67.03 & 63.86 \\

& VCR & 26.41 & 27.28 & 71.15 & 28.13 & 82.45 & 85.41 & 44.89 & \cellcolor{mycolor!36} 61.64 & 58.32 & 88.45 & 83.25 & 44.66 & 22.48 & 54.68 & 78.46 & 51.33 \\

& PolyFormer
& 32.12 & 46.83 & 72.13 & 48.91 & 89.60 & 83.51 & 59.02 & 58.11 & 56.15 & 91.93 & 86.82 & 47.29 & 34.42 & 64.38 & 73.18 & 69.77 \\
\hline
\textbf{Ours} & Best-Validated
& \cellcolor{mycolor!72} 36.46 & \cellcolor{mycolor!54} 63.65 & \cellcolor{mycolor!36} 75.28 & \cellcolor{mycolor!90} 78.60 & \cellcolor{mycolor!72} 90.87 & \cellcolor{mycolor!36} 87.29 & \cellcolor{mycolor!90} 74.36 & \cellcolor{mycolor!18} 61.32 & \cellcolor{mycolor!36} 62.25 & \cellcolor{mycolor!18} 94.84 & \cellcolor{mycolor!90} 88.70 & \cellcolor{mycolor!90} 89.57 & \cellcolor{mycolor!90} 47.31 & \cellcolor{mycolor!90} 82.31 & \cellcolor{mycolor!18} 83.54 & \cellcolor{mycolor!72} 81.02 \\
\cline{2-18}
\rowcolor{yellow!50} \cellcolor{white}& \cellcolor{white}CR-MAS & 31.52 & 62.82 & 50.63 & 54.04 & 89.44 & 79.81 & 54.04 & 59.18 & 61.80 & 93.73 & 87.53 & 54.37 & 45.72 & 65.43 & 82.65 & 66.56 \\
& \textit{Metric X} 
& 7 / 16 & \cellcolor{gray!25} \textbf{14 / 16} & 2 / 16 & \cellcolor{gray!25}\textbf{11 / 16} & 5 / 11 & 3 / 16 & 7 / 16 & 6 / 16 & \cellcolor{gray!25}\textbf{13 / 16} & 5 / 15 & \cellcolor{gray!25}\textbf{13 / 13} & \cellcolor{gray!25}\textbf{12 / 16} & \cellcolor{gray!25}\textbf{14 / 14} & \cellcolor{gray!25}\textbf{9 / 16} & 5 / 11 & 7 / 16 \\
& \textit{Metric Y} 
& -0.89 & \cellcolor{gray!25}\textbf{11.64} & -17.12 & \cellcolor{gray!25}\textbf{3.11} & \cellcolor{gray!25}\textbf{0.65} & \cellcolor{gray!25}\textbf{0.35} & -1.93 & -1.65 & \cellcolor{gray!25}\textbf{2.72} & \cellcolor{gray!25}\textbf{0.88} & \cellcolor{gray!25}\textbf{4.76} & \cellcolor{gray!25}\textbf{2.28} & \cellcolor{gray!25}\textbf{10.63} & -1.47 & \cellcolor{gray!25}\textbf{3.01} & \cellcolor{gray!25}\textbf{0.42} \\
\cline{2-18}
\rowcolor{yellow!50} \cellcolor{white}& \cellcolor{white}CR-TAS 
& 34.64 & 63.65 & 51.84 & 78.60 & 89.33 & 81.02 & 55.74 & 61.32 & 60.74 & 93.95 & 87.75 & 85.94 & 47.31 & 65.43 & 83.04 & 70.39 \\
& \textit{Metric X} 
& \cellcolor{gray!25}\textbf{11 / 16} & \cellcolor{gray!25}\textbf{14 / 16} & 2 / 16 & \cellcolor{gray!25}\textbf{16 / 16} & 5 / 11 & 3 / 16 & \cellcolor{gray!25}\textbf{8 / 16} & \cellcolor{gray!25}\textbf{12 / 16} & \cellcolor{gray!25}\textbf{13 / 16} & 7 / 15 & \cellcolor{gray!25}\textbf{13 / 13} & \cellcolor{gray!25}\textbf{15 / 16} & \cellcolor{gray!25}\textbf{14 / 14} & \cellcolor{gray!25}\textbf{9 / 16} & 5 / 11 & \cellcolor{gray!25}\textbf{9 / 16} \\
& \textit{Metric Y} 
& \cellcolor{gray!25}\textbf{2.23} & \cellcolor{gray!25}\textbf{12.47} & -15.91 & \cellcolor{gray!25}\textbf{27.67} & \cellcolor{gray!25}\textbf{0.54} & \cellcolor{gray!25}\textbf{1.56} & -0.23 & \cellcolor{gray!25}\textbf{0.49} & \cellcolor{gray!25}\textbf{1.66} & \cellcolor{gray!25}\textbf{1.10} & \cellcolor{gray!25}\textbf{4.98} & \cellcolor{gray!25}\textbf{33.85} & \cellcolor{gray!25}\textbf{12.22} & -1.47 & \cellcolor{gray!25}\textbf{3.40} & \cellcolor{gray!25}\textbf{4.25} \\
\hline
\end{tabular}
\makebox[\linewidth][c]{%
    \textit{Metric X}: \# baselines outperformed by \textsc{CR-Graphormer} (excl. OOMs, \textbf{bold} $\geq 50\%$); \;
    \textit{Metric Y}: \textsc{CR-Graphormer} accuracy minus baseline mean (\textbf{bold} $\geq 0$)
}
\label{tab:numerical2}
\end{table*}

From Table~\ref{tab:numerical2}, we observe that \textsc{CR-Graphormer} is competitive with the $16$ baselines: CR-MAS and CR-TAS outperform at least half of them on $7/16$ and $11/16$ graphs ($6/12$ and $8/12$ in the target regimes), and exceed the mean baseline accuracy on 11 and 13 datasets, respectively. CR-MAS surpasses the best baseline on \textit{squirrel} and \textit{pubmed} by 0.31 and 0.20 percentage points; CR-TAS does so on \textit{community}, \textit{squirrel}, and \textit{pubmed} by 5.40, 1.90, and 0.42 percentage points, respectively. CR-TAS is also the best validation-selected cascade model on \textit{chameleon}, \textit{community}, \textit{cycle}, and \textit{squirrel}, and achieves stronger overall performance than CR-MAS.

%\My{When compared with the 16 baselines, \textsc{CR-Graphormer} shows competitive results: CR-MAS and CR-TAS perform better than at least half of the baselines on $7 / 16$ and $11 / 16$ graphs ($6 / 12$ and $8 / 12$ graphs in the target regimes), respectively. Notably, both \textsc{CR-Graphormer} models outperform $14/14$ baselines on \textit{squirrel}, $13/13$ baselines on \textit{pubmed}, $14 / 16$ baselines on \textit{chameleon}, and $13 / 16$ baselines on \textit{grid}. Additionally, CR-MAS outperforms $12/16$ baselines on \textit{shape} and $11/16$ baselines on \textit{community}, while CR-TAS outperforms $15/16$ baselines on \textit{shape} and $16/16$ baselines on \textit{community} plus $12 / 16$ baselines on \textit{cycle} and $11 / 16$ baselines on \textit{actor}. Moreover, CR-MAS outperforms the baseline means on 11 datasets, exceeding the highest baseline accuracy on \textit{squirrel} (by $0.31\%$) and \textit{pubmed} (by $0.20\%$). Similarly, CR-TAS outperforms the baseline means on 13 datasets, surpassing the highest baseline accuracy on \textit{community} (by $5.40\%$), \textit{squirrel} (by $1.90\%$), and \textit{pubmed} (by $0.42\%$), and is the best validation-selected cascade model on 4 datasets (\textit{chameleon}, \textit{community}, \textit{cycle}, and \textit{squirrel}). Overall, CR-TAS consistently yields superior performance over CR-MAS with higher baseline outperformance counts and larger accuracy margins.}

\section{Conclusion} 
We introduce \textsc{Graph Cascades}, a contagion-based mesoscopic rewiring procedure that constructs sparse auxiliary graphs for both message passing and attention. The rewiring improves structural resolution under reinforced multi-hop support. We provide theory connecting this to homophily and effective resistance, clarifying how rewiring can improve label agreement while favoring mesoscopically well-supported neighbors. The cascade-rewired neighborhoods also instantiate \textsc{CR-Graphormer}, a sparse transformer whose token sequences are pre-selected by contagion reinforcement; this demonstrates that the auxiliary graph alone is sufficient to drive a sparse-attention GT at bounded sequence length. Across datasets, \textsc{Graph Cascades} improves node classification for multiple GNN and GT backbones (with the largest improvements on heterophilic and moderate- to high-degree homophilic graphs). A regression study further demonstrates that performance aligns strongly with structural properties after rewiring, suggesting that the rewired structure carries the predictive signal.

\textbf{Limitations.} \textsc{Graph Cascades} is less effective on highly regular, low-degree, or bottlenecked graphs, such as cycles, grids, and chain-like benchmarks including \textit{Roman-empire}. In these settings, structural homogeneity and sparse connectivity limit the redundant short-path support needed for contagion-based rewiring. Theorem~\ref{thm:walklength} formalizes this, while Appendix~\ref{app:edge-connectivity} shows that fixed-threshold TAS-$\kappa$ can add a nonlocal edge only when the target is already $\kappa$-edge-connected to the seed's closed neighborhood. Thus, when graphs are dominated by narrow cuts rather than cohesive mesoscopic regions, bottleneck-targeted rewiring methods may be preferable; combining these approaches is a natural direction for future work. As \textsc{Graph Cascades} is topology-only by design, integrating cascade-based reinforcement with feature-aware edge selection is left to future work. %A second limitation is that \textsc{Graph Cascades} is topology-only: on graphs with strong and reliable node features, feature-aware rewiring can exploit signals unavailable to our method.  A second limitation is that \textsc{Graph Cascades} is topology-only by design, making it applicable when features are noisy or scarce; on graphs with strong node features, feature-aware methods such as feature-similarity rewiring or joint feature--structure denoising can exploit signals our construction does not use.

\paragraph{Acknowledgments.} This research was supported in part by the Swedish Research Council project ExCLUS ($2024-05603$) and the Wallenberg AI, Autonomous Systems and Software Program (WASP), funded by the Knut and Alice Wallenberg Foundation. 
\bibliographystyle{plainnat}
\bibliography{example_paper}

@article{gasteiger2019diffusion,
  title={Diffusion improves graph learning},
  author={Gasteiger, Johannes and Wei{\ss}enberger, Stefan and G{\"u}nnemann, Stephan},
  journal={Advances in neural information processing systems},
  volume={32},
  year={2019}
}

@article{chaitanyaadjacency,
  title={{Adjacency Search Embeddings}},
  author={Chaitanya, Meher and Jaglan, Kshitijaa and Brandes, Ulrik},
  journal={Transactions on Machine Learning Research},
  year={2025}
}

@book{centola-book,
  title={How Behavior Spreads: The Science of Complex Contagions},
  author={Centola, Damon},
  year={2018},
  publisher={Princeton University Press}
}

@inproceedings{bospecformer,
  title={Specformer: Spectral {G}raph {N}eural {N}etworks {M}eet {T}ransformers},
  author={Bo, Deyu and Shi, Chuan and Wang, Lele and Liao, Renjie},
  booktitle={The Eleventh International Conference on Learning Representations},
  year={2023}
}

@article{wu2023sgformer,
  title={{SGFormer: Simplifying and Empowering Transformers for Large-Graph Representations}},
  author={Wu, Qitian and Zhao, Wentao and Yang, Chenxiao and Zhang, Hengrui and Nie, Fan and Jiang, Haitian and Bian, Yatao and Yan, Junchi},
  journal={Advances in Neural Information Processing Systems},
  pages={64753--64773},
  year={2023}
}

@inproceedings{ma2024polyformer,
  title={{Polyformer: Scalable Node-Wise Filters via Polynomial Graph Transformer}},
  author={Ma, Jiahong and He, Mingguo and Wei, Zhewei},
  booktitle={Proceedings of the 30th ACM SIGKDD Conference on Knowledge Discovery and Data Mining},
  pages={2118--2129},
  year={2024}
}

@inproceedings{
zhang2023rethinking,
title={{Rethinking the Expressive Power of {GNN}s via Graph Biconnectivity}},
author={Bohang Zhang and Shengjie Luo and Liwei Wang and Di He},
booktitle={The International Conference on Learning Representations },
year={2023}
}

@article{granovetter1978threshold,
  title={{Threshold Models of Collective Behavior}},
  author={Granovetter, Mark},
  journal={American Journal of Sociology},
  volume={83},
  number={6},
  pages={1420--1443},
  year={1978},
  publisher={University of Chicago Press}
}

@article{lim2022sign,
  title={{Sign and Basis Invariant Networks for Spectral Graph Representation Learning}},
  author={Lim, Derek and Robinson, Joshua and Zhao, Lingxiao and Smidt, Tess and Sra, Suvrit and Maron, Haggai and Jegelka, Stefanie},
  journal={The International Conference on Learning Representations },
   year={2023}
}

@inproceedings{
    kipf17-classifgcnn,
    title={{Semi-Supervised Classification with Graph Convolutional Networks}},
    author={Thomas N. Kipf and Max Welling},
    booktitle={International Conference on Learning Representations},
    year={2017}
}

@article{centola_and_macy,
  title = {{Complex Contagions and the Weakness of Long Ties}},
  author = {Centola, Damon and Macy, Michael},
  journal = {American Journal of Sociology},
  volume = 113,
  number = 3,
  pages = {702--734},
  year = 2007,
  publisher = {The University of Chicago Press}
}

@article{hamilton2017inductive,
  title={{Inductive Representation Learning on Large Graphs}},
  author={Hamilton, Will and Ying, Zhitao and Leskovec, Jure},
  journal={Advances in Neural Information Processing Systems},
  volume={30},
  year={2017}
}

@article{alon2020bottleneck,
title={{On the Bottleneck of Graph Neural Networks and its Practical Implications}},
author={Uri Alon and Eran Yahav},
journal={International Conference on Learning Representations},
year={2021}
}

@inproceedings{spielman-srivastava,
  title     = {{Graph Sparsification by Effective Resistances}},
  author    = {Daniel A. Spielman and Nikhil Srivastava},
  booktitle = {Proceedings of the 40th Annual ACM Symposium on Theory of Computing},
  pages     = {563--568},
  year      = 2008
}

@article{zhao2021gophormer,
  title={{Gophormer: Ego-Graph Transformer for Node Classification}},
  author={Zhao, Jianan and Li, Chaozhuo and Wen, Qianlong and Wang, Yiqi and Liu, Yuming and Sun, Hao and Xie, Xing and Ye, Yanfang},
  journal={arXiv preprint arXiv:2110.13094},
  year={2021}
}

@article{kreuzer2021rethinking,
  title={{Rethinking Graph Transformers with Spectral Attention}},
  author={Kreuzer, Devin and Beaini, Dominique and Hamilton, Will and L{\'e}tourneau, Vincent and Tossou, Prudencio},
  journal={Advances in Neural Information Processing Systems},
  volume={34},
  pages={21618--21629},
  year={2021}
}

@article{rampavsek2022recipe,
  title={{Recipe for a General, Powerful, Scalable Graph Transformer}},
  author={Ramp{\'a}{\v{s}}ek, Ladislav and Galkin, Michael and Dwivedi, Vijay Prakash and Luu, Anh Tuan and Wolf, Guy and Beaini, Dominique},
  journal={Advances in Neural Information Processing Systems},
  pages={14501--14515},
  year={2022}
}

@inproceedings{fu2024vcr,
  title={{VCR-Graphormer: A Mini-Batch Graph Transformer via Virtual Connections}},
  author={Fu, Dongqi and Hua, Zhigang and Xie, Yan and Fang, Jin and Zhang, Si and Sancak, Kaan and Wu, Hao and Malevich, Andrey and He, Jingrui and Long, Bo},
  booktitle={The International Conference on Learning Representations},
  year={2024}
}

@inproceedings{chen2022structure,
  title={{Structure-Aware Transformer for Graph Representation Learning}},
  author={Chen, Dexiong and O’Bray, Leslie and Borgwardt, Karsten},
  booktitle={International conference on machine learning},
  year={2022},
  organization={PMLR}
}

@inproceedings{kim2022pure,
  title     = {{Pure Transformers are Powerful Graph Learners}},
  author    = {Kim, Jinwoo and Nguyen, Tien Dat and Min, Seonwoo and Cho, Sungjun and Lee, Moontae and Lee, Honglak and Hong, Seunghoon},
  booktitle = {Advances in Neural Information Processing Systems},
  volume    = {35},
  year      = {2022}
}

@inproceedings{morris2021weisfeiler,
  title     = {{Weisfeiler and Leman Go Neural: Higher-Order Graph Neural Networks}},
  author    = {Morris, Christopher and Ritzert, Martin and Fey, Matthias and Hamilton, William L. and Lenssen, Jan Eric and Rattan, Gaurav and Grohe, Martin},
  booktitle = {Proceedings of the AAAI Conference on Artificial Intelligence},
  volume    = {33},
  year      = {2019}
}

@article{dwivedi2021benchmarking,
  title     = {{Benchmarking Graph Neural Networks}},
  author    = {Dwivedi, Vijay Prakash and Joshi, Chaitanya K. and Luu, Anh Tuan and Laurent, Thomas and Bengio, Yoshua and Bresson, Xavier},
  journal   = {Journal of Machine Learning Research},
  pages     = {1--48},
  year      = {2023}
}

@inproceedings{choromanski2021rethinking,
  title={{Rethinking Attention with Performers}},
  author={Choromanski, Krzysztof Marcin and Likhosherstov, Valerii and Dohan, David and Song, Xingyou and Gane, Andreea and Sarlos, Tamas and Hawkins, Peter and Davis, Jared Quincy and Mohiuddin, Afroz and Kaiser, Lukasz and Belanger, David Benjamin and Colwell, Lucy and Weller, Adrian},
  booktitle={International Conference on Learning Representations},
  year={2021}
}

@article{wu2022nodeformer,
  title={{NodeFormer: A Scalable Graph Structure Learning Transformer for Node Classification}},
  author={Wu, Qitian and Zhao, Wentao and Li, Zenan and Wipf, David P and Yan, Junchi},
  journal={Advances in Neural Information Processing Systems},
  pages={27387--27401},
  year={2022}
}

@inproceedings{kong2023goat,
  title={{GOAT: A Global Transformer on Large-Scale Graphs}},
  author={Kong, Kezhi and Chen, Jiuhai and Kirchenbauer, John and Ni, Renkun and Bruss, C Bayan and Goldstein, Tom},
  booktitle={International Conference on Machine Learning},
  pages={17375--17390},
  year={2023},
  organization={PMLR}
}

@inproceedings{shirzad2023exphormer,
  title={{Exphormer: Sparse Transformers for Graphs}},
  author={Shirzad, Hamed and Velingker, Ameya and Venkatachalam, Balaji and Sutherland, Danica J and Sinop, Ali Kemal},
  booktitle={International Conference on Machine Learning},
  pages={31613--31632},
  year={2023},
  organization={PMLR}
}

@article{grotschla2024benchmarking,
  title={{Benchmarking Positional Encodings for GNNs and Graph Transformers}},
  author={Gr{\"o}tschla, Florian and Xie, Jiaqing and Wattenhofer, Roger},
  journal={arXiv preprint arXiv:2411.12732},
  year={2024}
}

@article{vaswani2017attention,
  title={{Attention Is All You Need}},
  author={Vaswani, Ashish and Shazeer, Noam and Parmar, Niki and Uszkoreit, Jakob and Jones, Llion and Gomez, Aidan N and Kaiser, {\L}ukasz and Polosukhin, Illia},
  journal={Advances in Neural Information Processing Systems},
  year={2017}
}

@article{velivckovic2017graph,
  title={{Graph {A}ttention {N}etworks}},
   author={Veli{\v{c}}kovi{\'c}, Petar and Cucurull, Guillem and Casanova, Arantxa and Romero, Adriana and Lio, Pietro and Bengio, Yoshua},
journal={International Conference on Learning Representations},
year={2018}
}

@inproceedings{agrawal2024no,
  title={{No Prejudice! Fair Federated Graph Neural Networks for Personalized Recommendation}},
  author={Agrawal, Nimesh and Sirohi, Anuj Kumar and Kumar, Sandeep and others},
  booktitle={Proceedings of the AAAI Conference on Artificial Intelligence},
  volume={38},
  year={2024}
}

@inproceedings{chen2020measuring,
  title={{Measuring and Relieving the Over-Smoothing Problem for Graph Neural Networks from the Topological View}},
  author={Chen, Deli and Lin, Yankai and Li, Wei and Li, Peng and Zhou, Jie and Sun, Xu},
  booktitle={Proceedings of the AAAI Conference on Artificial Intelligence},
  volume={34},
  year={2020}
}

@inproceedings{kong2024spatio,
  title={{Spatio-Temporal Pivotal Graph Neural Networks for Traffic Flow Forecasting}},
  author={Kong, Weiyang and Guo, Ziyu and Liu, Yubao},
  booktitle={Proceedings of the AAAI Conference on Artificial Intelligence},
  volume={38},
  year={2024}
}

@inproceedings{deng2021graph,
  title={{Graph Neural Network-Based Anomaly Detection in Multivariate Time Series}},
  author={Deng, Ailin and Hooi, Bryan},
  booktitle={Proceedings of the AAAI Conference on Artificial Intelligence},
  volume={35},
  year={2021}
}

@article{mialon2021graphit,
  title={{Graphit: Encoding Graph Structure in Transformers}},
  author={Mialon, Gr{\'e}goire and Chen, Dexiong and Selosse, Margot and Mairal, Julien},
  journal={arXiv preprint arXiv:2106.05667},
  year={2021}
}

@inproceedings{zhou2025tokenphormer,
  title={{Tokenphormer: Structure-Aware Multi-Token Graph Transformer for Node Classification}},
  author={Zhou, Zijie and Lu, Zhaoqi and Wei, Xuekai and Chen, Rongqin and Zhang, Shenghui and Ip, Pak Lon and others},
  booktitle={Proceedings of the AAAI Conference on Artificial Intelligence},
  pages={13428--13436},
  year={2025}
}

@article{bi2024make,
  title={Make {H}eterophilic {G}raphs {B}etter {F}it {G}{N}{N}: {A} {G}raph {R}ewiring {A}pproach}, 
  author={Bi, Wendong and Du, Lun and Fu, Qiang and Wang, Yanlin and Han, Shi and Zhang, Dongmei},
  journal={IEEE Transactions on Knowledge and Data Engineering},
  year={2024},
  publisher={IEEE}
}

@inproceedings{guo2023homophily,
  title={{Homophily-oriented Heterogeneous Graph Rewiring}},
  author={Guo, Jiayan and Du, Lun and Bi, Wendong and Fu, Qiang and Ma, Xiaojun and Chen, Xu and Han, Shi and Zhang, Dongmei and Zhang, Yan},
  booktitle={Proceedings of the ACM web conference 2023},
  pages={511--522},
  year={2023}
}

@article{ying2021transformers,
  title={{Do Transformers Really Perform Badly for Graph Representation?}},
  author={Ying, Chengxuan and Cai, Tianle and Luo, Shengjie and Zheng, Shuxin and Ke, Guolin and He, Di and Shen, Yanming and Liu, Tie-Yan},
  journal={Advances in Neural Information Processing Systems},
  pages={28877--28888},
  year={2021}
}

@inproceedings{chennagphormer,
  title={{NAGphormer: A Tokenized Graph Transformer for Node Classification in Large Graphs}},
  author={Chen, Jinsong and Gao, Kaiyuan and Li, Gaichao and He, Kun},
  booktitle={The International Conference on Learning Representations},
  year={2023}
}

@inproceedings{feng2022grand+,
  title={Grand+: Scalable graph random neural networks},
  author={Feng, Wenzheng and Dong, Yuxiao and Huang, Tinglin and Yin, Ziqi and Cheng, Xu and Kharlamov, Evgeny and Tang, Jie},
  booktitle={Proceedings of the ACM Web Conference 2022},
  pages={3248--3258},
  year={2022}
}

@inproceedings{bojchevski2020scaling,
  title={Scaling graph neural networks with approximate pagerank},
  author={Bojchevski, Aleksandar and Gasteiger, Johannes and Perozzi, Bryan and Kapoor, Amol and Blais, Martin and R{\'o}zemberczki, Benedek and Lukasik, Michal and G{\"u}nnemann, Stephan},
  booktitle={Proceedings of the 26th ACM SIGKDD international conference on knowledge discovery \& data mining},
  pages={2464--2473},
  year={2020}
}

@article{dwivedi2020generalization,
  title={{A Generalization of Transformer Networks to Graphs}},
  author={Dwivedi, Vijay Prakash and Bresson, Xavier},
  journal={arXiv preprint arXiv:2012.09699},
  year={2020}
}

@inproceedings{
topping2022understanding,
title={Understanding over-squashing and bottlenecks on graphs via curvature},
author={Jake Topping and Francesco Di Giovanni and Benjamin Paul Chamberlain and Xiaowen Dong and Michael M. Bronstein},
booktitle={International Conference on Learning Representations},
year={2022}
}

@article{arnaiz2022diffwire,
  title = 	 {{DiffWire: Inductive Graph Rewiring via the Lov{\'a}sz Bound}},
  author =       {Arnaiz-Rodr{\'i}guez, Adri{\'a}n and Begga, Ahmed and Escolano, Francisco and Oliver, Nuria M},
  journal = 	 {Proceedings of the First Learning on Graphs Conference},
  year = 	 {2022},
  publisher =    {PMLR}
}

@inproceedings{nguyen2023revisiting,
  title={{Revisiting Over-Smoothing and Over-Squashing using Ollivier-Ricci Curvature}},
  author={Nguyen, Khang and Hieu, Nong Minh and Nguyen, Vinh Duc and Ho, Nhat and Osher, Stanley and Nguyen, Tan Minh},
  booktitle={International Conference on Machine Learning},
  pages={25956--25979},
  year={2023},
  organization={PMLR}
}

@inproceedings{fesser2024mitigating,
  title={{Mitigating Over-Smoothing and Over-Squashing using Augmentations of Forman-Ricci Curvature}},
  author={Fesser, Lukas and Weber, Melanie},
  booktitle={Learning on Graphs Conference},
  pages={19--1},
  year={2024},
  organization={PMLR}
}

@inproceedings{
karhadkar2023fosr,
title={Fo{SR}: First-order spectral rewiring for addressing oversquashing in {GNN}s},
author={Kedar Karhadkar and Pradeep Kr. Banerjee and Guido Montufar},
booktitle={The International Conference on Learning Representations },
year={2023}
}

@article{luan2022revisiting,
  title={Revisiting heterophily for graph neural networks},
  author={Luan, Sitao and Hua, Chenqing and Lu, Qincheng and Zhu, Jiaqi and Zhao, Mingde and Zhang, Shuyuan and Chang, Xiao-Wen and Precup, Doina},
  journal={Advances in neural information processing systems},
  volume={35},
  pages={1362--1375},
  year={2022}
}

@article{zhu2020beyond,
  title={Beyond homophily in graph neural networks: Current limitations and effective designs},
  author={Zhu, Jiong and Yan, Yujun and Zhao, Lingxiao and Heimann, Mark and Akoglu, Leman and Koutra, Danai},
  journal={Advances in neural information processing systems},
  volume={33},
  pages={7793--7804},
  year={2020}
}

@inproceedings{li2022finding,
  title={Finding global homophily in graph neural networks when meeting heterophily},
  author={Li, Xiang and Zhu, Renyu and Cheng, Yao and Shan, Caihua and Luo, Siqiang and Li, Dongsheng and Qian, Weining},
  booktitle={International conference on machine learning},
  pages={13242--13256},
  year={2022},
  organization={PMLR}
}

@article{liang2024sign,
  title={Sign is not a remedy: Multiset-to-multiset message passing for learning on heterophilic graphs},
  author={Liang, Langzhang and Kim, Sunwoo and Shin, Kijung and Xu, Zenglin and Pan, Shirui and Qi, Yuan},
  journal={arXiv preprint arXiv:2405.20652},
  year={2024}
}

@article{song2023ordered,
  title={Ordered gnn: Ordering message passing to deal with heterophily and over-smoothing},
  author={Song, Yunchong and Zhou, Chenghu and Wang, Xinbing and Lin, Zhouhan},
  journal={arXiv preprint arXiv:2302.01524},
  year={2023}
}

@inproceedings{barberolocality,
  title={Locality-Aware Graph Rewiring in GNNs},
  author={Barbero, Federico and Velingker, Ameya and Saberi, Amin and Bronstein, Michael M and Di Giovanni, Francesco},
  booktitle={The Twelfth International Conference on Learning Representations},
  year={2024}
}

@inproceedings{rubiognns,
  title={GNNs Getting ComFy: Community and Feature Similarity Guided Rewiring},
  author={Rubio-Madrigal, Celia and Jamadandi, Adarsh and Burkholz, Rebekka},
  booktitle={The Thirteenth International Conference on Learning Representations},
  year={2025}
}

@inproceedings{linkerhagnerjoint,
  title={Joint Graph Rewiring and Feature Denoising via Spectral Resonance},
  author={Linkerh{\"a}gner, Jonas and Shi, Cheng and Dokmani{\'c}, Ivan},
  booktitle={The Thirteenth International Conference on Learning Representations},
  year={2025}
}

@inproceedings{attali2024delaunay,
  title={Delaunay graph: Addressing over-squashing and over-smoothing using delaunay triangulation},
  author={Attali, Hugo and Buscaldi, Davide and Pernelle, Nathalie},
  booktitle={Forty-first International Conference on Machine Learning},
  year={2024}
}

@inproceedings{attali2025dynamic,
  title={Dynamic Triangulation-Based Graph Rewiring for Graph Neural Networks},
  author={Attali, Hugo and Papastergiou, Thomas and Pernelle, Nathalie and Malliaros, Fragkiskos D},
  booktitle={Proceedings of the 34th ACM International Conference on Information and Knowledge Management},
  pages={87--97},
  year={2025}
}

@inproceedings{yan2022two,
  title={Two sides of the same coin: Heterophily and oversmoothing in graph convolutional neural networks},
  author={Yan, Yujun and Hashemi, Milad and Swersky, Kevin and Yang, Yaoqing and Koutra, Danai},
  booktitle={2022 IEEE International Conference on Data Mining (ICDM)},
  pages={1287--1292},
  year={2022},
  organization={IEEE}
}

@inproceedings{zhengunderstanding,
  title={Understanding and Enhancing Message Passing on Heterophilic Graphs via Compatibility Matrix},
  author={Zheng, Zhuonan and Bei, Yuanchen and Zhou, Zhiyao and Zhou, Sheng and Ma, Yao and Gu, Ming and Xu, Hongjia and Chen, Jiawei and Bu, Jiajun},
  booktitle={The Thirty-ninth Annual Conference on Neural Information Processing Systems},
  year={2025}
}

@article{platonov2023critical,
  title={A critical look at the evaluation of GNNs under heterophily: Are we really making progress?},
  author={Platonov, Oleg and Kuznedelev, Denis and Diskin, Michael and Babenko, Artem and Prokhorenkova, Liudmila},
  journal={arXiv preprint arXiv:2302.11640},
  year={2023}
}

@article{dwivedi2022long,
  title={Long range graph benchmark},
  author={Dwivedi, Vijay Prakash and Ramp{\'a}{\v{s}}ek, Ladislav and Galkin, Michael and Parviz, Ali and Wolf, Guy and Luu, Anh Tuan and Beaini, Dominique},
  journal={Advances in Neural Information Processing Systems},
  volume={35},
  pages={22326--22340},
  year={2022}
}

%%%%%%%%%%%%%%%%%%%%%%%%%%%%%%%%%%%%%%%%%%%%%%%%%%%%%%%%%%%%
\newpage

\appendix

\begin{center}
    \Large \bfseries APPENDIX
\end{center}

\section{Preliminaries and Related Work}\label{discuss}
\subsection{Preliminaries} \label{app:prelim}
\textit{\textbf{Graph Neural Networks (GNNs).}}
Let $\mathbf H\in\mathbb R^{n\times d}$ be the node features of $G$, with row $\mathbf h_i\in\mathbb R^d$ for node $i$. A GNN processes $\mathbf H$ through message-passing layers that respect graph sparsity, updating each node's embedding by aggregating over its one-hop neighborhood: $$\mathbf h'_i = \phi\!\left(\mathbf h_i,\ \bigoplus_{j\in N(i)}\psi(\mathbf h_i,\mathbf h_j)\right),$$ where $\psi$ and $\phi$ are learnable and $\bigoplus$ is a permutation-invariant aggregator. Stacking $L$ layers expands the receptive field to $L$-hop neighborhoods at the cost of oversmoothing, oversquashing, and weak long-range interactions \citep{alon2020bottleneck}. Sparse variants restrict message passing to a subset of neighbors, equivalent to running the same update on a rewired graph with neighborhoods $N^{\diamond}(i)$ in place of $N(i)$.

\textit{\textbf{Graph Transformers (GTs).}}
GTs adopt the standard transformer block: an attention layer followed by a feed-forward operation. Dense variants \citep{ying2021transformers, mialon2021graphit, kreuzer2021rethinking, zhao2021gophormer} treat each node as a token attending to every other node, encoding graph topology through positional encodings such as shortest-path embeddings \citep{ying2021transformers, zhang2023rethinking} or Laplacian eigenvectors \citep{dwivedi2020generalization}. The single-head self-attention update over a node set $N^\diamond(i)$ is 
$$\mathbf h'_i = \sum_{j\in N^\diamond(i)} \alpha_{ij}\, \mathbf v_j, \qquad \alpha_{ij}=\frac{\exp(\mathbf q_i^\top \mathbf k_j/\sqrt{d'})} {\sum_{r\in N^\diamond(i)}\exp(\mathbf q_i^\top \mathbf k_r/\sqrt{d'})},$$
where $\mathbf Q=\mathbf H\mathbf W_Q$, $\mathbf K=\mathbf H\mathbf W_K$, $\mathbf V=\mathbf H\mathbf W_V$ are learnable projections with $\mathbf W_Q,\mathbf W_K,\mathbf W_V\in\mathbb R^{d\times d'}$. Setting $N^\diamond(i)=N(i)$ recovers one-hop attention as in GATs \citep{velivckovic2017graph}. Scalable sparse variants instead define $N^\diamond(\cdot)$ via diffusion- or range-limited neighborhoods (NAGphormer \citep{chennagphormer}, VCR-Graphormer \citep{fu2024vcr}) or fixed expander/landmark patterns (Exphormer \citep{shirzad2023exphormer}).

These architectures rely on either strictly local neighborhoods or sparsified approximations of global interactions, and do not explicitly encode intermediate-scale structure. Our proposed \textsc{Graph Cascades} provides a structural augmentation that induces mesoscopic connectivity patterns, enabling efficient modeling of interactions beyond one-hop neighborhoods in both GNN and GT architectures.

\subsection{Extended Related Work}\label{related_work_app}

Existing research on GTs can be broadly classified into three categories: 
(i) GT architectures, 
(ii) positional encodings for graphs, and 
(iii) graph rewiring techniques.

\textit{\textbf{GT Architectures.}}
GT models adapt the self-attention mechanism of the original transformer \citep{vaswani2017attention} to graph-structured data. The GT network \citep{dwivedi2020generalization} applies dense self-attention across all $\mathcal{O}(n^{2})$ node pairs, achieving strong performance on several benchmarks. Subsequent architectures introduce graph-specific inductive biases into the attention mechanism: 
Graphormer \citep{ying2021transformers} encodes node centrality and pairwise shortest-path distances; 
GraphiT \citep{mialon2021graphit} employs kernelized relative encodings and explicit sub-path features; and
SAN \citep{kreuzer2021rethinking} and Gophormer \citep{zhao2021gophormer} incorporate edge and node attributes. 
Although these models achieve strong performance across domains such as recommendation, question answering, and bioinformatics, their fully connected attention mechanisms incur quadratic computational cost.
 
To address scalability, a growing body of work focuses on lightweight and sparse-attention GT variants. 
Exphormer \citep{shirzad2023exphormer} imposes expander-graph sparsity and GOAT \citep{kong2023goat} projects node features into lower-dimensional spaces. 
NodeFormer \citep{wu2022nodeformer} introduces topology-aware relational biases, achieving provably linear complexity in the number of nodes.
A parallel line of work embeds graphs to confine attention to compact, information-rich subsets through sparse tokenization schemes.
TokenGT \citep{kim2022pure} represents nodes and edges as independent tokens with structural positional encodings, achieving competitive or superior performance to GNNs on molecular benchmarks.
Tokenphormer \citep{zhou2025tokenphormer} extends this idea by generating mixed random-walk, hop, and global SGPM tokens for each node, enabling applicability to both homogeneous and heterogeneous graphs; however, its multi-token aggregation lacks an explicit heterophily-aware mechanism and may introduce label noise when neighboring nodes exhibit dissimilar labels or attributes.
PolyFormer \citep{ma2024polyformer} introduces PolyAttn, an attention-based node-wise polynomial graph filter that avoids the expensive positional encodings used by prior node-wise spectral methods, improving scalability.
Building on PolyAttn, PolyFormer attends over per-node polynomial ``tokens'' to efficiently capture spectral information and reports strong performance on both homophilic and heterophilic node-level benchmarks, scaling to very large graphs.
SGFormer \citep{wu2023sgformer} argues that a deep, multi-head GT is not necessary for strong performance: a single-layer global attention model can be surprisingly competitive on node property prediction while being much cheaper to run.
It removes positional encodings and other preprocessing from its design and scales linearly with the number of nodes.
Specformer \citep{bospecformer} is a spectral GNN that replaces traditional scalar spectral filters with a learnable set-to-set spectral filter by encoding all Laplacian eigenvalues and applying self-attention in the spectral domain.
It also introduces a learnable-basis decoder to enable nonlocal graph convolution, while remaining permutation-equivariant, and reports strong gains on both node- and graph-level benchmarks.
VCR-Graphormer \citep{fu2024vcr} explicitly targets heterophily by introducing virtual connections and PPR-based token lists that encode both local and global dependencies while maintaining linear complexity per layer.
Nevertheless, its reliance on virtual connections and structural partitioning increases preprocessing overhead and complicates scalability.
\textit{In contrast, our label-free mesoscopic rewiring captures both structural and heterophily-driven relationships without requiring virtual connections or graph partitioning, yielding improved performance on heterophilic graphs, and maintaining competitive accuracy on homophilic ones.}

\textit{\textbf{Positional Encodings for Graphs.}}
Since graphs lack an intrinsic notion of order, GTs must rely on positional encodings to inject structural information. 
Laplacian-based encodings (LapPE) \citep{dwivedi2020generalization}, full-spectrum methods \citep{kreuzer2021rethinking}, and sign-invariant variants such as SignNet \citep{lim2022sign} assign each node absolute coordinates in a spectral space, but can be sensitive to eigenvector permutations.
Random-walk and diffusion-based encodings (RWSE, RWDiffusion, RRWP) \citep{dwivedi2021benchmarking, grotschla2024benchmarking} capture multi-scale proximity patterns, while shortest-path or distance encodings underpin Graphormer's attention biases \citep{ying2021transformers}.
Weisfeiler--Lehman subtree encodings (WL-PE) \citep{morris2021weisfeiler} and edge-aware learnable schemes such as PureGT's relative positional encodings \citep{kim2022pure} further enrich positional signals.
Systematic studies \citep{rampavsek2022recipe} demonstrate that combining complementary positional encodings (e.g., Laplacian and shortest-path encodings) yields robust gains across heterogeneous benchmarks. Overall, well-designed positional encodings are now recognized as a primary driver of GT performance, enabling discrimination of node roles, distances, and higher-order structural motifs.

\textit{\textbf{Graph Rewiring Techniques.}}
Beyond architectural design, graph rewiring has emerged as an effective strategy for improving learning on challenging graph structures.
HDHGR \citep{guo2023homophily} and DHGR \citep{bi2024make} rely on (pseudo-)labels to add homophilic edges and remove heterophilic ones, which can discard informative cross-class links and make performance dependent on label quality and availability.
\emph{In contrast, our mesoscopic rewiring approach is entirely label-free. It connects node pairs reinforced by multi-path short walks (Theorem~\ref{thm:walklength}) while pruning weak or noisy edges, yielding consistent gains on heterophilic graphs (Proposition~\ref{prop:bayes-monotone-matrix} and Theorem \ref{thm:sbm-profile-homophily}) and remains competitive on many homophilic graphs, with predictable failures on low-degree regular structures. Our method integrates seamlessly with GTs using sparse attention or as a standalone GT, without introducing virtual edges or requiring graph partitioning.} Curvature-based approaches \citep{topping2022understanding, nguyen2023revisiting, fesser2024mitigating} and differentiable or spectral rewiring methods \citep{arnaiz2022diffwire, karhadkar2023fosr} typically aim to increase algebraic connectivity to facilitate information flow.
\emph{By contrast, our approach does not raise global connectivity or introduce shortcuts between weakly connected regions; instead, it emphasizes walk-reinforced \emph{mesoscopic} ties that better capture both structural and heterophily-driven relationships.} PPRGo \citep{bojchevski2020scaling} is a scalable GNN that decouples local feature transformation from graph propagation and then propagates predictions using precomputed sparse approximations of personalized PageRank (PPR), avoiding expensive multi-hop message passing or per-step power iteration while keeping long-range information. It is not ``rewiring'' in the strict topology-editing sense (i.e., explicitly adding/removing edges in the original graph), but it does effectively replace the adjacency-based propagation with a sparsified PPR diffusion operator which is equivalent to working on an implicit new weighted neighborhood graph defined by top-PPR mass—so it is reasonable to describe it as an implicit diffusion-based rewiring mechanism. 
GRAND+ \citep{feng2022grand+} is a scalable version of GRAND that precomputes a (generalized) mixed-order propagation matrix using a generalized forward push (GFPush) algorithm, so it can do GRAND-style random propagation graph augmentations in mini-batches, and it adds a confidence-aware consistency loss to improve generalization. It is not a formal method, i.e., it does not edit the original edge set, but because it replaces adjacency-based propagation with a learned/selected propagation matrix (a diffusion operator), it effectively induces an implicit weighted ``rewired'' neighborhood structure for propagation/augmentation.

\textit{\textbf{Positioning Relative to Recent Rewiring Directions.}}
Recent rewiring methods can be organized along two axes: the structural objective (bottleneck mitigation vs. neighborhood realignment) and the information used (purely structural vs. feature- or label-aware). Curvature-based~\citep{topping2022understanding, nguyen2023revisiting, fesser2024mitigating} and spectral~\citep{karhadkar2023fosr, arnaiz2022diffwire} methods are structural and target spectral bottlenecks to alleviate oversquashing. Locality-aware rewiring~\citep{barberolocality} preserves local geometry while adding shortcuts. Feature-aware methods take a complementary route: Delaunay-based approaches~\citep{attali2024delaunay, attali2025dynamic} replace the graph via feature-space triangulation, joint spectral-resonance methods~\citep{linkerhagnerjoint} co-optimize feature denoising and rewiring, and community–feature similarity methods~\citep{rubiognns} combine structural and feature signals. These methods use strictly more information than the original graph topology and are most appropriate when features are reliable and informative.

\textsc{Graph Cascades} occupies a distinct position on both axes. It is purely structural and label-free, like \textsc{FoSR} \citep{karhadkar2023fosr} and SDRF \citep{topping2022understanding}, but its objective is neighborhood realignment via reinforced multi-hop co-activation rather than bottleneck mitigation. Theorem~\ref{thm:walklength} formalizes this: cascade activations are supported by $\kappa$ edge-disjoint short paths, so the rewiring concentrates on cohesive mesoscopic regions rather than spanning weak cuts. Empirically, this distinction is reflected in where cascades help most---heterophilic graphs whose label structure is carried by reinforced multi-hop neighborhoods---and where they do not (chain-like or grid-like structures, where bottleneck-targeted methods remain more appropriate). \textsc{Graph Cascades} is therefore best understood as complementary to existing rewiring directions rather than as a replacement for any of them.

\section{Proofs of Theoretical Results}

In this section, we provide proofs of some theoretical results introduced in the main body.

\subsection{Proof of Proposition~\ref{prop:bayes-monotone-matrix}}
\label{proof1}

By Bayes' rule,
\[
    h_G
    =
    \Pr{\mathcal S\mid \mathcal E}
    =
    \frac{\Pr{\mathcal E\mid \mathcal S}\Pr{\mathcal S}}
         {\Pr{\mathcal E}},
\]
and
\[
    h_\mathcal F
    =
    \Pr{\mathcal S\mid \mathcal F}
    =
    \frac{\Pr{\mathcal F\mid \mathcal S}\Pr{\mathcal S}}
         {\Pr{\mathcal F}}.
\]
By Assumption~\ref{as2},
\[
    \Pr{\mathcal F\mid \mathcal S}
    \ge
    \Pr{\mathcal E\mid \mathcal S}.
\]
Therefore,
\[
    h_\mathcal F
    \ge
    \frac{\Pr{\mathcal E\mid \mathcal S}\Pr{\mathcal S}}
         {\Pr{\mathcal F}}
    =
    h_G\frac{\Pr{\mathcal E}}{\Pr{\mathcal F}}.
\]
If Assumption~\ref{as3} also holds, then
\[
    \frac{\Pr{\mathcal E}}{\Pr{\mathcal F}}\ge 1,
\]
and so
\[
    h_\mathcal F\ge h_G.
\]

If Assumption~\ref{as2} is strict, then the first inequality above is strict,
so \(h_\mathcal F>h_G\) whenever Assumption~\ref{as3} holds. If
Assumption~\ref{as3} is strict and \(h_G>0\), then
\[ h_G\frac{\Pr{\mathcal E}}{\Pr{\mathcal F}}
    >
    h_G,
\]
and therefore, \(h_\mathcal F>h_G\).

\subsection{Proof of Theorem \ref{thm:sbm-profile-homophily}} \label{profile_proof}

We first state the lemma used in the proof for completeness.

\begin{lemma}[Chernoff/Bernstein Bounds for Poisson-Binomial Sums]
\label{lem:tails}
Let $X=\sum_i X_i$, where the $X_i$ are independent Bernoulli random variables
with mean $\mu=\E X$. Then, for $0\le t\le \mu$,
\[
\P{X-\mu\le -t}\le \exp\left(-\frac{t^2}{2\mu}\right),
\]
and, for $t\ge 0$,
\[
\P{X-\mu\ge t}\le \exp\left(-\frac{t^2}{2\mu+2t/3}\right).
\]
\end{lemma}

\begin{proof}
The lower-tail inequality follows from the multiplicative Chernoff bound
$\P{X\le (1-\delta)\mu}\le \exp(-\delta^2\mu/2)$ with $t=\delta\mu$.
The upper-tail inequality follows from Bernstein's inequality using
$\mathrm{Var}(X)\le \mu$ and $|X_i-\E X_i|\le 1$.
\end{proof}
    
We prove Theorem \ref{thm:sbm-profile-homophily} in five steps. 

\paragraph{Step 1: Population Identity.}
Fix two distinct nodes $u,v$ with $y_u=a$ and $y_v=b$. By definition,
\[
 s_2(u,v)=|N(u)\cap N(v)|=\sum_{x\in V\setminus\{u,v\}} A_{ux}A_{vx}.
\]
Conditional on the labels, the summands are independent across $x$, because
different $x$'s involve disjoint unordered edge variables. If $y_x=c$, then
\begin{align*}
\E{A_{ux}A_{vx}\mid y_u=a,y_v=b,y_x=c}
&=\E{A_{ux}\mid y_u=a,y_x=c}\,\E{A_{vx}\mid y_v=b,y_x=c}
\\
&=B_{ac}B_{bc}.
\end{align*}
Therefore,
\begin{align*}
\mu_{ab}^{(n)}
&=\sum_{c=1}^C\bigl(n_c-\ind{c=a}-\ind{c=b}\bigr)B_{ac}B_{bc}=n\sum_{c=1}^C\pi_c B_{ac}B_{bc}+O(1)\\
&=n(BD_\pi B^\top)_{ab}+O(1).
\end{align*}
Equivalently,
\[
\frac{1}{n}\mu_{ab}^{(n)}=(BD_\pi B^\top)_{ab}+O(1/n).
\]
Define the finite-sample profile margin and the maximum two-hop mean by
\[
\gamma_n:=\min_{a\in[C]}\left(\mu_{aa}^{(n)}-\max_{b\ne a}\mu_{ab}^{(n)}\right),
\qquad
\mu_{\max}:=\max_{a,b\in[C]}\mu_{ab}^{(n)}.
\]
By the population identity and $\Delta_2>0$,
\[
\gamma_n=n\Delta_2+O(1),\qquad \mu_{\max}=O(n).
\]
Thus, for all sufficiently large $n$,
\begin{equation}
\label{eq:finite-sample-conditions}
\gamma_n^2\ge 48\mu_{\max}\log n,
\qquad
\gamma_n\ge 16\log n.
\end{equation}
We assume from now on that $n$ is large enough for
\eqref{eq:finite-sample-conditions} to hold.

\paragraph{Step 2: Per-Class Midpoint Thresholds.}
For each class $a\in[C]$, define the midpoint threshold
\[
\kappa_a^\star:=\frac{\mu_{aa}^{(n)}+\max_{b\ne a}\mu_{ab}^{(n)}}{2}.
\]
Then
\[
\mu_{aa}^{(n)}-\kappa_a^\star
=\frac{\mu_{aa}^{(n)}-\max_{b\ne a}\mu_{ab}^{(n)}}{2}
\ge \frac{\gamma_n}{2}.
\]
For every $b\ne a$,
\[
\kappa_a^\star-\mu_{ab}^{(n)}
\ge \kappa_a^\star-\max_{b'\ne a}\mu_{ab'}^{(n)}
=\frac{\mu_{aa}^{(n)}-\max_{b'\ne a}\mu_{ab'}^{(n)}}{2}
\ge \frac{\gamma_n}{2}.
\]
Thus, one threshold $\kappa_a^\star$ separates the same-class population mean
from all different-class population means for seed class $a$, with margin at
least $\gamma_n/2$ on both sides.

\paragraph{Step 3: Concentration Around the Midpoint.}
Fix distinct $u,v$ with $y_u=a$ and $y_v=b$. Then $s_2(u,v)$ is a sum of
independent Bernoulli variables with mean $\mu_{ab}^{(n)}\le \mu_{\max}$.

\medskip\noindent
\emph{Same-class case.}
Assume $b=a$. Let
\[
 t_{\mathrm{same}}:=\mu_{aa}^{(n)}-\kappa_a^\star\ge \gamma_n/2.
\]
Since $\kappa_a^\star\ge 0$, we have $t_{\mathrm{same}}\le \mu_{aa}^{(n)}$, so
Lemma~\ref{lem:tails} applies:
\begin{align*}
\P{s_2(u,v)\le \kappa_a^\star\mid y}
&\le \exp\left(-\frac{t_{\mathrm{same}}^2}{2\mu_{aa}^{(n)}}\right) \le \exp\left(-\frac{\gamma_n^2}{8\mu_{\max}}\right).
\end{align*}
By \eqref{eq:finite-sample-conditions}, the exponent is at most $-6\log n$
and therefore
\[
\P{s_2(u,v)\le \kappa_a^\star\mid y}\le n^{-6}\le n^{-4}.
\]

\medskip\noindent
\emph{Different-class case.}
Assume $b\ne a$. Let
\[
 t_{\mathrm{diff}}:=\kappa_a^\star-\mu_{ab}^{(n)}\ge \gamma_n/2.
\]
By Lemma~\ref{lem:tails},
\[
\P{s_2(u,v)\ge \kappa_a^\star\mid y}
\le \exp\left(-\frac{t_{\mathrm{diff}}^2}{2\mu_{ab}^{(n)}+2t_{\mathrm{diff}}/3}\right).
\]
For fixed $\mu\ge 0$, the function
\[
\phi_\mu(t):=\frac{t^2}{2\mu+2t/3},\qquad t>0,
\]
is increasing, since
\[
\phi_\mu'(t)=\frac{4\mu t+(2/3)t^2}{(2\mu+2t/3)^2}>0.
\]
Using $t_{\mathrm{diff}}\ge \gamma_n/2$ and $\mu_{ab}^{(n)}\le\mu_{\max}$,
\begin{align*}
\frac{t_{\mathrm{diff}}^2}{2\mu_{ab}^{(n)}+2t_{\mathrm{diff}}/3}
&\ge \frac{(\gamma_n/2)^2}{2\mu_{ab}^{(n)}+\gamma_n/3}=\frac{\gamma_n^2}{8\mu_{ab}^{(n)}+(4/3)\gamma_n}
\ge \frac{\gamma_n^2}{8\mu_{\max}+(4/3)\gamma_n}.
\end{align*}
We show that this last quantity is at least $4\log n$. It is enough to prove
\[
\gamma_n^2\ge 4\log n\bigl(8\mu_{\max}+(4/3)\gamma_n\bigr)
=32\mu_{\max}\log n+(16/3)\gamma_n\log n.
\]
Since $\gamma_n\ge16\log n$,
\[
(16/3)\gamma_n\log n\le \gamma_n^2/3.
\]
Also $\gamma_n^2\ge48\mu_{\max}\log n$ implies
\[
32\mu_{\max}\log n\le (2/3)\gamma_n^2.
\]
Therefore,
\[
32\mu_{\max}\log n+(16/3)\gamma_n\log n
\le (2/3)\gamma_n^2+(1/3)\gamma_n^2
=\gamma_n^2.
\]
Hence,
\[
\P{s_2(u,v)\ge \kappa_a^\star\mid y}\le n^{-4}.
\]

\paragraph{Step 4: Union Bound Over Ordered Pairs.}
Define the favorable event
\[
\mathcal F:=
\bigcap_{\substack{u,v\in V:\,u\ne v,\ y_u=y_v}}
\{s_2(u,v)>\kappa_{y_u}^\star\}
\cap
\bigcap_{\substack{u,v\in V:\,u\ne v,\ y_u\ne y_v}}
\{s_2(u,v)<\kappa_{y_u}^\star\}.
\]
There are at most $n(n-1)\le n^2$ ordered pairs $(u,v)$, and each corresponding
failure event has probability at most $n^{-4}$. Thus,
\[
\P{\mathcal F^c\mid y}\le n^2 n^{-4}=n^{-2}\le 2n^{-2}.
\]
Therefore, $\P{\mathcal F\mid y}\ge1-2n^{-2}$.

\paragraph{Step 5: Uniform Separation and Top-$k$ Homophily.}
On $\mathcal F$, fix any seed $u$ with $y_u=a$. For every same-class candidate
$v$ and every different-class candidate $w$,
\[
 s_2(u,v)>\kappa_a^\star>s_2(u,w).
\]
Hence, every same-class candidate has strictly larger two-hop score than every
different-class candidate relative to seed $u$.

If $k\le \min_a(n_a-1)$, every seed $u$ has at least $k$ other nodes in its own
class. Hence, the directed top-$k$ rule from $u$ selects only same-label nodes.
Symmetrization cannot create a cross-label edge, because an undirected edge
$\{u,v\}$ appears only if at least one of the directed edges $u\to v$ or
$v\to u$ was selected, and every selected directed edge is same-label.
Consequently, $h_{G_2^\star(k)}=1$.

\subsubsection{Edge Homophily of the Original SBM}\label{proof_sbm}

\begin{corollary}[Limiting Edge Homophily of the Original SBM]
\label{cor:original-sbm-homophily}
Under the SBM setting above, assume $\sum_{a,b=1}^C \pi_a \pi_b B_{ab} > 0$.
Then the edge homophily of the original graph satisfies the following: $$ h_G \xrightarrow{p} \frac{\sum_{a=1}^C \pi_a^2 B_{aa}} {\sum_{a,b=1}^C \pi_a\pi_b B_{ab}},$$
where $\xrightarrow{p}$ denotes convergence in probability. In particular, the
limit is strictly below $1/2$ whenever
$\sum_{a=1}^C \pi_a^2 B_{aa} < \sum_{a\neq b} \pi_a \pi_b B_{ab}$.
\end{corollary}

\begin{proof}
Let $M_{\rm same}$ denote the number of same-label edges and $M$ the total
number of edges, so that $h_G = M_{\rm same}/M$. By independence of edges in
the SBM,
\begin{equation}
\label{eq:Esame}
    \mathbb{E}[M_{\rm same}]
    =
    \sum_{a=1}^C \binom{n_a}{2} B_{aa}
    =
    \frac{n^2}{2}\sum_{a=1}^C \pi_a^2 B_{aa}+O(n),
\end{equation}
and
\begin{equation}
\label{eq:Etot}
    \mathbb{E}[M]
    =
    \sum_{a=1}^C \binom{n_a}{2} B_{aa}
    +
    \sum_{a<b} n_a n_b B_{ab}
    =
    \frac{n^2}{2}\sum_{a,b=1}^C \pi_a\pi_b B_{ab}+O(n).
\end{equation}
Both $M_{\rm same}$ and $M$ are sums of independent Bernoulli random variables,
so
\[
\mathrm{Var}(M)\le \mathbb{E}[M],
\qquad
\mathrm{Var}(M_{\rm same})\le \mathbb{E}[M_{\rm same}].
\]
Since $\sum_{a,b}\pi_a\pi_b B_{ab}>0$, \eqref{eq:Etot} gives
$\mathbb{E}[M]=\Theta(n^2)$, so Chebyshev's inequality implies
$M/\mathbb{E}[M]\xrightarrow{p} 1$. If $\sum_a \pi_a^2 B_{aa}>0$, the same
argument applied to \eqref{eq:Esame} gives
$M_{\rm same}/\mathbb{E}[M_{\rm same}]\xrightarrow{p} 1$, and the
continuous-mapping theorem yields 
\begin{equation}
\label{eq:original-sbm-homophily}
    h_G
    \xrightarrow{p}
    \frac{\sum_{a=1}^C \pi_a^2 B_{aa}}
         {\sum_{a,b=1}^C \pi_a\pi_b B_{ab}}.
\end{equation}If
$\sum_a \pi_a^2 B_{aa}=0$, then $B_{aa}=0$ for every $a$ with $\pi_a>0$, so
$M_{\rm same}=0$ deterministically and \eqref{eq:original-sbm-homophily}
still holds with limit zero.
\end{proof}

\subsection{Learnability After Cascade Rewiring}
\label{learnability}
\begin{corollary}[Learnability After Cascade Rewiring]\label{cor:learnability}
Assume the setting and profile-margin condition of
Theorem~\ref{thm:sbm-profile-homophily}, and let \(G_2^\star(k)\) be the symmetrized top-\(k\) two-hop cascade-rewired graph, with $k\le \min_{a\in[C]}(n_a-1)$. Suppose, independently of the graph conditional on the labels, node features are generated by $X_i = \mu_{y_i}+\sigma \xi_i, \; \xi_i\sim \mathcal N(0,I_d)$, where \(\mu_1,\dots,\mu_C\in\mathbb{R}^d\), \(\sigma>0\), and $\Delta := \min_{a\neq b}\|\mu_a-\mu_b\|_2 >0$. For each node \(i\), let \(N_\star(i)\) denote its neighborhood in \(G_2^\star(k)\), and define $Z_i := \frac{1}{q_i} \sum_{j\in \{i\}\cup N_\star(i)} X_j, \; q_i := |\{i\}\cup N_\star(i)|$. Consider the nearest-centroid affine linear classifier $\widehat y_i := \arg\max_{a\in[C]} \left\{ \mu_a^\top Z_i-\frac12\|\mu_a\|_2^2 \right\}$.
Then, with probability at least $1 - 2n^{-2} - n(C-1)\exp\left( -\frac{(k+1)\Delta^2}{8\sigma^2} \right)$, all nodes are classified correctly. In particular, if $\frac{(k+1)\Delta^2}{8\sigma^2} \ge \log n+\log(C-1)+r_n$ for some \(r_n\to\infty\), then $\Pr{\widehat y_i=y_i\ \text{for all }i\in V}\to 1$. Moreover, the expected misclassification fraction satisfies $\mathbb{E}\left[
\frac1n\sum_{i=1}^n \mathbf 1\{\widehat y_i\neq y_i\} \right] \le
2n^{-2} + (C-1)\exp\left( -\frac{(k+1)\Delta^2}{8\sigma^2} \right).
$
\end{corollary}
\begin{proof}
We prove the result conditional on the label vector \(y=(y_1,\dots,y_n)\). All probabilities below are over the randomness of the graph and the Gaussian features. The same bounds hold unconditionally if the labels are random and independent of the feature noise. Let 
\[  \mathcal E_{\mathrm{hom}} := \left\{ G_2^\star(k)\ \text{is perfectly homophilic} \right\}. \]
By Theorem~\ref{thm:sbm-profile-homophily}, under the stated profile-margin condition, $\Pr{\mathcal E_{\mathrm{hom}}}\ge 1-2n^{-2}$.
Equivalently, $\Pr{\mathcal E_{\mathrm{hom}}^c}\le 2n^{-2}$. We first analyze the classifier on the event \(\mathcal E_{\mathrm{hom}}\). Fix a realization of \(G_2^\star(k)\) for which \(\mathcal E_{\mathrm{hom}}\) holds. Since
\(G_2^\star(k)\) is perfectly homophilic, every edge of \(G_2^\star(k)\) connects two
vertices with the same label. Therefore, for every node \(i\), every vertex in
\(\{i\}\cup N_\star(i)\) has label \(y_i\).

Moreover, because \(G_2^\star(k)\) is the symmetrization of a directed top-\(k\) graph,
node \(i\) is adjacent in the symmetrized graph to at least the \(k\) vertices selected by
the directed top-\(k\) rule from \(i\). Hence,
\[
q_i
=
|\{i\}\cup N_\star(i)|
\ge k+1 .
\]

Fix a node \(i\in V\) and let \(a:=y_i\). On \(\mathcal E_{\mathrm{hom}}\), every
\(j\in \{i\}\cup N_\star(i)\) has label \(a\). Since the features are independent
conditional on the labels, and are independent of the graph conditional on the labels, we
have, conditional on \(G_2^\star(k)\), \(\mathcal E_{\mathrm{hom}}\), and the labels,
\[
X_j \sim \mathcal N(\mu_a,\sigma^2 I_d)
\qquad
\text{independently for all } j\in \{i\}\cup N_\star(i).
\]
Therefore,
\[
Z_i = \frac{1}{q_i} \sum_{j\in \{i\}\cup N_\star(i)} X_j
\sim \mathcal N\left(\mu_a,\frac{\sigma^2}{q_i}I_d\right).
\]
Equivalently, we may write
\[
Z_i=\mu_a+\eta_i, \qquad \eta_i\sim \mathcal N\left(0,\frac{\sigma^2}{q_i}I_d\right).
\]

The classifier is the nearest-centroid classifier because
\[
\arg\max_{c\in[C]} \left\{ \mu_c^\top Z_i-\frac12\|\mu_c\|_2^2 \right\} = \arg\min_{c\in[C]}\|Z_i-\mu_c\|_2^2 .
\]
As
\[
\|Z_i-\mu_c\|_2^2 = \|Z_i\|_2^2 -2\mu_c^\top Z_i +\|\mu_c\|_2^2,
\]
the term \(\|Z_i\|_2^2\) does not depend on \(c\).

Now fix an incorrect class \(b\neq a\). The classifier can prefer \(b\) over the true class
\(a\) only if
\[ \|Z_i-\mu_b\|_2^2 \le \|Z_i-\mu_a\|_2^2 .
\]
Substituting \(Z_i=\mu_a+\eta_i\), and defining
\[
\delta_{ab}:=\mu_a-\mu_b,
\]
this event becomes
\[
\|\eta_i+\delta_{ab}\|_2^2 \le \|\eta_i\|_2^2 .
\]
Expanding both sides gives
\[
\|\eta_i\|_2^2 + 2\delta_{ab}^\top \eta_i + \|\delta_{ab}\|_2^2 \le
\|\eta_i\|_2^2,
\]
or equivalently
\[
\delta_{ab}^\top \eta_i \le -\frac12\|\delta_{ab}\|_2^2 .
\]
Since
\[
\eta_i\sim \mathcal N\left(0,\frac{\sigma^2}{q_i}I_d\right),
\]
the one-dimensional projection satisfies
\[
\delta_{ab}^\top \eta_i \sim \mathcal N\left( 0, \frac{\sigma^2}{q_i}\|\delta_{ab}\|_2^2
\right).
\]
Using the standard Gaussian tail bound
\[
\Pr{W\le -t}\le \exp\left(-\frac{t^2}{2s^2}\right), \qquad W\sim \mathcal N(0,s^2),
\]
with
\[
t=\frac12\|\delta_{ab}\|_2^2, \qquad
s^2=\frac{\sigma^2}{q_i}\|\delta_{ab}\|_2^2,
\]
we obtain
\[ \Pr{ \|Z_i-\mu_b\|_2^2 \le \|Z_i-\mu_a\|_2^2 \ \middle|\  G_2^\star(k),\mathcal E_{\mathrm{hom}},y }
\le \exp\left( -\frac{q_i\|\delta_{ab}\|_2^2}{8\sigma^2} \right).
\]
By the definitions of \(q_i\) and \(\Delta\),
\[
q_i\ge k+1, \qquad \|\delta_{ab}\|_2=\|\mu_a-\mu_b\|_2\ge \Delta.
\]
Hence,
\[
\Pr{ \|Z_i-\mu_b\|_2^2 \le \|Z_i-\mu_a\|_2^2 \ \middle|\  G_2^\star(k),\mathcal E_{\mathrm{hom}},y
} \le \exp\left( -\frac{(k+1)\Delta^2}{8\sigma^2} \right).
\]

Taking a union bound over all incorrect classes \(b\neq a\), we get
\[
\Pr{ \widehat y_i\neq y_i \ \middle|\  G_2^\star(k),\mathcal E_{\mathrm{hom}},y }
\le (C-1) \exp\left( -\frac{(k+1)\Delta^2}{8\sigma^2} \right).
\]
This bound is uniform in the realization of \(G_2^\star(k)\) on
\(\mathcal E_{\mathrm{hom}}\). Therefore,
\[
\Pr{ \widehat y_i\neq y_i \ \middle|\  \mathcal E_{\mathrm{hom}},y } \le
(C-1) \exp\left( -\frac{(k+1)\Delta^2}{8\sigma^2} \right).
\]

Next, taking a union bound over all \(n\) vertices gives
\[
\Pr{ \exists i\in V:\widehat y_i\neq y_i \ \middle|\  \mathcal E_{\mathrm{hom}},y }
\le n(C-1) \exp\left( -\frac{(k+1)\Delta^2}{8\sigma^2} \right).
\]
Therefore,
\[
\begin{aligned}
\Pr{ \exists i\in V:\widehat y_i\neq y_i } &\le \Pr{\mathcal E_{\mathrm{hom}}^c} +
\Pr{ \exists i\in V:\widehat y_i\neq y_i \ \middle|\  \mathcal E_{\mathrm{hom}},y } \\
&\le 2n^{-2} + n(C-1) \exp\left( -\frac{(k+1)\Delta^2}{8\sigma^2} \right).
\end{aligned}
\]
Equivalently, with probability at least
\[
1 - 2n^{-2} - n(C-1) \exp\left( -\frac{(k+1)\Delta^2}{8\sigma^2} \right),
\]
we have
\[ \widehat y_i=y_i \qquad \text{for all } i\in V.
\]

Now suppose
\[ \frac{(k+1)\Delta^2}{8\sigma^2} \ge \log n+\log(C-1)+r_n,
\]
where \(r_n\to\infty\). Then
\[ n(C-1) \exp\left( -\frac{(k+1)\Delta^2}{8\sigma^2} \right) \le n(C-1)\exp\left(-\log n-\log(C-1)-r_n\right)
= e^{-r_n} \to 0.
\]
Since also \(2n^{-2}\to 0\), we conclude that
\[
\Pr{ \widehat y_i=y_i\ \text{for all }i\in V } \to 1.
\]

It remains to prove the expected misclassification fraction bound. Define
\[ L_n := \frac1n\sum_{i=1}^n \mathbf 1\{\widehat y_i\neq y_i\}.
\]
Since \(0\le L_n\le 1\),
\[ \mathbb E[L_n] = \mathbb E[L_n\mathbf 1_{\mathcal E_{\mathrm{hom}}}]
+ \mathbb E[L_n\mathbf 1_{\mathcal E_{\mathrm{hom}}^c}]
\le \mathbb E[L_n\mathbf 1_{\mathcal E_{\mathrm{hom}}}]
+ \Pr(\mathcal E_{\mathrm{hom}}^c).
\]
Using the single-node error bound on \(\mathcal E_{\mathrm{hom}}\),
\[
\begin{aligned}
\mathbb E[L_n\mathbf 1_{\mathcal E_{\mathrm{hom}}}]
&=
\frac1n\sum_{i=1}^n
\Pr{ \widehat y_i\neq y_i,\mathcal E_{\mathrm{hom}} }\\ &\le
\frac1n\sum_{i=1}^n (C-1) \exp\left( -\frac{(k+1)\Delta^2}{8\sigma^2} \right) \\
&=
(C-1) \exp\left( -\frac{(k+1)\Delta^2}{8\sigma^2} \right).
\end{aligned}
\]
Combining this with \(\Pr{\mathcal E_{\mathrm{hom}}^c}\le 2n^{-2}\), we obtain
\[
\mathbb E\left[ \frac1n\sum_{i=1}^n \mathbf 1\{\widehat y_i\neq y_i\} \right]
\le 2n^{-2} +(C-1)\exp\left( -\frac{(k+1)\Delta^2}{8\sigma^2} \right).
\]

\end{proof}

\subsection{Extending Profile Homophily on SBM : $\ell \geq 2$} \label{app:extension}
\begin{lemma}[Hoeffding Bound for Dependency Graphs]
\label{lem:dep-graph-hoeffding}
Let $X_1,\ldots,X_N\in[0,1]$ be random variables with a dependency graph of
maximum degree $\Delta$. That is, whenever two disjoint index sets have no edge
between them in the dependency graph, the corresponding collections of random
variables are independent. Let $S:=\sum_{i=1}^N X_i$. Then, for every $t>0$,
$$\Pr{ |S-\mathbb E S|\ge t } \le 2\exp\left( -\frac{2t^2}{(\Delta+1)N} \right).$$
\end{lemma}

\begin{proof}
Let $\chi$ be the chromatic number of the dependency graph. Since the maximum
degree is $\Delta$, we have $\chi \le \Delta+1$. Let
$I_1,\ldots,I_\chi$ be a proper coloring of the dependency graph and define
\[
    S_j:=\sum_{i\in I_j}X_i .
\]
Within each color class $I_j$, the variables are mutually independent. For
$\lambda>0$, H\"older's inequality gives
\[
    \mathbb E\exp\left(\lambda(S-\mathbb E S)\right)
    =
    \mathbb E
    \prod_{j=1}^{\chi}
    \exp\left(\lambda(S_j-\mathbb E S_j)\right)
    \le
    \prod_{j=1}^{\chi}
    \left[
        \mathbb E
        \exp\left(\chi\lambda(S_j-\mathbb E S_j)\right)
    \right]^{1/\chi}.
\]
Since the variables in each $S_j$ are independent and lie in $[0,1]$,
Hoeffding's lemma gives
\[
    \mathbb E
    \exp\left(\chi\lambda(S_j-\mathbb E S_j)\right)
    \le
    \exp\left(
        \frac{\chi^2\lambda^2 |I_j|}{8}
    \right).
\]
Therefore,
\[
    \mathbb E\exp\left(\lambda(S-\mathbb E S)\right)
    \le
    \exp\left(
        \frac{\chi\lambda^2}{8}
        \sum_{j=1}^{\chi}|I_j|
    \right)
    =
    \exp\left(
        \frac{\chi\lambda^2 N}{8}
    \right).
\]
Chernoff's bound yields
\[
    \Pr{S-\mathbb E S\ge t}
    \le
    \exp\left(
        -\frac{2t^2}{\chi N}
    \right)
    \le
    \exp\left(
        -\frac{2t^2}{(\Delta+1)N}
    \right).
\]
The lower tail is identical. Combining the two tails proves the claim.
\end{proof}

\subsection{Profile Homophily of $\ell$-Hop Top-$k$ Selection}\label{app:sbm-raw-l-hop}

\begin{theorem}[Profile Homophily of $\ell$-Hop Top-$k$ Selection]
\label{thm:sbm-raw-l-hop}
Let $G\sim\operatorname{SBM}(n,\pi,B)$ be a $C$-block stochastic block model
with fixed block sizes $n_a=\pi_a n$, where $\pi_a>0$ and
$\sum_{a=1}^C\pi_a=1$. Assume $B\in[0,1]^{C\times C}$ is symmetric. Conditional
on the labels, edges are independent and
$\Pr{A_{uv}=1\mid y_u=a,y_v=b}=B_{ab}$.
Let $D_\pi:=\operatorname{diag}(\pi_1,\ldots,\pi_C)$.
Fix an integer $\ell\ge 2$ independent of $n$. For two distinct nodes $u,v$,
define the raw $\ell$-hop reinforcement score $s_\ell(u,v):=(A^\ell)_{uv}$,
i.e., the number of length-$\ell$ walks from $u$ to $v$. Define the
$\ell$-hop profile matrix $K_\ell:=B(D_\pi B)^{\ell-1}$, and assume that
$K_\ell$ separates same-label pairs from different-label pairs:
\[
    \Delta_\ell
    :=
    \min_{a\in[C]}
    \left[
        (K_\ell)_{aa}
        -
        \max_{b\ne a}(K_\ell)_{ab}
    \right]
    >0.
\]
Then, with probability $1-o(1)$, for every seed node $u\in V$ and every pair
of distinct nodes $v,w\in V\setminus\{u\}$ satisfying $y_v=y_u$ and
$y_w\ne y_u$, we have $s_\ell(u,v)>s_\ell(u,w)$. Consequently, for every
$k\le \min_{a\in[C]}(n_a-1)$, the directed top-$k$ graph obtained by
connecting each node $u$ to the $k$ nodes $v\ne u$ with the largest raw score
$s_\ell(u,v)$ contains only same-label directed edges. Its symmetrization is
therefore perfectly homophilic: $h_{G_\ell^\star(k)}=1$ with probability
$1-o(1)$.
\end{theorem}

\begin{proof}
We first establish a finite-sample concentration bound. The asymptotic statement follows
because $\Delta_\ell>0$ implies the finite-sample margin condition for all
sufficiently large $n$.

Throughout the proof, probabilities and expectations are conditional on the
fixed labels; we suppress this conditioning from the notation.

\paragraph{Step 1: $\ell$-Hop Profile Mean.}
Fix two distinct nodes $u,v$ with labels $y_u=a$ and $y_v=b$. The raw score
\[
    s_\ell(u,v)=(A^\ell)_{uv}
\]
counts length-$\ell$ walks
\[
    u=x_0,x_1,\ldots,x_{\ell-1},x_\ell=v.
\]

First consider walks whose intermediate nodes
$x_1,\ldots,x_{\ell-1}$ are all distinct and different from $u$ or $v$. For a
label sequence
\[
    c_1,\ldots,c_{\ell-1}\in[C],
\]
the number of such intermediate tuples with $y_{x_i}=c_i$ is
\[
    n^{\ell-1}\pi_{c_1}\cdots\pi_{c_{\ell-1}}
    +
    O_\ell(n^{\ell-2}).
\]
For each such tuple, conditional independence of SBM edges gives
\[
    \mathbb E
    \left[
        A_{ux_1}A_{x_1x_2}\cdots A_{x_{\ell-1}v}
    \right]
    =
    B_{ac_1}B_{c_1c_2}\cdots B_{c_{\ell-1}b}.
\]
Therefore the expected number of simple length-$\ell$ paths from $u$ to $v$ is
\[
    n^{\ell-1}
    \sum_{c_1,\ldots,c_{\ell-1}}
    \pi_{c_1}\cdots\pi_{c_{\ell-1}}
    B_{ac_1}B_{c_1c_2}\cdots B_{c_{\ell-1}b}
    +
    O_\ell(n^{\ell-2}).
\]
The summation is exactly the $(a,b)$ entry of
\[
    B(D_\pi B)^{\ell-1}=K_\ell .
\]
The remaining length-$\ell$ walks, namely those with repeated vertices or with
an intermediate vertex equal to $u$ or $v$, are at most $O_\ell(n^{\ell-2})$ in
number. Since each contributes at most one, their total contribution to the
expectation is also $O_\ell(n^{\ell-2})$. Hence,
\[
    \mu_{ab}^{(\ell,n)}
    =
    n^{\ell-1}(K_\ell)_{ab}
    +
    O_\ell(n^{\ell-2}).
\]

\paragraph{Step 2: Uniform Concentration of Raw $\ell$-Hop Scores.}
Let $\widetilde s_\ell(u,v)$ denote the number of simple length-$\ell$ paths
from $u$ to $v$. We can write
\[
    \widetilde s_\ell(u,v)
    =
    \sum_{\mathbf x} Z_{\mathbf x},
\]
where $\mathbf x=(x_1,\ldots,x_{\ell-1})$ ranges over valid simple intermediate
tuples and
\[
    Z_{\mathbf x}
    :=
    A_{ux_1}A_{x_1x_2}\cdots A_{x_{\ell-1}v}.
\]
Each $Z_{\mathbf x}$ is a Bernoulli random variable.

Construct a dependency graph on the variables $\{Z_{\mathbf x}\}$ by connecting
two variables whenever their corresponding paths share at least one edge. If
two collections of variables have no edge between them in this dependency
graph, then their underlying sets of SBM edge variables are disjoint. Hence, the
collections are independent, so this is a valid dependency graph.

The number of variables is
\[
    N=O_\ell(n^{\ell-1}).
\]
The maximum degree of the dependency graph satisfies
\[
    \Delta=O_\ell(n^{\ell-2}).
\]
Indeed, fix one simple length-$\ell$ path from $u$ to $v$. Another simple
length-$\ell$ path can share one of its $\ell$ edges. If the shared edge is the
first edge incident to $u$, or the last edge incident to $v$, then one
intermediate coordinate is fixed and the remaining $\ell-2$ intermediate
coordinates are free, giving $O_\ell(n^{\ell-2})$ choices. If the shared edge
is internal, then two intermediate coordinates are fixed and only $\ell-3$
remain free, giving $O_\ell(n^{\ell-3})$ choices. The endpoint-edge cases
dominate, so
\[
    \Delta=O_\ell(n^{\ell-2}).
\]
For $\ell=2$, distinct simple two-hop paths between fixed endpoints share no
edge, so the dependency degree is zero; this is also covered by the bound above.

By Lemma~\ref{lem:dep-graph-hoeffding}, for a sufficiently large constant
$C_\ell>0$ depending only on $\ell$ and $C$, for every fixed ordered pair
$(u,v)$,
\[
    \Pr{
        \left|
            \widetilde s_\ell(u,v)
            -
            \mathbb E\widetilde s_\ell(u,v)
        \right|
        >
        C_\ell n^{\ell-\frac32}\sqrt{\log n}
    }
    \le
    2n^{-4}.
\]
Indeed,
\[
    N(\Delta+1)=O_\ell(n^{2\ell-3}),
\]
so the deviation scale from Lemma~\ref{lem:dep-graph-hoeffding} is
\[
    \sqrt{N(\Delta+1)\log n}
    =
    O_\ell\left(n^{\ell-\frac32}\sqrt{\log n}\right).
\]

Taking a union bound over at most $n^2$ ordered pairs, with probability at
least $1-2n^{-2}$,
\[
    \left|
        \widetilde s_\ell(u,v)
        -
        \mathbb E\widetilde s_\ell(u,v)
    \right|
    \le
    C_\ell n^{\ell-\frac32}\sqrt{\log n}
\]
simultaneously for all distinct $u,v$.

It remains to pass from simple paths to all walks. Let
\[
    s_\ell(u,v)=\widetilde s_\ell(u,v)+R_\ell(u,v),
\]
where $R_\ell(u,v)$ counts non-simple length-$\ell$ walks from $u$ to $v$.
As in Step 1,
\[
    0\le R_\ell(u,v)\le C_\ell' n^{\ell-2}
\]
for a constant $C_\ell'>0$ depending only on $\ell$ and $C$. Therefore,
\[
    0\le \mathbb E R_\ell(u,v)\le C_\ell' n^{\ell-2}.
\]
Hence,
\[
    \left|
        s_\ell(u,v)-\mu_{y_uy_v}^{(\ell,n)}
    \right|
    \le
    \left|
        \widetilde s_\ell(u,v)
        -
        \mathbb E\widetilde s_\ell(u,v)
    \right|
    +
    2C_\ell'n^{\ell-2}.
\]
After increasing $C_\ell$ if necessary, we obtain the uniform bound
\[
    \left|
        s_\ell(u,v)-\mu_{y_uy_v}^{(\ell,n)}
    \right|
    \le
    C_\ell n^{\ell-\frac32}\sqrt{\log n}
    +
    C_\ell n^{\ell-2}
    =:\varepsilon_{n,\ell}
\]
simultaneously for all distinct $u,v$, with probability at least $1-2n^{-2}$.

\paragraph{Step 3: Uniform Separation of Same-Label and Different-Label Scores.}
Define the finite-sample class-separation margin
\[
\gamma_{n,\ell} := \min_{a \in [C]}\left[\mu_{aa}^{(\ell,n)} - \max_{b \neq a}\mu_{ab}^{(\ell,n)}\right].
\]
By Step 1, $\gamma_{n,\ell} = n^{\ell-1}\Delta_\ell + O_\ell(n^{\ell-2})$. Assume the concentration event from Step 2 holds. Fix a seed node $u$ with
label $y_u=a$. Let $v$ satisfy $y_v=a$, and let $w$ satisfy $y_w=b\ne a$.
Then
\[
    s_\ell(u,v)
    \ge
    \mu_{aa}^{(\ell,n)}
    -
    \varepsilon_{n,\ell},
\]
while
\[
    s_\ell(u,w)
    \le
    \mu_{ab}^{(\ell,n)}
    +
    \varepsilon_{n,\ell}
    \le
    \mu_{aa}^{(\ell,n)}
    -
    \gamma_{n,\ell}
    +
    \varepsilon_{n,\ell}.
\]
Since $\Delta_\ell > 0$, we have $\gamma_{n,\ell} = n^{\ell-1}\Delta_\ell + O_\ell(n^{\ell-2})$, while $\varepsilon_{n,\ell} = O_\ell(n^{\ell-3/2}\sqrt{\log n})$. Hence, $\gamma_{n,\ell} > 2\varepsilon_{n,\ell}$ for all sufficiently large $n$ and
\[
    s_\ell(u,v)>s_\ell(u,w).
\]
In particular, $s_\ell(u,v) > s_\ell(u,w)$ holds whenever the finite-sample margin satisfies $\gamma_{n,\ell} > 2\varepsilon_{n,\ell}$, with $\varepsilon_{n,\ell} = C_\ell n^{\ell - 3/2}\sqrt{\log n} + C_\ell n^{\ell - 2}$. Thus, uniformly over every seed node, every same-label candidate has strictly
larger raw $\ell$-hop score than every different-label candidate.

\paragraph{Step 4: Homophily of the Top-$k$ Graph.}
Assume
\[
    k\le \min_{a\in[C]}(n_a-1).
\]
Then every seed node $u$ has at least $k$ other nodes with the same label. By
Step 3, all same-label candidates have larger score $s_\ell(u,\cdot)$ than all
different-label candidates. Therefore, every directed top-$k$ edge $u\to v$
satisfies
\[
    y_u=y_v.
\]
Symmetrization cannot create a cross-label edge, since an undirected edge
$\{u,v\}$ is included only if at least one of the directed edges $u\to v$ or
$v\to u$ was selected. Hence, the symmetrized graph is perfectly homophilic:
\[
    h_{G_\ell^\star(k)}=1.
\]

Finally, suppose $\Delta_\ell>0$. By Step 1,
\[
    \mu_{ab}^{(\ell,n)}
    =
    n^{\ell-1}(K_\ell)_{ab}
    +
    O_\ell(n^{\ell-2}).
\]
Therefore,
\[
    \gamma_{n,\ell}
    =
    n^{\ell-1}\Delta_\ell
    +
    O_\ell(n^{\ell-2}).
\]
Since
\[
    n^{\ell-1}
    \gg
    n^{\ell-\frac32}\sqrt{\log n}
    +
    n^{\ell-2},
\]
$\gamma_{n,\ell} > 2\varepsilon_{n,\ell}$ holds for all sufficiently large $n$. Hence, the perfect-homophily conclusion holds with probability
$1-o(1)$.
\end{proof}

\subsection{Proof of Lemma \ref{lem:paths} (Effective Resistance Characterization)} \label{proof3}

\textbf{Activation Times and the Activation DAG.} For every \(x\in A^v\), define its activation time
\[
T(x):=\min\{t\in\{0,1,\dots,\ell\}:x\in S_t^v\}.
\]
Thus, \(T(x)=0\) for \(x\in S_0^v\), and \(T(x)\ge 1\) for every non-initially active
vertex. Construct a directed graph \(D^v=(A^v,\vec E^{\,v})\) as follows. For every edge \(\{a,b\}\in E(G_v)\) with \(T(a)<T(b)\), include the directed arc \(a\to b\). If \(T(a)=T(b)\), include no arc corresponding to \(\{a,b\}\). Thus, every arc strictly increases activation time, so \(D^v\) is acyclic.
 
\textbf{Cut Lower Bound on Subsets Disjoint From $S_0^v$.} We first prove a cut lower bound. Let \(U\subseteq A^v\) be nonempty and suppose \(U\cap S_0^v=\emptyset\). Choose \(x\in U\) minimizing \(T(x)\) over all vertices in \(U\). Since \(U\cap S_0^v=\emptyset\), we have \(T(x)\ge 1\). By minimality of \(T(x)\), no vertex of \(U\) is active before time \(T(x)\), so
\[
U\cap S_{T(x)-1}^v=\emptyset .
\]
Because \(x\) first activates at time \(T(x)\), the uniform threshold condition gives
\[
\bigl|N(x)\cap S_{T(x)-1}^v\bigr|\ge \kappa .
\]
Every vertex \(y\in N(x)\cap S_{T(x)-1}^v\) lies outside \(U\), and satisfies
\(T(y)<T(x)\). Hence, the edge \(\{y,x\}\) appears in \(D^v\) as the incoming arc
\(y\to x\). Therefore, at least \(\kappa\) arcs enter \(U\):
\[
\bigl|\delta^-_{D^v}(U)\bigr|
\ge
\bigl|N(x)\cap S_{T(x)-1}^v\bigr|
\ge \kappa ,
\]
where
\[
\delta^-_{D^v}(U)
:=
\{(a,b)\in \vec E^{\,v}: a\in A^v\setminus U,\ b\in U\}.
\]

\textbf{Augmented Network and Capacity Assignment.} We now add a new source vertex \(s\). For every \(z\in S_0^v\), add an arc \(s\to z\)
of capacity \(\kappa\), and give every arc of \(D^v\) capacity \(1\). We claim that
every \(s\)-\(u\) cut in this network has capacity at least \(\kappa\).

\textbf{Lower Bound on Every $s$-$u$ Cut.} Let \(X\) be the source side of an arbitrary \(s\)-\(u\) cut, so \(s\in X\) and
\(u\notin X\). Define
\[
U := A^v\setminus X .
\]
Then \(u\in U\), so \(U\neq\emptyset\). If \(U\cap S_0^v\neq\emptyset\), then for any
\(z\in U\cap S_0^v\), the source arc \(s\to z\) crosses the cut and has capacity
\(\kappa\). Hence, the cut has capacity at least \(\kappa\).

It remains to consider the case \(U\cap S_0^v=\emptyset\). By the cut lower bound proved
above,
\[
\bigl|\delta^-_{D^v}(U)\bigr|\ge \kappa .
\]
Each arc in \(\delta^-_{D^v}(U)\) crosses from \(X\cap A^v\) into \(U=A^v\setminus X\),
and each such arc has capacity \(1\). Thus, the cut again has capacity at least \(\kappa\).

\textbf{Max-Flow/Min-Cut and Integral Decomposition.} Therefore, every \(s\)-\(u\) cut has capacity at least \(\kappa\). By the max-flow/min-cut
theorem, the maximum \(s\)-\(u\) flow has value at least \(\kappa\). Since all capacities
are integral, there is an integral flow of value at least \(\kappa\). Decomposing this
flow into directed \(s\)-\(u\) paths and selecting any \(\kappa\) unit-flow paths, we
obtain \(\kappa\) directed paths from \(s\) to \(u\). The source arcs \(s\to z\) may be
used more than once, but every arc of \(D^v\) has unit capacity, so the portions of these
paths inside \(D^v\) are pairwise arc-disjoint.

Removing the initial source arcs gives \(\kappa\) pairwise arc-disjoint directed paths in
\(D^v\), each starting at some vertex of \(S_0^v\) and ending at \(u\). Since \(D^v\)
contains at most one orientation of any undirected edge of \(G_v\), these directed paths
correspond to \(\kappa\) pairwise edge-disjoint undirected paths in \(G_v\).

\textbf{Length Bound.} Finally, along every directed path in \(D^v\), activation time strictly increases at each
arc. A directed path starting in \(S_0^v\) starts at activation time \(0\) and ends at
activation time \(T(u)\). Hence, it has at most \(T(u)\) arcs, and so each resulting
undirected path has length at most \(T(u)\le \ell\).

\subsection{Proof of Lemma \ref{lem:res}}\label{proof4}

If some path \(P_i\) intersects \(S\) in more than one vertex, replace \(P_i\) by the
suffix of \(P_i\) beginning at its last vertex in \(S\) and ending at \(u\). This operation
does not increase its length and preserves pairwise edge-disjointness. Hence, we may assume
that each \(P_i\) meets \(S\) only at its initial endpoint.

Orient each path \(P_i\) from \(u\) toward its endpoint in \(S\), and send \(1/\kappa\)
units of current along \(P_i\). Since the paths are pairwise edge-disjoint, every used
edge carries current of magnitude exactly \(1/\kappa\), and all unused edges carry zero
current. The superposition of these path flows is a unit flow from \(u\) to the set \(S\):
the total outflow from \(u\) is \(1\), the total inflow into \(S\) is \(1\), and flow is
conserved at every vertex in \(W\setminus (S\cup\{u\})\).

By Thomson's principle, effective resistance is the minimum energy over all unit flows
from \(u\) to \(S\). Therefore,
\[
R_{\mathrm{eff}}^H(S,u)
\le
\sum_{e\in F} f(e)^2
=
\sum_{i=1}^{\kappa} \ell_i \left(\frac{1}{\kappa}\right)^2
=
\frac{1}{\kappa^2}\sum_{i=1}^{\kappa}\ell_i .
\]
If \(\ell_i\le L\) for all \(i\), this gives
\[
R_{\mathrm{eff}}^H(S,u)
\le
\frac{1}{\kappa^2}\cdot \kappa L
=
\frac{L}{\kappa}.
\]

\subsection{Proof of Theorem~\ref{thm:walklength}}\label{proof5}

Fix \(u\in S_\ell^v\setminus S_0^v\). By
Lemma~\ref{lem:paths}, the induced activated subgraph \(G_v\)
contains \(\kappa\) pairwise edge-disjoint paths from \(S_0^v\) to \(u\), each of length
at most \(T(u)\). Applying Lemma~\ref{lem:res} with
\(H=G_v\), \(S=S_0^v\), and \(L=T(u)\), we obtain
\[
R_{\mathrm{eff}}^{G_v}(S_0^v,u)
\le
\frac{T(u)}{\kappa}.
\]
Since \(u\) activates by time \(\ell\), \(T(u)\le \ell\). Hence,
\[
R_{\mathrm{eff}}^{G_v}(S_0^v,u)
\le
\frac{T(u)}{\kappa}
\le
\frac{\ell}{\kappa}.
\]

\subsection{Finite-Sample Cascade Scores and Short-Walk Reinforcement}
\label{app:finite-sample-cascades}

In the main text we define the short-walk event \(\mathcal F_L := \{(A^2+\cdots+A^L)_{uv}\ge \kappa\}\) as a population-level proxy for the empirical
co-activation score \(f_v(u)\). The following result makes this relation precise. It is stated
conditional on the input graph \(G\); if \(G\) is random, the same statement holds conditionally on
\(G\).

\begin{theorem}[Empirical Cascade Scores and Short-Walk Support]
\label{thm:finite-sample-cascades}
Let $G = (V, E)$ be a fixed undirected graph with adjacency matrix $A$ and $|V| = n$. For each seed $v \in V$ with $N_v \geq 1$, let $S_{v,b}$ denote the final active set of cascade realization $b$, and define $\widehat{p}_v(u) := \frac{1}{N_v} \sum_{b=1}^{N_v} \mathbf{1}\{u \in S_{v,b} \setminus \{v\}\}, \qquad p_v(u) := \mathbb{E}\bigl[\widehat{p}_v(u)\bigr]$. Assume that, for each such seed $v$, the $N_v$ realizations can be partitioned into independent blocks $B_{v,1}, \ldots, B_{v,m_v}$ with sizes $a_{v,j} = |B_{v,j}|$; dependence within a block is permitted. Let $N_{\mathrm{eff}}(v) := \frac{\bigl(\sum_{j=1}^{m_v} a_{v,j}\bigr)^2}{\sum_{j=1}^{m_v} a_{v,j}^2} = \frac{N_v^2}{\sum_{j=1}^{m_v} a_{v,j}^2}$ 
denote the effective number of independent cascade samples for seed $v$. Then the following statements hold.

\begin{enumerate}[leftmargin=0.5cm]
\item[\textnormal{(i)}] \emph{(Concentration.)} For any $\delta \in (0,1)$, with probability at least $1 - \delta$, for every ordered pair $(v, u)$ with $u \neq v$ and $N_v \geq 1$, $\bigl|\widehat{p}_v(u) - p_v(u)\bigr| \leq \varepsilon_v(\delta) := \sqrt{\frac{\log\bigl(2n(n-1)/\delta\bigr)}{2 N_{\mathrm{eff}}(v)}}$. Equivalently, $|f_v(u) - N_v p_v(u)| \leq N_v \varepsilon_v(\delta)$ uniformly over all such pairs. %We apply a union bound over the $n(n-1)$ ordered pairs $(v, u)$ with $u \neq v$, which already accounts for the union across seeds.

\item[\textnormal{(ii)}] \emph{(Top-$k$ Stability.)} Fix a seed $v$ and an integer $k$ with $1 \leq k \leq n - 2$. Let $T_k(v)$ denote the population top-$k$ set under $p_v$ and $\widehat{T}_k(v)$ the empirical top-$k$ set under $\widehat{p}_v$, using a common deterministic tie-breaking rule. Let $p_{v,(1)} \geq p_{v,(2)} \geq \cdots$ be the sorted values of $\{p_v(u) : u \neq v\}$, and define the population margin $\gamma_v := p_{v,(k)} - p_{v,(k+1)}$. If $\gamma_v > 2 \varepsilon_v(\delta)$, then $\widehat{T}_k(v) = T_k(v)$ on the event of part \textnormal{(i)}.

\item[\textnormal{(iii)}] \emph{(Short-Walk Certificate for Fixed-Threshold TAS.)} Consider fixed-threshold TAS with $\tau \geq 1$, initialization $S_0^v = \{v\} \cup R$ for some $R \subseteq N(v)$, and at most $\ell$ post-initialization vertex activations. The conclusion below applies both to the clean sequential TAS rule of Section~2.1 and to the hub-restricted implementation, provided that any activation requires at least $\tau$ distinct already-active support vertices. Then for every nonlocal node $u \notin N[v]$, $f_v(u) > 0$ implies $\sum_{s=2}^{\ell+1} (A^s)_{vu} \geq \tau$.

The same implication holds with $p_v(u) > 0$ in place of $f_v(u) > 0$. When counts are aggregated over a threshold set $\mathcal{T}$, the implication holds unconditionally with $\tau_{\min} := \min \mathcal{T}$; if threshold-specific counts $f_v^{(\tau)}(u)$ are retained, the implication holds for that threshold $\tau$.
\end{enumerate}
\end{theorem}

\begin{proof}
For \emph{(i)}, fix an ordered pair \((v,u)\), \(u\neq v\), with \(N_v\ge 1\). For block \(j\), define
\[
Y_{v,j}(u):=
\sum_{b\in B_{v,j}}\mathbf{1}\{u\in S_{v,b}\setminus\{v\}\}.
\]
The variables \(Y_{v,1}(u),\ldots,Y_{v,m_v}(u)\) are independent by assumption and satisfy
\(0\le Y_{v,j}(u)\le a_{v,j}\). Since
\[
\widehat p_v(u)=\frac{1}{N_v}\sum_{j=1}^{m_v}Y_{v,j}(u),
\]
Hoeffding's inequality for independent bounded random variables gives, for every \(\epsilon>0\),
\[
\Pr{
\big|\widehat p_v(u)-p_v(u)\big|\ge \epsilon
}
\le
2\exp\!\left(
-\frac{2N_v^2\epsilon^2}{\sum_{j=1}^{m_v}a_{v,j}^2}
\right)
=
2\exp\!\left(-2N_{\rm eff}(v)\epsilon^2\right).
\]
Taking
\[
\epsilon=
\sqrt{\frac{\log\big(2n(n-1)/\delta\big)}{2N_{\rm eff}(v)}}
\]
makes the failure probability at most \(\delta/(n(n-1))\). A union bound over all \(n(n-1)\) ordered
pairs \((v,u)\), \(u\neq v\), proves \emph{(i)}. Pairs with seeds for which no cascades are run are
excluded by the condition \(N_v\ge 1\).

For \emph{(ii)}, work on the event from \emph{(i)} and fix \(v\). For any
\(a\in T_k(v)\) and \(b\notin T_k(v)\), the population margin gives
\(p_v(a)-p_v(b)\ge \gamma_v\). Therefore,
\[
\widehat p_v(a)-\widehat p_v(b)
\ge p_v(a)-p_v(b)-2\varepsilon_v(\delta)
\ge \gamma_v-2\varepsilon_v(\delta)>0.
\]
Thus, every population top-\(k\) node has a strictly larger empirical score than every non-top-\(k\)
node, so the empirical and population top-\(k\) sets coincide.

For \emph{(iii)}, consider a single TAS-\(\tau\) realization from seed \(v\). Order the
post-initialization activations by their insertion time. We first show by induction that, after \(r\)
post-initialization activations, every active vertex other than \(v\) itself has a walk from \(v\) of
length at most \(r+1\), while \(v\) itself has a trivial length-zero walk. At initialization,
every vertex in \(R\subseteq N(v)\) has a length-one walk from \(v\), so the claim holds for \(r=0\).
If \(x\) is the \((r+1)\)-st activated vertex, the TAS rule activates \(x\) only after at least one
already-active support vertex \(z\in N(x)\) has been observed, because \(\tau\ge 1\). By the
induction hypothesis, \(z\) has a walk from \(v\) of length at most \(r+1\), or \(z=v\) has a
length-zero walk. Appending the edge \((z,x)\) gives a walk from \(v\) to \(x\) of length at most
\(r+2\). This proves the induction claim.

Now let \(u\notin N[v]\) be activated as the \(r\)-th post-initialization activation, where
\(1\le r\le \ell\). Since \(u\notin N[v]\), it was not active at initialization. At the activation
time of \(u\), the rule has observed at least \(\tau\) distinct already-active support vertices
\(z_1,\ldots,z_\tau\in N(u)\). By the induction claim, each \(z_i\) has a walk from \(v\) of length at
most \(r\), or a length-zero walk if \(z_i=v\). The latter case is impossible because
\(z_i\in N(u)\) and \(u\notin N[v]\) imply \(z_i\neq v\). Appending the distinct final edges
\((z_i,u)\) gives \(\tau\) walks from \(v\) to \(u\), each of length at most \(r+1\le \ell+1\). These
walks are distinct because they end with distinct second-to-last vertices \(z_i\). Because
\(u\notin N[v]\), each walk has length at least two. Hence,
\[
\sum_{s=2}^{\ell+1}(A^s)_{vu}\ge \tau .
\]
If $f_v(u) > 0$, some sampled realization activated $u$, so the deterministic argument above applies. If $p_v(u) > 0$, then a realization activating $u$ occurs with positive probability, and the same certificate follows. Under the nonzero-count convention of \eqref{eq:topk}, every selected nonlocal directed edge $v \to u$ has $f_v(u) > 0$, so the certificate applies to every such edge of $G^\ast$.
\end{proof}

\begin{corollary}[TAS Lower Bound From Two-Hop Common Neighbors]
\label{cor:tas-two-hop-lower-bound}
Consider TAS with threshold $\tau$, and suppose the initial active set is $S_0^v = \{v\} \cup R$, where $R$ is drawn uniformly from the $r_v$-subsets of $N(v)$ with $r_v \leq d(v)$. For a nonlocal node $u \notin N[v]$, define the number of two-hop common neighbors
\[
h_{vu} := |N(v) \cap N(u)| = (A^2)_{vu},
\]
and the hypergeometric tail
\[
\Phi_{\tau, r_v, d(v)}(h) := \sum_{i = \tau}^{\min\{r_v, h\}} \frac{\binom{h}{i}\binom{d(v) - h}{r_v - i}}{\binom{d(v)}{r_v}}.
\]
The function $h \mapsto \Phi_{\tau, r_v, d(v)}(h)$ is nondecreasing.

\begin{enumerate}[leftmargin=0.5cm]
\item[\textnormal{(i)}] \emph{(Parallel Activation.)} If TAS activates all initially eligible vertices in parallel, then for every nonlocal $u \notin N[v]$,
\[
p_v(u) \geq \Phi_{\tau, r_v, d(v)}\bigl((A^2)_{vu}\bigr).
\]

\item[\textnormal{(ii)}] \emph{(Sequential Uniform Activation.)} Suppose instead that, at each step, one initially eligible vertex is selected uniformly at random. Let
\[
M_v := \max\Bigl\{1,\, \max_{R} \bigl|\{x \in V \setminus S_0^v : \sigma_0(x) \geq \tau\}\bigr|\Bigr\}
\]
denote the maximum size of the initially eligible set over all initializations $R$, taken to be at least one. Then for every nonlocal $u \notin N[v]$,
\[
p_v(u) \geq \frac{1}{M_v} \, \Phi_{\tau, r_v, d(v)}\bigl((A^2)_{vu}\bigr).
\]
\end{enumerate}

On the concentration event of Theorem~\ref{thm:finite-sample-cascades}, the same lower bounds hold for the empirical score $\widehat{p}_v(u)$ up to a subtractive correction of $\varepsilon_v(\delta)$.
\end{corollary}

\begin{proof}
Let \(S_v\) denote the final active set of a generic clean-TAS realization from seed \(v\). For
\(u \notin N[v]\), the initial support of \(u\) is
\[
\sigma_0(u) = |N(u) \cap R|.
\]
Among the \(d(v)\) vertices in \(N(v)\), exactly
\[
h_{vu}=|N(v)\cap N(u)|=(A^2)_{vu}
\]
are also neighbors of \(u\). Since \(R\) is sampled uniformly among all \(r_v\)-subsets of \(N(v)\),
\[
\sigma_0(u)\sim \operatorname{Hypergeom}\bigl(d(v),h_{vu},r_v\bigr).
\]
Therefore,
\[
\Pr{\sigma_0(u)\ge \tau}
=
\Phi_{\tau,r_v,d(v)}(h_{vu}).
\]

\noindent\textbf{Parallel Activation.}
Assume first that clean TAS activates all initially eligible vertices in parallel, and that at least
one activation layer is performed. On the event \(\{\sigma_0(u)\ge \tau\}\), the vertex \(u\) is
eligible at the first layer. Hence, \(u\) activates at the first layer, so
\[
\{\sigma_0(u)\ge \tau\}
\subseteq
\{u\in S_v\}.
\]
Since \(u\neq v\), this gives
\[
p_v(u)
=
\Pr{u\in S_v\setminus\{v\}}
\ge
\Pr{\sigma_0(u)\ge \tau}
=
\Phi_{\tau,r_v,d(v)}(h_{vu}).
\]

\noindent\textbf{Sequential Uniform Activation.}
Now suppose instead that clean TAS activates one initially eligible vertex at a time, choosing
uniformly among the initially eligible vertices, and that at least one activation step is performed.
For an initialization \(R\), let
\[
\mathcal E(R)
:=
\{x\in V\setminus S_0^v:\sigma_0(x)\ge \tau\}
\]
be the initially eligible set. Define
\[
M_v
:=
\max\left\{1,\max_R |\mathcal E(R)|\right\}.
\]
Condition on the event \(\{\sigma_0(u)\ge \tau\}\). Since \(u\notin N[v]\), we have
\(u\notin S_0^v\), and so \(u\in \mathcal E(R)\). Thus,
\[
1\le |\mathcal E(R)|\le M_v.
\]
Because the first activated vertex is chosen uniformly from \(\mathcal E(R)\), the conditional
probability that \(u\) is selected at the first activation step is
\[
\frac{1}{|\mathcal E(R)|}
\ge
\frac{1}{M_v}.
\]
Therefore,
\begin{align*}
    p_v(u)&\ge \Pr{u\text{ is selected at the first activation step}}\\
    &=\mathbb E\!\left[ \mathbf 1\{\sigma_0(u)\ge \tau\}\frac{1}{|\mathcal E(R)|}
\right] \ge \frac{1}{M_v}\Pr{\sigma_0(u)\ge \tau}.
\end{align*}

Using the hypergeometric identity above,
\[
p_v(u) \ge \frac{1}{M_v} \Phi_{\tau,r_v,d(v)}(h_{vu}).
\]

\noindent\textbf{Monotonicity of \(\Phi\).}
To prove monotonicity, couple the hypergeometric experiment with \(h\) marked elements to the one
with \(h+1\) marked elements by changing one unmarked element into a marked element and using the
same sampled subset. Under this coupling, the number of marked sampled elements cannot decrease, so the probability of sampling at least \(\tau\) marked elements is nondecreasing in \(h\).
Therefore, \(\Phi_{\tau,r_v,d(v)}(h)\) is nondecreasing in \(h\).

\noindent\textbf{Empirical Bounds.}
On the concentration event of Theorem~\ref{thm:finite-sample-cascades},
\[
|\widehat p_v(u)-p_v(u)|\le \varepsilon_v(\delta)
\]
simultaneously over all ordered pairs \((v,u)\) with \(u\neq v\) and \(N_v\ge 1\). Therefore, for parallel clean TAS,
\[
\widehat p_v(u)
\ge
p_v(u)-\varepsilon_v(\delta)
\ge
\Phi_{\tau,r_v,d(v)}(h_{vu})-\varepsilon_v(\delta),
\]
and for sequential uniform clean TAS,
\[
\widehat p_v(u)
\ge
p_v(u)-\varepsilon_v(\delta)
\ge
\frac{1}{M_v}\Phi_{\tau,r_v,d(v)}(h_{vu})-\varepsilon_v(\delta).
\]
Since the concentration event is uniform, these empirical bounds hold simultaneously for all
nonlocal pairs \(u\notin N[v]\) and all seeds \(v\) with \(N_v\ge 1\), with probability at least
\(1-\delta\).
\end{proof}

%%%%%New for Camera ready

\paragraph{One-round, parallel, nonlocal TAS.}
Fix integers $1\leq\tau\leq r$, independent of $n$.
For each seed $v$, let $U_v:=V\setminus N[v]$ and
$r_v:=\min\{r,d(v)\}$. Draw $N_v\geq1$ subsets $R_{v,b}$, each
uniform among the $r_v$-subsets of $N(v)$ conditional on $G$.
In realization $b$, activate exactly those $u\in U_v$ satisfying
$|N(u)\cap R_{v,b}|\geq\tau$, with no further propagation.
Let $\widehat p_v(u)$ be the empirical activation frequency. Select the
$k$ highest positive scores (all positive scores if fewer exist), with
fixed deterministic ties. Symmetrize the selected pairs to obtain the
unweighted graph $\widehat G$; retain no original edges.
For each seed, assume independent blocks of realizations conditional on
$G$, with block sizes $a_{v,j}$ and sample counts fixed given $G$, and set
\[
 N_{\mathrm{eff}}(v):=N_v^2/\sum_j a_{v,j}^2.
\]
Within-block dependence is allowed; independence across seeds is not
required. Independent realizations give $N_{\mathrm{eff}}(v)=N_v$.

\begin{corollary}[End-to-end guarantee for one-round TAS]
\label{cor:one-round-tas-end-to-end}
Under the fixed-parameter SBM assumptions of
Theorem~\ref{thm:sbm-profile-homophily}, including $\Delta_2>0$,
assume $B_{aa}<1$ for every class $a$.
For the variant above, there exist constants $c_0,c_1>0$, depending
only on $(B,\pi,r,\tau)$, such that, for all sufficiently large $n$,
every $\delta\in(0,1)$ and integer $1\leq k\leq c_0n$, if
\[
 \min_v N_{\mathrm{eff}}(v)\geq c_1\log(n/\delta),
\]
then $\widehat G$ is perfectly homophilic with minimum degree at least
$k$, with probability at least $1-2n^{-2}-\delta$.

If, conditional on the labels, the features
$X_i\sim\mathcal N(\mu_{y_i},\sigma^2I_d)$ are independent across
vertices and jointly independent of the graph and cascade sampling,
with $\sigma>0$ and $\Delta:=\min_{a\ne b}\|\mu_a-\mu_b\|_2>0$,
then classifying the uniform feature average over each vertex and its
neighbors in $\widehat G$ by the nearest true class mean gives
\[
 \mathbb E\!\left[\frac1n\sum_{i=1}^n
                     \mathbf 1\{\widehat y_i\ne y_i\}\right]
 \leq 2n^{-2}+\delta+
       (C-1)\exp\!\left(-\frac{(k+1)\Delta^2}{8\sigma^2}\right).
\]
Probability and expectation are conditional on the labels and include
all construction and feature randomness.
\end{corollary}

\begin{proof}
We condition throughout on the labels. Write
$K:=B\operatorname{diag}(\pi)B^\top$ and define
\[
 \overline d_a:=\sum_b\pi_bB_{ab},\qquad
 \theta_{ab}:=K_{ab}/\overline d_a,\qquad
 H_{r,\tau}(x):=\sum_{j=\tau}^r\binom rj x^j(1-x)^{r-j},
\]
with continuous endpoint values. Let
\[
 \eta_0:=\min_a\left[H_{r,\tau}(\theta_{aa})-
                         \max_{b\ne a}H_{r,\tau}(\theta_{ab})\right],
 \qquad \rho:=\min_a\pi_a(1-B_{aa}).
\]
We will take $c_0:=\rho/2$ and $c_1:=96/\eta_0^2$.
First, $0\leq K_{ab}\leq\overline d_a$, since $B_{bc}\in[0,1]$.
The positive profile margin implies $K_{aa}>0$ and hence
$\overline d_a>0$. For $1\leq\tau\leq r$, $H_{r,\tau}$ is strictly
increasing on $[0,1]$: on $(0,1)$ its derivative is
\[
 H_{r,\tau}'(x)
 =r\binom{r-1}{\tau-1}x^{\tau-1}(1-x)^{r-\tau}>0.
\]
Thus $\theta_{aa}>\theta_{ab}$ for $b\ne a$ implies $\eta_0>0$.
The condition $B_{aa}<1$ gives $\rho>0$.

\emph{Exact activation probabilities.}
For a fixed graph and $u\in U_v$, put
$h_{vu}:=|N(v)\cap N(u)|=(A^2)_{vu}$.
Since $u\notin N[v]$, its initial support comes only from $R_{v,b}$.
Consequently, the one-round activation probability is exactly
\[
 p_v(u):=\mathbb E[\widehat p_v(u)\mid G]
 =\sum_{j=\tau}^{\min\{r_v,h_{vu}\}}
   \frac{\binom{h_{vu}}{j}
         \binom{d(v)-h_{vu}}{r_v-j}}
        {\binom{d(v)}{r_v}}.
\]
We use the conventions that an empty sum is zero and that a binomial
coefficient with an invalid lower argument is zero.
This is an equality because only the initial parallel round is used.

\emph{Uniform SBM event.}
For $y_v=a$ and $y_u=b$, $u\ne v$, the conditional expectations satisfy
\[
 \mathbb E[d(v)]=n\overline d_a-B_{aa},\qquad
 \mathbb E[h_{vu}]=nK_{ab}-B_{ab}(B_{aa}+B_{bb}).
\]
The number
$M_v:=|\{u\in U_v:y_u=y_v\}|$ of same-class nonlocal candidates has
expectation $(n_a-1)(1-B_{aa})$.
Each of $d(v)$, $h_{vu}$, and $M_v$ is a sum of at most $n$ independent
Bernoulli variables conditional on the labels. For $h_{vu}$, the summands
are $A_{vw}A_{wu}$ over distinct intermediate vertices $w$; these
summands use disjoint edges.
Hoeffding's inequality with $t_n:=\sqrt{3n\log n}$ and a union bound over
all these variables gives a simultaneous event $\mathcal E_n$ with
\[
 \Pr(\mathcal E_n^c)
 \leq 2n^{-4}+4n^{-5}\leq 2n^{-2}\qquad(n\geq2),
\]
on which all deviations from their expectations are at most $t_n$.
Therefore, uniformly over vertices and distinct pairs,
\[
 \frac{d(v)}n=\overline d_a+
       O\!\left(\sqrt{\frac{\log n}{n}}\right),\qquad
 \frac{h_{vu}}n=K_{ab}+
       O\!\left(\sqrt{\frac{\log n}{n}}\right),
\]
and, for all sufficiently large $n$,
$d(v)\geq r$ and $M_v\geq\rho n/2\geq k$.

For fixed $r$ and $d(v)\geq r$, sampling with and without replacement
can be coupled until the first repeated draw. Hence
\[
 \left|p_v(u)-H_{r,\tau}\!\left(\frac{h_{vu}}{d(v)}\right)\right|
 \leq\frac{r(r-1)}{2d(v)}.
\]
Since $\min_a\overline d_a>0$ and $H_{r,\tau}$ is a fixed polynomial,
it follows that, on $\mathcal E_n$ and for all sufficiently large $n$,
\[
 \sup_{\substack{v\in V,\ u\in U_v}}
 \left|p_v(u)-H_{r,\tau}(\theta_{y_vy_u})\right|
 \leq\frac{\eta_0}{4}.
\]
Thus every same-class nonlocal candidate has population score at least
$3\eta_0/4$, and every same-class nonlocal score exceeds every
cross-class nonlocal score for the same seed by at least $\eta_0/2$.
The latter comparison is vacuous when no cross-class candidate exists.

\emph{Finite-sample selection.}
For $n\geq2$ and $\delta\in(0,1)$,
$\log(2n(n-1)/\delta)\leq3\log(n/\delta)$, so our sampling condition
implies $N_{\mathrm{eff}}(v)\geq
32\eta_0^{-2}\log(2n(n-1)/\delta)$ for every seed.
Conditional on $G$, the sum of activation indicators within block $j$
lies in $[0,a_{v,j}]$, and these block sums are independent for a fixed
seed. Hoeffding's inequality followed by a union bound over at most
$n(n-1)$ ordered pairs yields, with conditional probability at least
$1-\delta$,
\[
 |\widehat p_v(u)-p_v(u)|
 \leq\sqrt{\frac{\log(2n(n-1)/\delta)}{2N_{\mathrm{eff}}(v)}}
 \leq\frac{\eta_0}{8}
 \quad\text{for every }v\in V,\ u\in U_v.
\]
On this event and $\mathcal E_n$, every same-class nonlocal score is
strictly positive, and every such score exceeds every cross-class
nonlocal score by at least $\eta_0/4$.
Since $M_v\geq k$, each seed selects exactly $k$ same-class vertices.
Symmetrization preserves label agreement and cannot remove these $k$
neighbors. The combined construction-failure probability is at most
$2n^{-2}+\delta$.

\emph{Classification.}
Let $q_i:=1+|N_{\widehat G}(i)|$ and
$Z_i:=q_i^{-1}\sum_{j\in\{i\}\cup N_{\widehat G}(i)}X_j$.
The classifier is $\widehat y_i:=\arg\min_a\|Z_i-\mu_a\|_2^2$,
with fixed deterministic ties.
Condition on a successful output graph and the labels. The construction
uses the original graph and feature-independent sampling randomness,
so the selected features remain independent Gaussians conditional on
that output. Every selected neighbor of $i$ has label $y_i$ and
$q_i\geq k+1$, giving
\[
 Z_i\sim\mathcal N\!\left(\mu_{y_i},\frac{\sigma^2}{q_i}I_d\right).
\]
For an incorrect class $b\ne y_i$, the nearest-mean decision can prefer
$b$ only if a one-dimensional Gaussian projection crosses half the
squared mean separation. The standard Gaussian tail bound gives
\[
 \Pr\!\left[\|Z_i-\mu_b\|_2^2\leq\|Z_i-\mu_{y_i}\|_2^2
       \mid\widehat G,\text{ successful construction},y\right]
 \leq \exp\!\left(-\frac{q_i\|\mu_{y_i}-\mu_b\|_2^2}
                            {8\sigma^2}\right).
\]
A union bound over incorrect classes gives the stated single-node tail.
Averaging over nodes and bounding the loss by one on construction failure
proves the expected-error bound. The construction-failure term is
included only once.
\end{proof}

%%%%%%%%% 
\subsection{Edge-Connectivity Certificate for Fixed-Threshold TAS}
\label{app:edge-connectivity}

Theorem~\ref{thm:walklength} bounds the effective resistance of threshold activations. The same activation-path argument yields a complementary edge-connectivity consequence, which we record here. Throughout this section, $\kappa$ denotes a single fixed threshold; when the final TAS graph aggregates over multiple thresholds, the certificate below applies to each cascade run separately, with $\kappa$ taken to be the threshold of the contributing run.

\begin{definition}[Set-To-Vertex Edge Connectivity]
Let $G = (V, E)$ be an undirected graph. For a nonempty set $S \subseteq V$ and a vertex $u \in V \setminus S$, define
\[
\lambda_G(S, u) := \min\bigl\{ |F| : F \subseteq E,\ \text{no path from } S \text{ to } u \text{ in } G \setminus F \bigr\}.
\]
By Menger's theorem, $\lambda_G(S, u)$ equals the maximum number of pairwise edge-disjoint paths from $S$ to $u$ in $G$.
\end{definition}

Let $F_k^{\kappa}(v)$ denote the directed top-$k$ cascade-neighbor set obtained from positive co-activation counts in TAS-$\kappa$ runs seeded at $v$. A selected directed edge $v \to u$ is \emph{nonlocal} if $u \notin N[v]$.

\begin{proposition}[Edge-Connectivity Certificate]
\label{prop:tas-edge-connectivity-certificate}
Consider a TAS-$\kappa$ cascade from seed $v$ with initial active set $S_0^v = \{v\} \cup R$, where $R \subseteq N(v)$, and walk length $\ell$. If a node $u \notin S_0^v$ is activated within $\ell$ steps, then $G$ contains $\kappa$ pairwise edge-disjoint paths from $S_0^v$ to $u$, each of length at most the activation time $T(u) \leq \ell$. In particular,
$\lambda_G(S_0^v, u) \geq \kappa$. If $v \to u$ is a nonlocal directed edge in the TAS-$\kappa$ top-$k$ graph, then $\lambda_G(N[v], u) \geq \kappa$.

For the symmetrized graph, the certificate holds in at least one of the two selected orientations: if $\{v, u\}$ appears because $v \to u$ was selected, then $\lambda_G(N[v], u) \geq \kappa$; if it appears because $u \to v$ was selected, then $\lambda_G(N[u], v) \geq \kappa$.
\end{proposition}

\begin{proof}
Let $u \notin S_0^v$ activate within $\ell$ steps. By Lemma~\ref{lem:paths}, the activated subgraph $G[S_\ell^v]$ contains $\kappa$ pairwise edge-disjoint paths from $S_0^v$ to $u$, each of length at most $T(u) \leq \ell$. These paths exist in $G$ as well, so by Menger's theorem $\lambda_G(S_0^v, u) \geq \kappa$.

Suppose now that $v \to u$ is a nonlocal directed edge selected by the TAS-$\kappa$ top-$k$ rule. Top-$k$ selection requires positive co-activation, so some cascade run from $v$ activated $u$. Since $S_0^v \subseteq N[v]$ and $u \notin N[v]$, the activation cannot be due to initialization; hence, $u$ was reached by propagation. The first part gives $\lambda_G(S_0^v, u) \geq \kappa$ for that run.

Edge connectivity is monotone in the source set: if $S \subseteq S'$ and $u \notin S'$, then $\lambda_G(S', u) \geq \lambda_G(S, u)$, because every cut separating $S'$ from $u$ also separates $S$ from $u$. Since $S_0^v \subseteq N[v]$ and $u \notin N[v]$, we obtain $\lambda_G(N[v], u) \geq \lambda_G(S_0^v, u) \geq \kappa$. The symmetrized statement follows by applying the directed argument to whichever orientation was selected before symmetrization.
\end{proof}

\begin{corollary}[Cut Obstruction]
\label{cor:tas-cut-obstruction}
Let $V = A \sqcup B$ be a partition of the vertices and let
$E_G(A,B) := \{(a,b)\in E : a\in A,\ b\in B\}$.
If $|E_G(A,B)| < \kappa$, $N[v] \subseteq A$, and $u \in B$, then the nonlocal directed edge $v \to u$ cannot be selected by TAS-$\kappa$ cascade rewiring.
\end{corollary}

\begin{proof}
The cross-cut $E_G(A, B)$ separates $N[v]$ from $u$, so $\lambda_G(N[v], u) \leq |E_G(A, B)| < \kappa$. By Proposition~\ref{prop:tas-edge-connectivity-certificate}, any selected nonlocal directed edge $v \to u$ must satisfy $\lambda_G(N[v], u) \geq \kappa$, a contradiction.
\end{proof}

\begin{corollary}[Reachable-Certificate Set]
\label{cor:certificate-set}
For each vertex $v$, define
\[
B_\kappa(v) := \{ u \in V \setminus N[v] : \lambda_G(N[v], u) \geq \kappa \}.
\]
Every nonlocal directed neighbor selected by TAS-$\kappa$ lies in $B_\kappa(v)$:
\[
F_k^\kappa(v) \setminus N[v] \subseteq B_\kappa(v).
\]
For an undirected edge $\{v, u\}$ in the symmetrized nonlocal graph, at least one of $u \in B_\kappa(v)$ or $v \in B_\kappa(u)$ holds.
\end{corollary}

\begin{proof}
The directed statement is Proposition~\ref{prop:tas-edge-connectivity-certificate}. The undirected statement follows because $\{v, u\}$ is included only if at least one of $v \to u$ or $u \to v$ was selected.
\end{proof}

\begin{remark}
Proposition~\ref{prop:tas-edge-connectivity-certificate} shows that TAS-$\kappa$ acts as a redundancy filter: it promotes a nonlocal pair only when the target is already connected to the seed's closed neighborhood by at least $\kappa$ edge-disjoint paths in $G$. Cascade rewiring therefore reinforces mesoscopic regions with redundant short-path support, but cannot create shortcuts across genuine bottlenecks. This is the structural reason cascade rewiring is complementary to bottleneck-targeted methods such as \textsc{FoSR} and \textsc{SDRF}, which are designed to add edges across small cuts.

The certificate is threshold-specific. If counts are aggregated across multiple thresholds, an edge's certificate corresponds to the threshold(s) at which it co-activated; without threshold-specific counts, only the minimum threshold used gives an unconditional certificate. MAS admits an analogous certificate, but only along cascade prefixes where every activation up to the target has support at least $\kappa$.
\end{remark}

\section{Constructing the Multiset $\mathcal{M}_v$: Pseudocode and Complexity Analysis} \label{app:time_complexity}

\subsection{Pseudocode and Implementation}

In this section, we present pseudocode for MAS and TAS procedures introduced in Section~\ref{sec:contagion}. 
These procedures, originally formulated as adjacency-search rules by \citet{chaitanyaadjacency}, are reinterpreted here as \textit{discrete-time} contagion dynamics. 
MAS and TAS systematically explore node neighborhoods to construct multisets $\mathcal{M}_v$, which capture local and mesoscopic structural information and form the basis for our graph rewiring strategies.

MAS first explores densely connected regions by activating nodes with maximal neighbor support and gradually lowers the activation threshold to reach sparser regions. The procedure relies on the following parameters:

\begin{itemize}[leftmargin=0.5cm]
    \item \textbf{Walk Length ($\ell$):} Number of steps in the contagion; default $= 10$.
    \item \textbf{Number of Neighbors ($r$):} Size of the random subset $R \subseteq N(v)$ selected for initialization; default $= 5$.
    \item \textbf{Number of Permutations ($P$):} Number of times the neighbor list is permuted; default $= 5$.
    \item \textbf{Top-$k$ Parameter $k$}: Number of co-activated nodes to be selected; default $=$ average degree.
    \item \textbf{Hub Threshold $D$}: A fixed degree cutoff that determines whether an activated node is allowed to propagate the contagion based on its degree; default $=$ median degree.
\end{itemize}

TAS activates nodes only when they receive sufficient multi-neighbor reinforcement, governed by a global structural threshold $\tau$. TAS uses the same parameters as MAS, with the addition of a threshold set $\mathcal{T}$ that determines the values $\tau$ iterates over. By default, $\mathcal{T} = \{1,2,3,4,5\}$.

\begin{algorithm}[ht!]
\caption{Maximum Adjacency Search (MAS) with Seed-Aware Hub-Restricted Propagation}
\label{alg:mas-hub-restricted}
\begin{algorithmic}[1]
\Require Graph $G=(V,E)$, walk length $\ell$, number of starting neighbors $r$,
number of permutations $P$, top-$k$ parameter $k$, hub threshold $D$
\Ensure Dictionary $\mathcal F$ storing the top-$k$ co-activation counts for each seed node

\State $\mathcal F \gets$ empty dictionary

\For{$v\in V$}
    \State $N_v \gets N(v)$ \Comment{seed-neighborhood access}
    \State $\mathrm{counter} \gets$ empty sparse map
    \If{$|N_v|>0$}
        \For{$p=1,\ldots,P$}
            \State $\pi_v \gets \mathrm{permute}(N_v)$
            \For{$i\in\{0,r,2r,\ldots,|N_v|-1\}$}
                \State $s \gets \min\{i,\max(0,|N_v|-r)\}$
                \State $R \gets \pi_v[s:s+r]$
                \State $S \gets \{v\}\cup R$ \Comment{initial active set}
                \State $\mathrm{support} \gets$ empty sparse map
                \State $\mathrm{processed} \gets \emptyset$
                \State $Q \gets \{x\in S : d(x)\le D\}$ 

                \While{$Q\neq\emptyset$}
                    \State $x \gets \mathrm{pop}(Q)$
                    \If{$x\in \mathrm{processed}$}
                        \State \textbf{continue}
                    \EndIf
                    \State $\mathrm{processed} \gets \mathrm{processed}\cup\{x\}$
                    \For{$y\in N(x)$}
                        \If{$y\notin S$}
                            \State $\mathrm{support}[y]\gets \mathrm{support}[y]+1$
                        \EndIf
                    \EndFor
                \EndWhile

                \For{$t=1,\ldots,\ell$}
                    \If{there is no $y\notin S$ with $\mathrm{support}[y]>0$}
                        \State \textbf{break}
                    \EndIf
                    \State $u \gets \arg\max_{y\notin S}\mathrm{support}[y]$
                    \State $S \gets S\cup\{u\}$
                    \State delete $\mathrm{support}[u]$

                    \If{$d(u)\le D$}
                        \For{$y\in N(u)$}
                            \If{$y\notin S$}
                                \State $\mathrm{support}[y]\gets \mathrm{support}[y]+1$
                            \EndIf
                        \EndFor
                    \EndIf
                    \Comment{if $d(u)>D$, then $u$ is active but not expanded}
                \EndFor

                \For{$u\in S\setminus\{v\}$}
                    \State $\mathrm{counter}[u]\gets \mathrm{counter}[u]+1$
                \EndFor
            \EndFor
        \EndFor
    \EndIf
    \State $\mathcal F[v]\gets \mathrm{TopK}(\mathrm{counter},k)$
    \Comment{select top-$k$ }
\EndFor

\State \Return $\mathcal F$
\end{algorithmic}
\end{algorithm}

\paragraph{Seed-Aware Hub-Restricted Propagation.}
We distinguish between using a hub as a seed and using a hub as a propagation source. For a seed node $v$, the algorithm may access $N(v)$ to form the initial subset $R\subseteq N(v)$; this cost is charged once per permutation of the seed neighborhood. After initialization, however, an activated node $u$ is
expanded only if $d(u)\le D$, where $D$ is a fixed degree threshold. If $d(u)>D$, then $u$ remains active and contributes to the co-activation counts, but its adjacency list is not scanned and the contagion does not propagate through it. Thus, when a hub is the seed, repeated chunks of $N(v)$ still probe the different regions incident to the hub, while cascades seeded elsewhere do not expand explosively by propagating through a hub reached later in the process.

\noindent\textbf{Sparse Counting/Top-$k$ Selection.}
Although the pseudocode may suggest a counter over all nodes, the implementation maintains a \emph{sparse} associative map (hash/dictionary) that stores counts only for nodes that appear in cascades. The top-$k$ set is computed over these nonzero entries via a heap, rather than by sorting all $n$ nodes.

\noindent\textbf{Constant-Parameter Regime.}
In our experiments, $P,\ell,r,k,$ and $|\mathcal{T}|$ are treated as constants independent of $n$ and $m$ (i.e., small fixed hyperparameters). We use adjacency lists to iterate over the neighbors of a node $u$, and the associated complexity is $O(\deg(u))$.

\subsection{Time Complexity}

\paragraph{Seed-Neighborhood Initialization Cost.}
For a source node $v$, the algorithm accesses $N(v)$ only to form the initial
subsets $R\subseteq N(v)$. Let
\[
q_v := \max\left\{1,\left\lceil \frac{d(v)}{r}\right\rceil\right\}
\]
be the number of neighbor chunks used for $v$. Across one permutation of
$N(v)$, forming all chunks costs $O(d(v))$. Hence, across $P$ permutations and
$|\mathcal T|$ thresholds, the total seed-neighborhood initialization cost is
\[
O\!\left(P|\mathcal T|\sum_{v\in V} d(v)\right)
=
O(P|\mathcal T|m).
\]
For fixed $P$ and $|\mathcal T|$, this is $O(m)$.

\paragraph{Number of Cascade Initializations.}
The number of chunks for source $v$ is $q_v$, so
\[
\sum_{v\in V} q_v
\le
\sum_{v\in V}\left(1+\frac{d(v)}{r}\right)
=
n+\frac{2m}{r}.
\]
Therefore, MAS runs at most \(P\left(n+\dfrac{2m}{r}\right)\) cascades and TAS runs at most \(P|\mathcal T|\left(n+\dfrac{2m}{r}\right)\) cascades.

\begin{algorithm}[ht]
\caption{Threshold Adjacency Search (TAS) with Seed-Aware Hub-Restricted Propagation}
\label{alg:tas-hub-restricted}
\begin{algorithmic}[1]
\Require Graph $G=(V,E)$, walk length $\ell$, number of starting neighbors $r$,
number of permutations $P$, threshold set $\mathcal T$, top-$k$ parameter $k$,
hub threshold $D$
\Ensure Dictionary $\mathcal F$ storing the top-$k$ co-activation counts for each source node

\State $\mathcal F \gets$ empty dictionary

\For{$v\in V$}
    \State $N_v \gets N(v)$ \Comment{seed-neighborhood access}
    \State $\mathrm{counter} \gets$ empty sparse map
    \If{$|N_v|>0$}
        \For{$\tau\in\mathcal T$}
            \For{$p=1,\ldots,P$}
                \State $\pi_v \gets \mathrm{permute}(N_v)$
                \For{$i\in\{0,r,2r,\ldots,|N_v|-1\}$}
                    \State $s \gets \min\{i,\max(0,|N_v|-r)\}$
                    \State $R \gets \pi_v[s:s+r]$
                    \State $S_0 \gets \{v\}\cup R$
                    \State $S \gets \textsc{HubRestricted-Delayed-BFS}(G,S_0,\ell,\tau,D)$

                    \For{$u\in S\setminus\{v\}$}
                        \State $\mathrm{counter}[u]\gets \mathrm{counter}[u]+1$
                    \EndFor
                \EndFor
            \EndFor
        \EndFor
    \EndIf
    \State $\mathcal F[v]\gets \mathrm{TopK}(\mathrm{counter},k)$
    \Comment{select top-$k$}
\EndFor

\State \Return $\mathcal F$
\end{algorithmic}
\end{algorithm}

\paragraph{Hub-Restricted Propagation Cost.}
Each cascade starts from at most $r+1$ active nodes and activates at most
$\ell$ additional nodes. Thus, each cascade processes at most $r+1+\ell$ active
nodes. After initialization, an active node $u$ is expanded only if $d(u)\le D$.
If $d(u)>D$, then $u$ is counted but $N(u)$ is not scanned. Hence, every
expanded node contributes at most $D$ adjacency inspections, and every cascade
has propagation cost
\[
O(D(r+\ell)).
\]
For fixed $D,r,\ell$, this is constant per cascade. Combining with the cascade
count above gives total propagation cost
\[
O\!\left(
P|\mathcal T|D(r+\ell)
\left(n+\frac{2m}{r}\right)
\right)
=
O(n+m)
\]
when $P,|\mathcal T|,D,r,\ell$ are fixed constants. In the experiments $D$ is set to the median degree. 

\paragraph{Sparse Counting and Top-$k$ Extraction.}
For each seed $v$, each cascade contributes at most $r+1+\ell$ activated nodes
to the sparse co-activation counter. Therefore, the total number of nonzero
counter updates over all seeds is
\[
O\!\left(
P|\mathcal T|(r+\ell)
\sum_{v\in V} q_v
\right)
=
O(n+m)
\]
under fixed cascade parameters.

Let $C_v$ be the set of nodes with a nonzero co-activation count for seed $v$.
Using linear-time selection, the top-$k$ nodes can be extracted in
$O(|C_v|+k)$ time. Thus, top-$k$ extraction over all seeds costs
\[
O\!\left(\sum_v |C_v| + nk\right)
=
O(n+m+nk).
\]
If $k$ is fixed, or if $k$ is set to the rounded average degree
$k=O(1+m/n)$ as in our experiments, then $nk=O(n+m)$.

\begin{algorithm}[ht!]
\caption{\textsc{HubRestricted-Delayed-BFS}}
\label{alg:hub-restricted-delayed-bfs}
\begin{algorithmic}[1]
\Require Graph $G=(V,E)$, initial active set $S_0$, walk length $\ell$,
threshold $\tau$, hub threshold $D$
\Ensure Activated set $S$

\State $S \gets S_0$
\State $\mathrm{support} \gets$ empty sparse map
\State $\mathrm{processed} \gets \emptyset$
\State $I \gets$ queue initialized with all $x\in S_0$ such that $d(x)\le D$
\State $\mathrm{new} \gets 0$

\While{$I\neq\emptyset$ and $\mathrm{new}<\ell$}
    \State $x \gets \mathrm{dequeue}(I)$
    \If{$x\in \mathrm{processed}$}
        \State \textbf{continue}
    \EndIf
    \State $\mathrm{processed} \gets \mathrm{processed}\cup\{x\}$

    \For{$y\in N(x)$}
        \If{$y\notin S$}
            \State $\mathrm{support}[y]\gets \mathrm{support}[y]+1$

            \If{$\mathrm{support}[y]\ge \tau$}
                \State $S\gets S\cup\{y\}$
                \State $\mathrm{new}\gets \mathrm{new}+1$
                \State delete $\mathrm{support}[y]$

                \If{$d(y)\le D$}
                    \State $\mathrm{enqueue}(I,y)$
                \EndIf
                \Comment{if $d(y)>D$, then $y$ is active but not expanded}

                \If{$\mathrm{new}=\ell$}
                    \State \textbf{break}
                \EndIf
            \EndIf
        \EndIf
    \EndFor
\EndWhile

\State \Return $S$
\end{algorithmic}
\end{algorithm}

\paragraph{Overall Complexity.}
Under seed-aware hub-restricted propagation, fixed $P,|\mathcal T|,r,\ell,D$,
sparse counters, and $k=O(1+m/n)$, the total construction time of
$G^*=(V,E^\ast,W^\ast)$ is
\[
O(n+m).
\]
The output size is also bounded by
\[
|E^\ast| \le nk = O(n+m),
\]
and when $k$ is set to the average degree, $|E^\ast|=O(m)$.

Without the hub-restricted propagation rule, the full-exposure implementation
has output-sensitive runtime
\[
O\!\left(
\sum_{v\in V}
\sum_{c\in\mathcal C(v)}
\sum_{u\in A_{v,c}} d(u)
\right),
\]
which can be superlinear on graphs where high-degree hubs are repeatedly reached
from many seeds. Hub-restricted propagation replaces these repeated hub scans
by $O(1)$ counting operations, yielding the worst-case linear bound above.

\section{Further Experiments and Implementation Details}

This section provides implementation details and additional experimental results that complement the main results of the paper. We focus on robustness analysis, parameter sensitivity, and extended evaluations across graph structures to better understand the behavior of the proposed methods. The main results are presented in Section \ref{sec:empirical}, with standard deviations provided in Appendix~\ref{app:numerical1_appendix}.

\subsection{Implementation Details} \label{app:exp}

We evaluate the GNN and GT contagion-based models using two layers with a hidden dimension of 512, a dropout rate of 10\%, and all other model parameters set to their default. Attention-based architectures (GAT and all GTs) additionally employ $8$ attention heads. In GNN architectures, ReLU activation is applied between layers before dropout, and log-softmax activation is applied at the output layer.

All experiments are implemented in PyTorch and PyTorch Geometric. All training is performed using the Adam optimizer with a peak learning rate of $0.01$ and weight decay of $1\times10^{-5}$, minimizing the negative log-likelihood loss. We employ a linear learning rate scheduler with $500$ warm-up steps, followed by linear decay from $0.01$ to $0.0001$ over $1000$ total training steps. Each model is trained for up to $2000$ epochs using a batch size of $2000$. Early stopping is applied with a patience of $50$ epochs.

\textit{\textbf{Hardware.}} All experiments were conducted on a workstation equipped with a single AMD Ryzen Threadripper PRO 5955WX processor (16 cores, 4.00–4.50GHz, 64MB cache, PCIe 4.0), one NVIDIA GeForce RTX 4090 GPU, and 128GB of DDR4-3200. 

\begin{table}[ht!]
\centering
\small
\setlength{\tabcolsep}{3.55pt}
\caption{DGL benchmark dataset statistics.}
\begin{tabular}{lcccccccc}
\toprule
\textbf{Dataset} & \# \textbf{Nodes} & \# \textbf{Edges} & \# \textbf{Labels} & \# \textbf{Features} & \textbf{Average Degree} & \textbf{Label Homophily} & \textbf{Connected} \\
\midrule
\textit{actor} & 7600 & 29707 & 5 & 932 & 7.82 & \red{\textbf{0.22}} & Yes \\
\textit{chameleon} & 2277 & 31421 & 5 & 2325 & 27.60 & \red{\textbf{0.23}} & Yes \\
\textit{citeseer} & 3327 & 4676 & 6 & 3703 & 2.81 & 0.74 & No  \\
\textit{community} & 1400 & 3871 & 8 & 10 & 5.53 & 0.70 & Yes \\
\textit{computer} & 13752 & 245861 & 10 & 767 & 35.76 & 0.78 & No  \\
\textit{cora} & 2708 & 5278 & 7 & 1433 & 3.90 & 0.81 & No  \\
\textit{cornell} & 183 & 280 & 5 & 1703 & 3.06 & \red{\textbf{0.13}} & Yes \\
\textit{cycle} & 871 & 970 & 2 & 1 & 2.23 & 0.90 & Yes \\
\textit{grid} & 1231 & 1705 & 2 & 1 & 2.77 & 0.91 & Yes \\
\textit{photo} & 7650 & 119082 & 8 & 745 & 31.13 & 0.83 & No  \\
\textit{pubmed} & 19717 & 44327 & 3 & 500 & 4.50 & 0.80 & Yes \\
\textit{shape} & 700 & 1760 & 4 & 1 & 5.03 & 0.77 & Yes \\
\textit{squirrel} & 5201 & 198493 & 5 & 2089 & 76.33 & \red{\textbf{0.22}} & Yes \\
\textit{texas} & 183 & 295 & 5 & 1703 & 3.22 & \red{\textbf{0.11}} & Yes \\
\textit{wiki} & 11701 & 216123 & 10 & 300 & 36.94 & 0.66 & No  \\
\textit{wisconsin} & 251 & 466 & 5 & 1703 & 3.71 & \red{\textbf{0.21}} & Yes \\
\bottomrule
\end{tabular}
\label{tab:dataset_stats}
\end{table}

\subsection{Baseline Methods and Summary of Findings}\label{app:baselines}
\paragraph{GNN and GT baselines.} We compare with $16$ representative approaches: GCN \citep{kipf17-classifgcnn}, GraphSAGE \citep{hamilton2017inductive}, GAT \citep{velivckovic2017graph}, PPRGo \citep{bojchevski2020scaling}, GRAND+ \citep{feng2022grand+}, GT \citep{dwivedi2020generalization}, Gophormer \citep{zhao2021gophormer}, Graphormer \citep{ying2021transformers}, SAN \citep{kreuzer2021rethinking}, GraphGPS \citep{rampavsek2022recipe}, NAGphormer \citep{chennagphormer}{NAG}, Specformer \citep{bospecformer}, Exphormer \citep{shirzad2023exphormer}, SGFormer \citep{wu2023sgformer}, VCR-Graphormer \citep{fu2024vcr}{VCR}, and PolyFormer \citep{ma2024polyformer}. Across datasets, \textsc{Graph Cascades} consistently enhances node classification across various GNN and GT backbones, achieving competitive performance and largest gains on heterophilic graphs as well as homophilic graphs with moderate-to-high degrees (see Section \ref{sec:empirical}, Appendices \ref{app:numerical1_appendix}-\ref{app:correlation}, and Appendices \ref{app:longrange}-\ref{10-backbones}).

\paragraph{Rewiring Baselines.} To position \textsc{Graph Cascades} against existing graph rewiring methods, we compare cascade operators (MAS and TAS) against structural baselines targeting oversquashing (\textsc{FoSR} \citep{karhadkar2023fosr} and \textsc{SDRF} \citep{topping2022understanding}), a diffusion-based sparsification method (GDC \citep{gasteiger2019diffusion}), and a direct two-hop common-neighbor proxy ($A^2$ rewiring). By adding edges that increase algebraic connectivity (\textsc{FoSR}) or break high-curvature bottlenecks (\textsc{SDRF}), these structural methods \emph{create} shortcuts across narrow cuts to alleviate bottleneck-induced oversquashing. Meanwhile, GDC sparsifies global diffusion matrices to smooth feature propagation, and $A^2$ rewiring relies on static two-hop common-neighbor counts to reinforce local node overlap. In contrast, \textsc{Graph Cascades} is complementary rather than competitive: it \emph{reinforces} mesoscopic regions using dynamic, weighted multi-hop flux that already contains redundant short-path support, yet provably cannot bridge cuts of size smaller than $\kappa$ (Appendix~\ref{app:edge-connectivity}). These approaches therefore span complementary structural philosophies: bottleneck-targeted methods (\textsc{FoSR}, \textsc{SDRF}) excel on graphs dominated by narrow cuts, while diffusion and local overlap baselines (GDC, $A^2$)---alongside \textsc{Graph Cascades}---operate primarily on graphs whose label-relevant structure benefits from multi-hop neighborhood reinforcement. The empirical comparison in Appendix~\ref{app:rewiring} confirms this division.

%\textsc{Graph Cascades} outperforms FoSR and SDRF on 12 of 16 benchmarks, with the four losses concentrated in regimes the theory identifies as out of scope: bottleneck-dominated graphs (\textit{cycle}, \textit{grid}), where Theorem~\ref{thm:walklength} and Appendix~\ref{app:edge-connectivity} show that cascade rewiring cannot bridge narrow cuts; and the low-degree homophilic regime (\textit{citeseer}), where same-label pairs lack reinforced support. On \emph{community}, all three rewiring methods underperform the GCN baseline, while \textsc{CR-Graphormer} (78.60\%) achieves the best result in the paper, indicating an architectural rather than rewiring effect. 

\paragraph{Heterophily Baselines.} To assess how \textsc{Graph Cascades} compares against architectures explicitly designed for heterophilic graphs, we run additional experiments against seven state-of-the-art baselines: ACMGCN \citep{luan2022revisiting}, CMGNN \citep{zhengunderstanding}, GGCN \citep{yan2022two}, GloGNN \citep{li2022finding}, H2GCN \citep{zhu2020beyond}, M2MGNN \citep{liang2024sign}, and OrderedGNN \citep{song2023ordered}. These methods incorporate heterophily-specific mechanisms such as adaptive channel mixing, compatibility-matrix-aware message passing, multiset-to-multiset aggregation, and ordered hop-wise neuron blocks, and have been shown to outperform standard GNNs on heterophilic benchmarks. Overall, \textsc{Graph Cascades} remains competitive with specialized heterophily-aware architectures, achieving top-tier performance across diverse graph regimes without requiring architecture-specific modifications. Further details are provided in Section \ref{app:heterophily}.

\subsection{Main Results with Standard Deviations} \label{app:numerical1_appendix}

\begin{table}[ht!]
\centering
\tiny
\setlength{\tabcolsep}{4pt}
\caption{Test accuracy (mean $\pm$ std, in \%) across DGL datasets. ``Our Best Validation selected Method'' indicates the best validation-selected cascade models from Table~\ref{std1} per dataset. OOM indicates out-of-memory errors. Accuracies for the top-5 best-performing baselines are encoded by shading, with darker cells indicating better performance.}

\begin{tabular}{l l c c c c c c c c}
\hline
& & \textbf{\textit{actor}} & \textbf{\textit{chameleon}} & \textbf{\textit{citeseer}} & \textbf{\textit{community}} & \textbf{\textit{computer}} & \textbf{\textit{cora}} & \textbf{\textit{cornell}} & \textbf{\textit{cycle}} \\
\hline
GNNs & GCN
& 29.59$\pm$0.91 & 57.30$\pm$2.67 & \cellcolor{mycolor!54} 75.70$\pm$1.41 & 55.63$\pm$1.92 & 90.30$\pm$0.66 & \cellcolor{mycolor!72} 87.96$\pm$0.83 & 45.11$\pm$6.20 & 60.73$\pm$6.59 \\

& GraphSAGE
& 34.57$\pm$1.03 & \cellcolor{mycolor!72} 64.36$\pm$1.60 & \cellcolor{mycolor!90} 75.88$\pm$1.44 & 52.09$\pm$2.06 & \cellcolor{mycolor!18} 90.39$\pm$0.50 & \cellcolor{mycolor!90} 88.09$\pm$0.89 & \cellcolor{mycolor!18} 64.47$\pm$4.69 & 59.27$\pm$2.73 \\

& GAT
& 28.85$\pm$0.95 & 58.63$\pm$2.57 & \cellcolor{mycolor!72} 75.77$\pm$1.37 & \cellcolor{mycolor!18} 58.99$\pm$2.57 & OOM & \cellcolor{mycolor!54} 87.38$\pm$1.02 & 44.15$\pm$6.40 & 59.27$\pm$2.73\\

& PPRGo & 31.36$\pm$1.14 & 45.83$\pm$2.62 & 74.33$\pm$0.91& 39.43$\pm$2.19 & \cellcolor{mycolor!90} 91.13$\pm$0.58 & 84.30$\pm$0.73 &46.91$\pm$7.09 & 59.27$\pm$2.73 \\

& GRAND+ & 30.20$\pm$1.08 & 45.54$\pm$2.73 & 73.80$\pm$0.60 & 39.37$\pm$2.01 & \cellcolor{mycolor!54}90.85$\pm$0.69 & 83.70$\pm$0.46 & 45.11$\pm$6.46 & 59.27$\pm$2.73 \\
\hline
GTs & GT
& 34.19$\pm$1.16 & \cellcolor{mycolor!90} 64.83$\pm$1.56 & 68.60$\pm$5.85 & 54.00$\pm$3.45 & OOM & 85.41$\pm$0.81 & 61.91$\pm$8.02 & 56.32$\pm$7.47 \\

& Gophormer
& 34.90$\pm$1.40 & 53.16$\pm$3.32 & 31.48$\pm$9.62 & \cellcolor{mycolor!54} 71.00$\pm$4.17 & \cellcolor{mycolor!36} 90.60$\pm$0.68 & 49.18$\pm$12.26 & \cellcolor{mycolor!36} 67.13$\pm$6.92 & 54.98$\pm$9.41 \\

& Graphormer
& \cellcolor{mycolor!18} 35.01$\pm$1.09 & 50.70$\pm$4.40 & 33.18$\pm$2.97 & \cellcolor{mycolor!72} 73.20$\pm$2.30 & OOM & 35.89$\pm$5.59 & 51.49$\pm$9.60 & \cellcolor{mycolor!90} 81.48$\pm$6.13 \\

& SAN
& \cellcolor{mycolor!54} 36.04$\pm$0.91 & 59.22$\pm$2.00 & 73.37$\pm$2.73 & 46.00$\pm$2.39 & OOM & 85.07$\pm$1.47 & \cellcolor{mycolor!54} 67.55$\pm$5.69 & 59.27$\pm$2.73 \\

& GraphGPS
& \cellcolor{mycolor!36} 35.19$\pm$1.31 & \cellcolor{mycolor!36} 60.78$\pm$1.35 & \cellcolor{mycolor!18} 75.23$\pm$1.63 & 44.00$\pm$2.30 & OOM & 86.65$\pm$1.04 & 61.38$\pm$5.49 & 55.11$\pm$10.10 \\

& NAG & 31.68$\pm$1.26 & 40.38$\pm$1.81 & 74.72$\pm$1.67 & 46.17$\pm$2.49 & 88.66$\pm$0.73 & \cellcolor{mycolor!18} 86.75$\pm$0.82 & 58.94$\pm$7.11 & \cellcolor{mycolor!72} 77.24$\pm$9.34 \\
& Specformer & 30.74$\pm$1.10 & 42.29$\pm$2.45 & 70.33$\pm$1.78 & 47.69$\pm$3.84 & 88.05$\pm$0.39 & 82.49$\pm$1.08 & 54.00$\pm$3.90 & 54.32$\pm$1.30 \\

& Exphormer
& \cellcolor{mycolor!90} 36.62$\pm$0.89 & \cellcolor{mycolor!18} 59.36$\pm$3.10 & 68.34$\pm$2.11 & \cellcolor{mycolor!36} 65.63$\pm$2.75 & 89.22$\pm$0.80 & 78.17$\pm$2.40 & \cellcolor{mycolor!72} 68.19$\pm$6.75 & \cellcolor{mycolor!54} 61.80$\pm$6.57 \\
& SGFormer & 31.06$\pm$1.33 & 42.37$\pm$3.90 & 69.91$\pm$1.11 & 44.61$\pm$2.76 & 85.40$\pm$0.23 & 81.41$\pm$2.80 & 55.19$\pm$2.43 & 55.19$\pm$1.44 \\
& VCR
& 26.41$\pm$1.20 & 27.28$\pm$2.51 & 71.15$\pm$1.36 & 28.13$\pm$4.60 & 82.45$\pm$5.30 & 85.41$\pm$1.33 & 44.89$\pm$6.09 & \cellcolor{mycolor!36} 61.64$\pm$4.12 \\
& PolyFormer & 32.12$\pm$1.12 & 46.83$\pm$2.61 & 72.13$\pm$1.22 & 48.91$\pm$1.77 & 89.60$\pm$0.31 & 83.51$\pm$2.17 & 59.02$\pm$3.67 & 58.11$\pm$2.91 \\

\hline
\textbf{Our} & \textbf{Best}
& GPS-TAS & CR-TAS & GPS-TAS & CR-TAS & GCN-TAS & GCN-TAS & NAG-MAS/TAS & CR-TAS \\
\textbf{Validated} & \textbf{Method}
& \cellcolor{mycolor!72} 36.46$\pm$1.31 & \cellcolor{mycolor!54} 63.65$\pm$2.82 & \cellcolor{mycolor!36} 75.28$\pm$1.31 & \cellcolor{mycolor!90} 78.60$\pm$3.50 & \cellcolor{mycolor!72} 90.87$\pm$0.79 & \cellcolor{mycolor!36} 87.29$\pm$0.77 & \cellcolor{mycolor!90} 74.36$\pm$5.96 & \cellcolor{mycolor!18} 61.32$\pm$3.67 \\

\hline

\end{tabular}

\setlength{\tabcolsep}{3.79pt}
\begin{tabular}{l l c c c c c c c c}
\hline
& & \textbf{\textit{grid}} & \textbf{\textit{photo}} & \textbf{\textit{pubmed}} & \textbf{\textit{shape}} & \textbf{\textit{squirrel}} & \textbf{\textit{texas}} & \textbf{\textit{wiki}} & \textbf{\textit{wisconsin}} \\
\hline
GNNs & GCN
& \cellcolor{mycolor!18} 59.34$\pm$5.45 & 93.90$\pm$0.74 & \cellcolor{mycolor!36} 87.16$\pm$0.35 & \cellcolor{mycolor!72} 89.46$\pm$11.87 & 39.10$\pm$1.43 & 54.57$\pm$7.12 & \cellcolor{mycolor!54}83.89$\pm$0.73 & 51.41$\pm$4.81 \\

& GraphSAGE
& 58.19$\pm$1.76 & \cellcolor{mycolor!90} 95.31$\pm$0.60 & \cellcolor{mycolor!72} 87.33$\pm$0.34 & 43.11$\pm$2.97 & \cellcolor{mycolor!72} 45.41$\pm$1.19 & \cellcolor{mycolor!36} 79.79$\pm$6.19 & \cellcolor{mycolor!90} 84.57$\pm$0.78 & 72.89$\pm$4.63 \\

& GAT
& 58.19$\pm$1.76 & 94.49$\pm$0.52 & 86.88$\pm$0.45 & 43.11$\pm$2.97 & OOM & 56.70$\pm$6.45 & OOM & 49.45$\pm$5.54 \\

& PPRGo & 58.19$\pm$1.76 & \cellcolor{mycolor!72} 95.09$\pm$0.63 & 86.72$\pm$0.39 & 43.11$\pm$2.97 & 30.22$\pm$ 1.56 & 58.51$\pm$5.46 & \cellcolor{mycolor!72} 84.16$\pm$0.98 & 55.63$\pm$4.90 \\

& GRAND+ & 58.19$\pm$1.76 & 94.82$\pm$0.59 & 85.95$\pm$0.43 & 43.11$\pm$2.97 & 30.23$\pm$1.54 & 56.49$\pm$6.34 & 83.38$\pm$0.89 & 53.20$\pm$5.33\\
\hline
GTs & GT
& 54.64$\pm$7.14 & OOM & 86.59$\pm$0.79 & 23.23$\pm$2.73 & OOM & 77.77$\pm$5.05 & OOM & \cellcolor{mycolor!18} 73.52$\pm$4.49 \\

& Gophormer
& \cellcolor{mycolor!54} 62.62$\pm$8.81 & \cellcolor{mycolor!54} 94.98$\pm$0.64 & 40.40$\pm$13.63 & \cellcolor{mycolor!36} 79.83$\pm$8.97 & 37.44$\pm$2.80 & \cellcolor{mycolor!72} 81.49$\pm$5.82 & 83.34$\pm$0.91 & \cellcolor{mycolor!90} 81.95$\pm$5.08 \\

& Graphormer
& \cellcolor{mycolor!90} 72.14$\pm$2.65 & 89.02$\pm$0.75 & OOM & \cellcolor{mycolor!54} 85.57$\pm$2.15 & \cellcolor{mycolor!36} 40.55$\pm$2.05 & 68.62$\pm$10.28 & OOM & 71.33$\pm$8.80 \\

& SAN
& 57.27$\pm$4.26 &\cellcolor{mycolor!36} 94.85$\pm$1.24 & OOM & 29.40$\pm$9.52 & 39.00$\pm$1.38 & \cellcolor{mycolor!54} 80.00$\pm$7.48 & OOM & \cellcolor{mycolor!54} 80.63$\pm$3.91 \\

& GraphGPS
& 57.41$\pm$10.00 & 94.64$\pm$0.63 & OOM & 46.74$\pm$14.93 & \cellcolor{mycolor!18} 40.04$\pm$1.27 & 69.68$\pm$6.58 & OOM & 72.66$\pm$4.95 \\

& NAG & \cellcolor{mycolor!72} 70.99$\pm$1.99 & 93.94$\pm$0.67 & \cellcolor{mycolor!54} 87.26$\pm$0.52 & 49.89$\pm$6.51 & 25.80$\pm$1.26 & 63.30$\pm$6.43 & 82.40$\pm$0.98 & 61.25$\pm$4.32 \\
& Specformer
& 53.77$\pm$1.19 & 89.22$\pm$0.76 & 85.58$\pm$0.48 & 45.77$\pm$6.69 & 31.44$\pm$2.30 & 63.26$\pm$5.32 & 71.82$\pm$0.31 & 68.90$\pm$5.01 \\

& Exphormer
& 57.31$\pm$4.95 & 94.25$\pm$0.65 & \cellcolor{mycolor!18} 86.97$\pm$0.61 & \cellcolor{mycolor!18} 76.31$\pm$3.44 & \cellcolor{mycolor!54} 42.31$\pm$3.37 & \cellcolor{mycolor!18} 79.47$\pm$6.95 & \cellcolor{mycolor!36}  83.84$\pm$0.86 & \cellcolor{mycolor!36} 80.39$\pm$5.56 \\
& SGFormer
& 52.53$\pm$3.77 & 87.93$\pm$1.14 & 85.07$\pm$0.81 & 42.80$\pm$2.10 & 32.86$\pm$3.29 & 61.61$\pm$5.01 & 67.03$\pm$0.39 & 63.86$\pm$4.38 \\
& VCR &
58.32$\pm$1.96 & 88.45$\pm$7.85 & 83.25$\pm$0.44 & 44.66$\pm$6.15 & 22.48$\pm$2.00 & 54.68$\pm$7.28 & 78.46$\pm$2.69 & 51.33$\pm$7.46 \\

& PolyFormer
& 56.15$\pm$4.23 & 91.93$\pm$0.21 & 86.82$\pm$0.17 & 47.29$\pm$8.02 & 34.42$\pm$4.45 & 64.38$\pm$4.40 & 73.18$\pm$0.99 & 69.77$\pm$3.73 \\

\hline
\textbf{Our} & \textbf{Best}
& NAG-MAS & GPS-TAS & VCR-TAS & GCN-TAS & CR-TAS & VCR-TAS & GCN-TAS & NAG-MAS/TAS \\

\textbf{Validated} & \textbf{Method}
& \cellcolor{mycolor!36} 62.25$\pm$2.21 & \cellcolor{mycolor!18} 94.84$\pm$0.82 & \cellcolor{mycolor!90} 88.70$\pm$0.91 & \cellcolor{mycolor!90} 89.57$\pm$11.79 & \cellcolor{mycolor!90} 47.31$\pm$2.19 & \cellcolor{mycolor!90} 82.31$\pm$6.59 & \cellcolor{mycolor!18} 83.54$\pm$0.76 & \cellcolor{mycolor!72} 81.02$\pm$4.13 \\

\hline

\end{tabular}
\label{std2}
\end{table}

\begin{table}[ht!]
\centering
\tiny
\setlength{\tabcolsep}{6.5pt}
\caption{Test accuracy (mean $\pm$ std, in \%) across DGL datasets. OOM indicates out-of-memory errors. Darker cell colors indicate better performance.}
\begin{tabular}{l c c c c c c c c}
\hline
\textbf{Method} & \textbf{\textit{actor}} & \textbf{\textit{chameleon}} & \textbf{\textit{citeseer}} & \textbf{\textit{community}} & \textbf{\textit{computer}} & \textbf{\textit{cora}} & \textbf{\textit{cornell}} & \textbf{\textit{cycle}} \\
\hline
GCN & 29.59$\pm$0.91 & 57.30$\pm$2.67 & \cellcolor{myred!60} 75.70$\pm$1.41 & \cellcolor{myred!60} 55.63$\pm$1.92 & \cellcolor{myred!30} 90.30$\pm$0.66 & \cellcolor{myred!60} 87.96$\pm$0.83 & 45.11$\pm$6.20 & \cellcolor{myred!60} 60.73$\pm$6.59 \\
GCN-MAS & \cellcolor{myred!30} 34.28$\pm$0.96 & \cellcolor{myred!30} 57.67$\pm$2.07 & 74.71$\pm$1.61 & \cellcolor{myred!30} 51.76$\pm$2.06 & 89.34$\pm$0.84 & 84.93$\pm$1.12 & \cellcolor{myred!60} 60.85$\pm$5.00 & 59.27$\pm$2.73 \\
GCN-TAS & \cellcolor{myred!60} 34.62$\pm$0.92 & \cellcolor{myred!60} 58.75$\pm$2.54 & \cellcolor{myred!30} 75.46$\pm$1.46 & 50.11$\pm$2.36 & \cellcolor{myred!60} 90.87$\pm$0.79 & \cellcolor{myred!30} 87.29$\pm$0.77 & \cellcolor{myred!30} 57.02$\pm$5.37 & \cellcolor{myred!30} 60.39$\pm$3.76 \\
\hline
GraphGPS & 35.19$\pm$1.31 & \cellcolor{mygreen!60} 60.78$\pm$1.35 & \cellcolor{mygreen!30}75.23$\pm$1.63 & 44.00$\pm$2.30 & OOM & \cellcolor{mygreen!60} 86.65$\pm$1.04 & 61.38$\pm$5.49 & 55.11$\pm$10.10 \\
GPS-MAS & \cellcolor{mygreen!30} 36.39$\pm$1.17 & \cellcolor{mygreen!30} 59.90$\pm$1.75 & 74.09$\pm$1.38 & \cellcolor{mygreen!60} 47.43$\pm$3.04 & OOM & 85.49$\pm$1.46 & \cellcolor{mygreen!60} 68.72$\pm$6.02 & \cellcolor{mygreen!30} 55.78$\pm$2.70 \\
GPS-TAS & \cellcolor{mygreen!60} 36.46$\pm$1.31 & 59.66$\pm$1.66 & \cellcolor{mygreen!60} 75.28$\pm$1.31 & \cellcolor{mygreen!30} 45.26$\pm$2.95 & OOM & \cellcolor{mygreen!30} 86.53$\pm$1.08 & \cellcolor{mygreen!30} 67.23$\pm$6.49 & \cellcolor{mygreen!60} 57.51$\pm$5.41 \\
\hline
NAG & 31.68$\pm$1.26 & 40.38$\pm$1.81 & \cellcolor{mycyan!60} 74.72$\pm$1.67 & 46.17$\pm$2.49 & 88.66$\pm$0.73 & \cellcolor{mycyan!30} 86.75$\pm$0.82 & 58.94$\pm$7.11 & \cellcolor{mycyan!60} 77.24$\pm$9.34 \\
NAG-MAS & \cellcolor{mycyan!60} 35.10$\pm$1.03 & \cellcolor{mycyan!60} 59.15$\pm$1.79 & 72.17$\pm$1.61 & \cellcolor{mycyan!30} 52.59$\pm$2.21 & \cellcolor{mycyan!30} 89.10$\pm$0.59 & 84.84$\pm$0.74 & \cellcolor{mycyan!60} 74.36$\pm$5.96 & 59.00$\pm$2.86 \\
NAG-TAS & \cellcolor{mycyan!60} 35.10$\pm$1.03 & \cellcolor{mycyan!30} 58.41$\pm$2.26 & \cellcolor{mycyan!30} 74.15$\pm$1.63 & \cellcolor{mycyan!60} 58.21$\pm$2.86 & \cellcolor{mycyan!60} 89.19$\pm$0.69 & \cellcolor{mycyan!60} 86.82$\pm$1.00 & \cellcolor{mycyan!60} 74.36$\pm$5.96 & \cellcolor{mycyan!30} 59.20$\pm$2.19 \\
\hline
VCR & 26.41$\pm$1.20 & 27.28$\pm$2.51 & 71.15$\pm$1.36 & 28.13$\pm$4.60 & 82.45$\pm$5.30 & \cellcolor{mypurple!60} 85.41$\pm$1.33 & 44.89$\pm$6.09 & \cellcolor{mypurple!60} 61.64$\pm$4.12 \\
VCR-MAS & \cellcolor{mypurple!30} 36.18$\pm$1.23 & \cellcolor{mypurple!60}  60.02$\pm$2.87 & \cellcolor{mypurple!30} 71.47$\pm$1.64 & \cellcolor{mypurple!30}47.22$\pm$2.98 & \cellcolor{mypurple!60}89.43$\pm$2.43& 85.08$\pm$0.57 & \cellcolor{mypurple!30} 73.85$\pm$5.72 & \cellcolor{mypurple!30} 51.70$\pm$3.31 \\
VCR-TAS & \cellcolor{mypurple!60} 36.28$\pm$1.42 & \cellcolor{mypurple!30} 59.61$\pm$1.95 & \cellcolor{mypurple!60} 74.14$\pm$1.53 & \cellcolor{mypurple!60} 57.97$\pm$1.81 & \cellcolor{mypurple!30} 88.95$\pm$2.39 & \cellcolor{mypurple!30} 85.11$\pm$0.62 & \cellcolor{mypurple!60} 75.02$\pm$4.68 & 50.09$\pm$4.37 \\
\hline
CR-MAS & 31.52$\pm$1.47 & 62.82$\pm$2.06 & 50.63$\pm$2.29 & 54.04$\pm$2.17 & \cellcolor{myblack!45} 89.44$\pm$0.69 & 79.81$\pm$2.01 & 54.04$\pm$7.38 & 59.18$\pm$2.55 \\
CR-TAS & \cellcolor{myblack!45} 34.64$\pm$1.21 & \cellcolor{myblack!45} 63.65$\pm$2.82 & \cellcolor{myblack!45} 51.84$\pm$4.17 & \cellcolor{myblack!45} 78.60$\pm$3.50 & 89.33$\pm$0.82 & \cellcolor{myblack!45} 81.02$\pm$2.05 & \cellcolor{myblack!45} 55.74$\pm$6.85 & \cellcolor{myblack!45} 61.32$\pm$3.67 \\
\hline
\end{tabular}

\setlength{\tabcolsep}{6.25pt}
\begin{tabular}{l c c c c c c c c}
\hline
\textbf{Method} & \textbf{\textit{grid}} & \textbf{\textit{photo}} & \textbf{\textit{pubmed}} & \textbf{\textit{shape}} & \textbf{\textit{squirrel}} & \textbf{\textit{texas}} & \textbf{\textit{wiki}} & \textbf{\textit{wisconsin}} \\
\hline
GCN & \cellcolor{myred!60} 59.34$\pm$5.45 & 93.90$\pm$0.74 & 87.16$\pm$0.35 & \cellcolor{myred!30} 89.46$\pm$11.87 & 39.10$\pm$1.43 & 54.57$\pm$7.12 & \cellcolor{myred!60} 83.89$\pm$0.73 & 51.41$\pm$4.81 \\
GCN-MAS & \cellcolor{myred!30} 58.30$\pm$1.84 & \cellcolor{myred!30} 93.91$\pm$0.63 & \cellcolor{myred!30} 87.87$\pm$0.54 & 57.69$\pm$5.49 & \cellcolor{myred!30} 40.01$\pm$1.96 & \cellcolor{myred!60} 66.49$\pm$7.71 & 81.79$\pm$0.66 & \cellcolor{myred!30} 68.52$\pm$4.82 \\
GCN-TAS & 58.19$\pm$1.76 & \cellcolor{myred!60} 94.17$\pm$0.60 & \cellcolor{myred!60} 88.15$\pm$0.45 & \cellcolor{myred!60} 89.57$\pm$11.79 & \cellcolor{myred!60} 42.38$\pm$1.68 & \cellcolor{myred!30} 65.00$\pm$6.42 & \cellcolor{myred!30} 83.54$\pm$0.76 & \cellcolor{myred!60} 71.56$\pm$4.41 \\
\hline
GraphGPS & \cellcolor{mygreen!60} 57.41$\pm$10.00 & 94.64$\pm$0.63 & OOM & 46.74$\pm$14.93 & \cellcolor{mygreen!60} 40.04$\pm$1.27 & 69.68$\pm$6.58 & OOM & 72.66$\pm$4.95 \\
GPS-MAS & 55.44$\pm$6.50 & \cellcolor{mygreen!30} 94.71$\pm$0.78 & OOM & \cellcolor{mygreen!30} 54.89$\pm$6.41 & \cellcolor{mygreen!30} 39.20$\pm$1.17 & \cellcolor{mygreen!30} 75.00$\pm$5.56 & OOM & \cellcolor{mygreen!60} 78.44$\pm$3.68 \\
GPS-TAS & \cellcolor{mygreen!30} 56.33$\pm$5.59 & \cellcolor{mygreen!60} 94.84$\pm$0.82 & OOM & \cellcolor{mygreen!60} 56.26$\pm$13.41 & 39.18$\pm$1.34 & \cellcolor{mygreen!60} 75.64$\pm$4.71 & OOM & \cellcolor{mygreen!30} 78.28$\pm$3.32 \\
\hline
NAG & \cellcolor{mycyan!60} 70.99$\pm$1.99 & 93.94$\pm$0.67 & 87.26$\pm$0.52 & 49.89$\pm$6.51 & 25.80$\pm$1.26 & 63.30$\pm$6.43 & \cellcolor{mycyan!30} 82.40$\pm$0.98 & 61.25$\pm$4.32 \\
NAG-MAS & \cellcolor{mycyan!30} 62.25$\pm$2.21 & \cellcolor{mycyan!30} 94.27$\pm$0.69 & \cellcolor{mycyan!30} 88.19$\pm$0.45 & \cellcolor{mycyan!30} 57.54$\pm$4.31 & \cellcolor{mycyan!30} 40.89$\pm$1.60 & \cellcolor{mycyan!60} 81.28$\pm$6.16 & 82.38$\pm$1.19 & \cellcolor{mycyan!60} 81.02$\pm$4.13 \\
NAG-TAS & 60.60$\pm$2.44 & \cellcolor{mycyan!60} 94.61$\pm$0.51 & \cellcolor{mycyan!60} 88.26$\pm$0.37 & \cellcolor{mycyan!60} 84.43$\pm$4.05 & \cellcolor{mycyan!60} 42.13$\pm$1.41 & \cellcolor{mycyan!60} 81.28$\pm$6.16 & \cellcolor{mycyan!60} 82.59$\pm$1.20 & \cellcolor{mycyan!60} 81.02$\pm$4.13 \\
\hline
VCR & \cellcolor{mypurple!60} 58.32$\pm$1.96 & 88.45$\pm$7.85 & 83.25$\pm$0.44 & 44.66$\pm$6.15 & 22.48$\pm$2.00 & 54.68$\pm$7.28 & 78.46$\pm$2.69 & 51.33$\pm$7.46 \\
VCR-MAS & 44.05$\pm$2.58& \cellcolor{mypurple!30} 93.88$\pm$0.29& \cellcolor{mypurple!30}87.59$\pm$0.46& \cellcolor{mypurple!30}52.12$\pm$5.17& \cellcolor{mypurple!30} 39.37$\pm$1.49 & \cellcolor{mypurple!30} 80.99$\pm$5.63 & \cellcolor{mypurple!30} 82.68$\pm$0.54 & \cellcolor{mypurple!30} 79.93$\pm$3.76 \\
VCR-TAS & \cellcolor{mypurple!30} 47.84$\pm$1.28 & \cellcolor{mypurple!60} 94.22$\pm$1.27& \cellcolor{mypurple!60}88.70$\pm$0.91& \cellcolor{mypurple!60}76.94$\pm$3.73 & \cellcolor{mypurple!60} 46.06$\pm$1.04 & \cellcolor{mypurple!60} 82.31$\pm$6.59 & \cellcolor{mypurple!60}  82.84$\pm$0.52 & \cellcolor{mypurple!60} 81.00$\pm$2.66\\
\hline
CR-MAS & \cellcolor{myblack!45} 61.80$\pm$2.51 & 93.73$\pm$0.89 & 87.53$\pm$0.63 & 54.37$\pm$5.11 & 45.72$\pm$2.69 & \cellcolor{myblack!45} 65.43$\pm$9.02 & 82.65$\pm$0.90 & 66.56$\pm$5.20 \\
CR-TAS & 60.74$\pm$3.46 & \cellcolor{myblack!45} 93.95$\pm$0.52 & \cellcolor{myblack!45} 87.75$\pm$0.56 & \cellcolor{myblack!45} 85.94$\pm$7.79 & \cellcolor{myblack!45} 47.31$\pm$2.19 & \cellcolor{myblack!45} 65.43$\pm$6.79 & \cellcolor{myblack!45} 83.04$\pm$0.83 & \cellcolor{myblack!45} 70.39$\pm$5.33 \\
\hline
\end{tabular}
\label{std1}
\end{table}

% \begin{table}[ht!]
% \centering
% \small
% \setlength{\tabcolsep}{2pt}
% \caption{DGL benchmark dataset statistics (extended).}
% \begin{tabular}{lccccccc}
% \toprule
% Dataset & \# Nodes & \# Edges & \# Labels & \# Features & Average Degree & Label Homophily & Connected \\
% \midrule
% arxiv & 169343 & 1166243 & 5 & 128 & 13.67 & \red{\textbf{0.26}} & Yes \\
% Roman-empire & 22662 & 65854 & 18 & 300 & 2.91 & \red{\textbf{0.05}} & Yes \\
% Amazon-ratings & 24492 & 186100 & 5 & 300 & 7.60 & 0.38 & Yes \\
% Minesweeper & 10000 & 78804 & 2 & 7 & 7.88 & 0.68 & Yes \\
% Questions & 48921 & 307080 & 2 & 301 & 6.28 & 0.90 & Yes \\
% PascalVOC-SP & 4827 & 27318 & 9 & 14 & 5.66 & 0.91 & No \\
% COCO-SP & 4788 & 27026 & 16 & 14 & 5.64 & 0.91 & No \\
% chameleon\_filtered & 890 & 8854 & 5 & 2325 & 19.90 & \red{\textbf{0.23}} & Yes \\
% squirrel\_filtered & 2223 & 46998 & 5 & 2089 & 42.28 & \red{\textbf{0.19}} & Yes \\
% \bottomrule
% \end{tabular}
% \label{tab:new_dataset_stats}
% \end{table}

\begin{figure*}[!ht]
\centering
\includegraphics[height=4.8cm]{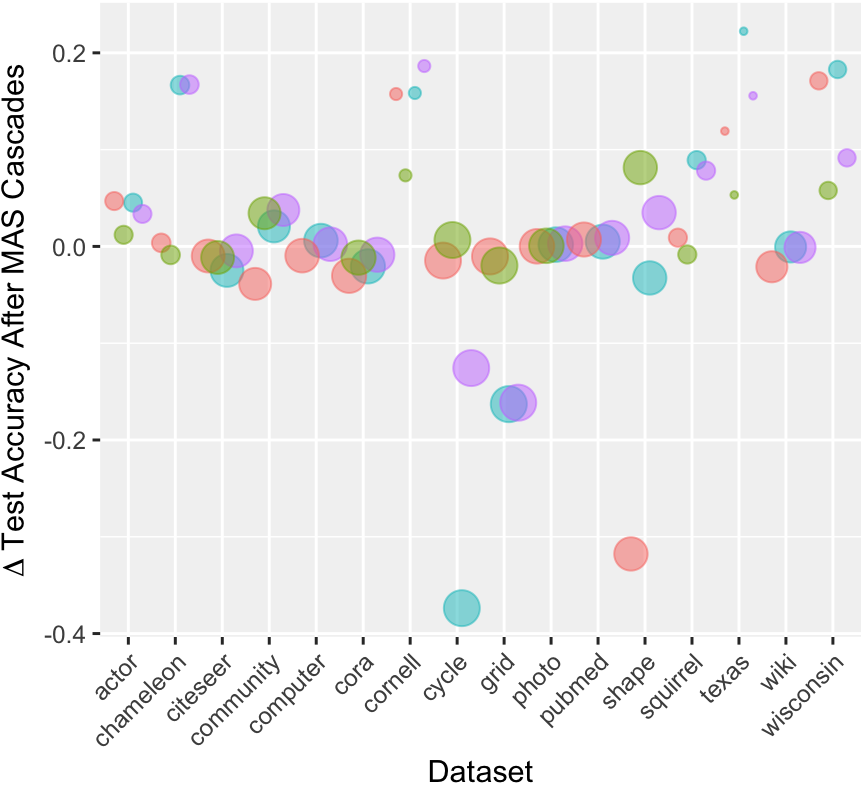}
\hspace{2pt}
\includegraphics[height=4.8cm]{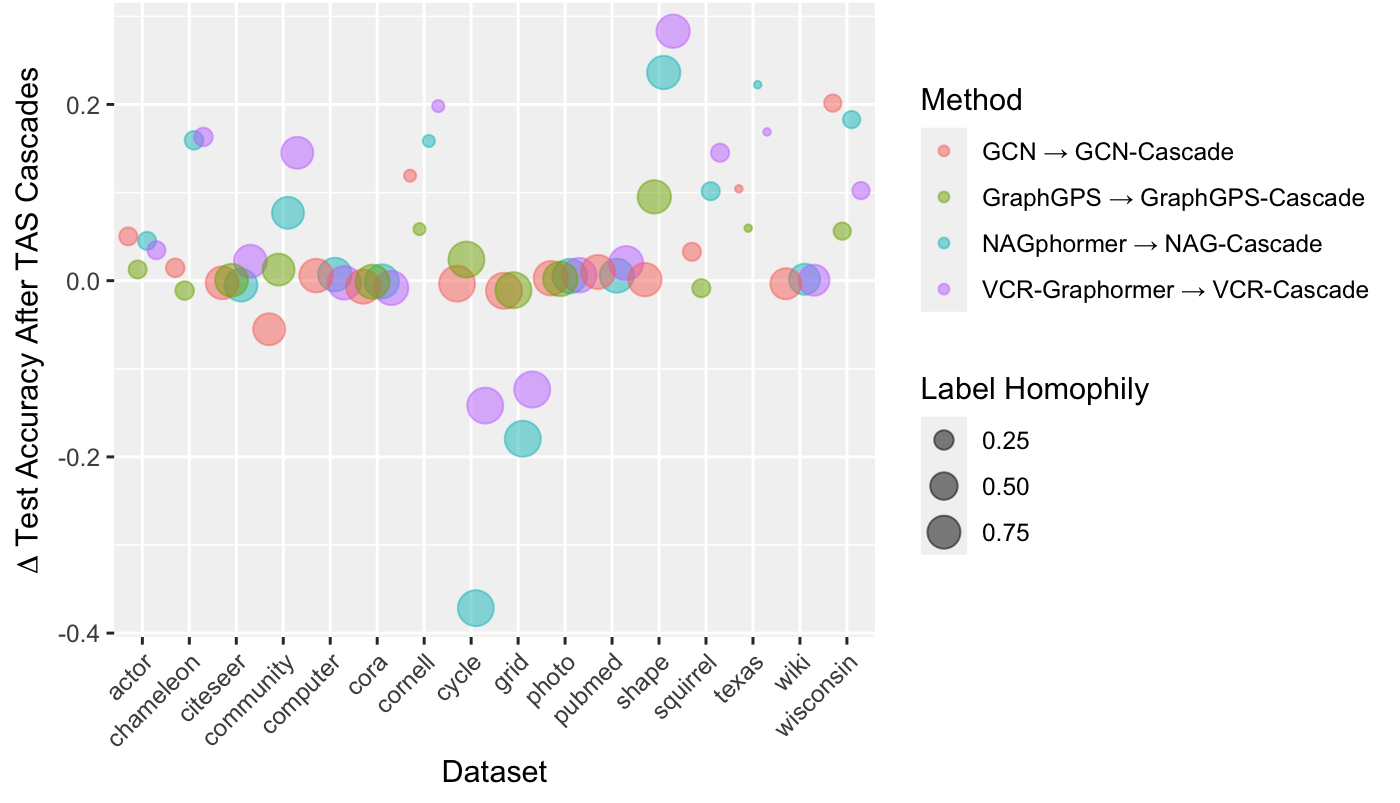}\\
\caption{Changes in mean test accuracy on DGL graphs as a result of MAS and TAS cascades.}\label{fig:summary}
\end{figure*}

\subsection{Effect of Cascades on Homophily}\label{exp1} 

\paragraph{Main Observations.} Figure~\ref{fig:homo} shows that, after rewiring, low-degree homophilic graphs become markedly less homophilic, whereas moderate- to high-degree homophilic and heterophilic graphs become mildly more homophilic. Benchmark graphs exhibit the same trend as the synthesized ones, consistent with Proposition~\ref{prop:bayes-monotone-matrix}. 

\begin{figure}[ht!]
    \centering    % Left column: two stacked figures
    \begin{minipage}{0.83\linewidth}
        \centering
        \includegraphics[width=0.49\linewidth]{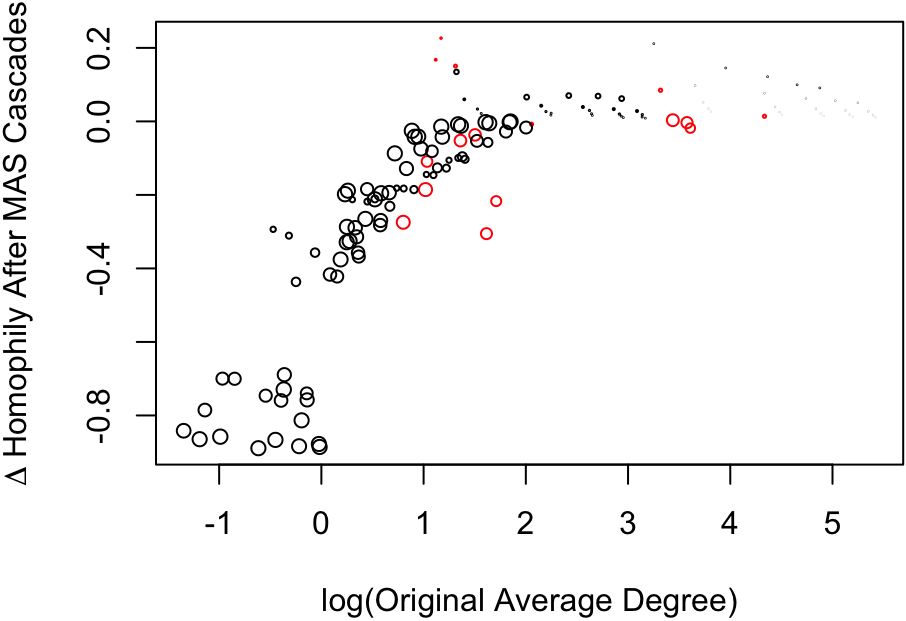}
        \includegraphics[width=0.49\linewidth]{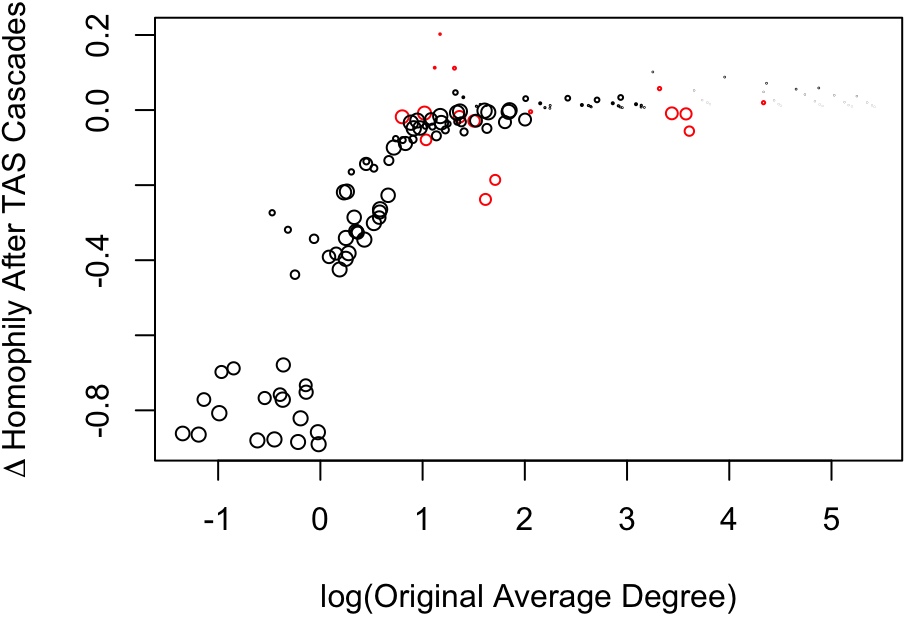}
    \end{minipage}
    \hfill
    \begin{minipage}{0.16\linewidth}
        \centering
        \includegraphics[width=\linewidth]{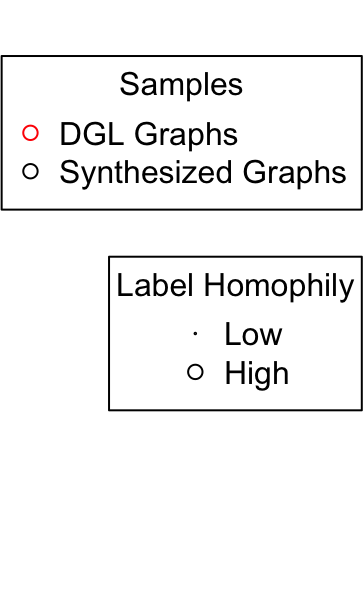}
    \end{minipage}
    \caption{Effects of mesoscopic cascades on label homophily.}
    \label{fig:homo}
\end{figure}

\paragraph{Empirical Verification of Proposition~\ref{prop:bayes-monotone-matrix}}\label{app:paragraph_empirical}
Table~\ref{tab:reinforcement-homophily-full} reports the resulting homophily quantities \(h_G\) and
\(h_{\mathcal F}\), which are the conditional probabilities most directly relevant to the
classification effect, computed using the population reinforcement event $\mathcal F$ from
Proposition~\ref{prop:bayes-monotone-matrix} with $\ell \in \{2,4,6\}$ and
$\kappa = 2$. The table organizes datasets by which assumptions of
Proposition~\ref{prop:bayes-monotone-matrix} are expected to hold, and pairs
the homophily diagnostic with the maximum within-architecture accuracy gain
from cascade rewiring ($\Delta_{\mathrm{acc}}$). Three regimes emerge.

On heterophilic graphs (\textit{texas}, \textit{cornell}, \textit{wisconsin}, \textit{chameleon}, \textit{squirrel},
\textit{actor}), both Assumptions~\ref{as2} and~\ref{as3} hold, $h_\mathcal F$ exceeds $h_G$ by
a factor of $1.5$--$4.2$, and rewiring yields large accuracy gains ($+9.87$
to $+32.74$). On moderate- to high-degree homophilic graphs (\textit{community}, \textit{shape}, \textit{computer}, \textit{photo}, \textit{wiki}, \textit{pubmed}), Assumption~\ref{as3} fails because
reinforcement is no longer selective relative to direct adjacency;
consequently, $h_\mathcal F$ drops slightly below $h_G$, and the mechanism predicts no
homophily-driven gain. We still observe positive $\Delta_{\mathrm{acc}}$
here, indicating that on these graphs the rewiring contributes through a
complementary mechanism (mesoscopic structural reorganization,
c.f.~Theorem~\ref{thm:walklength}) rather than through homophily
amplification.
On low-degree highly homophilic graphs (\textit{citeseer}, \textit{cora}, \textit{cycle}, \textit{grid}), Assumption~\ref{as2} fails: same-label pairs are not
connected by enough short walks to be reinforced, so $h_\mathcal F$ falls
substantially below $h_G$ and rewiring yields no benefit. The failure is
most pronounced on the highly regular \textit{cycle} and \textit{grid} graphs, where the
sparse and uniform structure prevents reinforcement from concentrating on
informative neighbors. On \textit{grid}, all the baselines (see Table~\ref{std2}) degrade under cascade rewiring, with within-architecture losses
ranging from $-1.04$ to $-10.48$, so $\Delta_{\mathrm{acc}} \le 0$ across
the board. On \textit{cycle}, three of the four backbones lose substantially
($-0.34$, $-18.04$, $-11.55$); the single positive value
($\Delta_{\mathrm{acc}} = +2.40$, GraphGPS) is therefore an outlier rather
than a representative gain, so we mention $\Delta_{\mathrm{acc}} \le 0$. The diagnostic therefore calibrates the \emph{direction} of the rewiring effect across all sixteen benchmarks consistently with Proposition~\ref{prop:bayes-monotone-matrix}: gains where both assumptions hold, neutrality where Assumption~\ref{as3} fails, and degradation where Assumption~\ref{as2} fails.

\begin{table}[ht!]
\centering
\caption{Empirical verification of the reinforcement homophily mechanism (Proposition~\ref{prop:bayes-monotone-matrix})
across DGL benchmarks. Rows are grouped by direct evaluation of the alignment and selectivity inequalities \(\Pr{\mathcal F\mid S}\ge \Pr{\mathcal E\mid S}\) and \(\Pr{\mathcal F}\le \Pr{\mathcal E}\); the raw probabilities
are omitted for readability. $h_G = \Pr{y_u = y_v \mid (u,v) \in E}$ is the original-graph edge homophily; $h_\mathcal F = \Pr{y_u = y_v \mid (A^2 + \cdots A^{\ell})_{uv} \ge \kappa}$ is the population reinforcement homophily, computed by direct enumeration with $\ell \in \{2,4,6\}$, $\kappa = 2$. $\Delta_{\mathrm{acc}}$ reports the maximum
within-architecture gain from cascade rewiring across baselines, computed from Table~\ref{std2}.}
\label{tab:reinforcement-homophily-full}
\small
\setlength{\tabcolsep}{13pt}
\begin{tabular}{llccccc}
\toprule
& \textbf{Dataset} & $h_G$ & $h_\mathcal F$ & $h_\mathcal F - h_G$ & $h_\mathcal F / h_G$ & $\Delta_{\mathrm{acc}}$ \\
\midrule
\multicolumn{7}{l}{\textit{\textbf{Heterophilic --- both Assumptions~\ref{as2} and~\ref{as3} hold $\Rightarrow$ Proposition~\ref{prop:bayes-monotone-matrix} applies}}} \\
& \textit{cornell}    & 0.13 & 0.55        & $+0.42$ & $4.2\times$ & $+30.13$ \\
& \textit{texas}      & 0.11 & 0.42        & $+0.31$ & $3.8\times$ & $+27.63$ \\
& \textit{wisconsin}  & 0.21 & 0.48        & $+0.27$ & $2.3\times$ & $+29.67$ \\
& \textit{chameleon}  & 0.23 & 0.46        & $+0.23$ & $2.0\times$ & $+32.74$ \\
& \textit{squirrel}   & 0.22 & 0.33        & $+0.12$ & $1.6\times$ & $+23.58$ \\
& \textit{actor}      & 0.22 & $0.32$ & $+0.10$ & $1.5\times$ & $+9.87$  \\
\midrule
\multicolumn{7}{l}{\textit{\textbf{Mid-Degree Homophilic --- Assumption~\ref{as3} fails mildly; $h_F \lesssim h_G$}}} \\
& \textit{community}  & 0.70 & 0.65        & $-0.05$ & $0.93\times$ & $+29.84$ \\
& \textit{shape}      & 0.77 & $0.70$ & $-0.07$ & $0.91\times$ & $+34.54$ \\
& \textit{pubmed}     & 0.80 & $ 0.70$ & $-0.10$ & $0.88\times$ & $+5.45$ \\
\midrule
\multicolumn{7}{l}{\textit{\textbf{High-Degree Homophilic --- Assumption~\ref{as3} fails moderately due to noise dilution}}} \\
& \textit{photo}      & 0.83 & $ 0.70$ & $-0.13$ & $0.84\times$ & $+5.77$ \\
& \textit{computer}   & 0.78 & $0.65$ & $-0.13$ & $0.83\times$ & $+6.98$ \\
& \textit{wiki}       & 0.66 & $ 0.55$ & $-0.11$ & $0.83\times$ & $+4.38$ \\
\midrule
\multicolumn{7}{l}{\textit{\textbf{Low-Degree Homophilic --- Assumption~\ref{as2} fails as same-label pairs lack walk support}}} \\
& \textit{citeseer}   & 0.74 & $ 0.55$ & $-0.19$ & $0.74\times$ & $+2.99$  \\
& \textit{cora}       & 0.81 & $ 0.60$ & $-0.21$ & $0.74\times$ & $+0.07$  \\
& \textit{cycle}      & 0.90 & $0.55$ & $-0.35$ & $0.61\times$ & $\le 0$  \\
& \textit{grid}       & 0.91 & $ 0.60$ & $-0.31$ & $0.66\times$ & $\le 0$  \\
\bottomrule
\end{tabular}
\end{table}

\paragraph{Shortcut Quality and Node-Level Prediction Changes.} Here we analyze the \textit{actual nonlocal shortcut neighborhoods} added by \textsc{Graph Cascades} and relate them directly to node-level shifts in the true-class log-probability margin,
$$\Delta m_v = m_v^{\text{cascade}} - m_v^{\text{baseline}},$$
and separate prediction changes into:
\begin{itemize}
    \item \textbf{wrong-to-right (WR)}: errors corrected by rewiring.
    \item \textbf{right-to-wrong (RW)}: new errors introduced by rewiring.
\end{itemize}

Thus, WR and RW measure beneficial changes and injected prediction noise separately rather than combining them into a single net-accuracy value. Labels are used only for this post-hoc analysis; neither the rewiring procedure nor the model uses them.

\begin{table}[ht]
\centering
\small
\setlength{\tabcolsep}{6.3pt}
\renewcommand{\arraystretch}{1.05}
\caption{Node-level prediction change analysis (wrong-to-right vs. right-to-wrong) and mean log-probability margin changes ($\Delta m_v$) using NAGphormer as the backbone across graph regimes.}
\label{tab:nag_node_level_analysis}
\begin{tabular}{llcccc}
\hline
\textbf{Regime} & \textbf{Dataset} & $n_{\text{test}}$ & \textbf{WR : RW (MAS)} & \textbf{WR : RW (TAS)} & \textbf{\shortstack{Mean $\Delta m_v$ \\ (MAS / TAS)}} \\
\hline
Heterophilic & \textit{chameleon} & 570 & $195 : 78 = \mathbf{2.50\times}$ & $203 : 95 = \mathbf{2.14\times}$ & $+0.60 / +0.57$ \\
             & \textit{texas}     & 47  & $14 : 12 = 1.17\times$ & $13 : 8 = 1.62\times$ & $+0.12 / +0.46$ \\
             & \textit{wisconsin} & 64  & $15 : 11 = 1.36\times$ & $13 : 12 = 1.08\times$ & $+0.51 / +0.05$ \\
             & \textit{actor}     & 1900 & $425 : 402 = 1.06\times$ & $417 : 388 = 1.07\times$ & $+0.02 / +0.06$ \\
             & \textit{cornell}   & 47  & $13 : 10 = 1.30\times$ & $10 : 14 = 0.71\times$ & $-0.12 / -0.49$ \\
\hline
\multirow{2}{*}{\shortstack[l]{Mid to high-degree \\ homophilic}} & \textit{community} & 350 & $102 : 84 = 1.21\times$ & $111 : 74 = 1.50\times$ & $+0.48 / +0.60$ \\
& \textit{shape} & 175 & $55 : 37 = 1.49\times$ & $83 : 15 = \mathbf{5.53\times}$ & $-0.23 / +0.82$\\

\hline
\multirow{2}{*}{\shortstack[l]{Low-degree \\ homophilic}} & \textit{cora} & 677 & $24 : 42 = 0.57\times$ & $22 : 14 = 1.57\times$ & $-0.14 / -0.07$ \\
& \textit{citeseer} & 833 & $154 : 183 = 0.84\times$ & $165 : 169 = 0.98\times$ & $-0.28 / -0.15$ \\
\hline
\multirow{2}{*}{\shortstack[l]{Bottlenecked \\ or regular}} & \textit{grid} & 309 & $55 : 70 = 0.79\times$ & $56 : 76 = 0.74\times$ & $-0.27 / -0.23$ \\
& \textit{cycle} & 219 & $17 : 77 = \mathbf{0.22\times}$ & $21 : 79 = \mathbf{0.27\times}$ & $-2.65 / -2.66$ \\
\hline
\end{tabular}
\end{table} 

Table \ref{tab:nag_node_level_analysis} reports the analysis with NAGphormer as the backbone. The node-level pattern closely follows the regimes identified in Table~\ref{tab:regimes}. Among the 14 dataset--operator pairs in the heterophilic and moderate to high-degree homophilic regimes, \textit{13 produce more corrected predictions than newly introduced errors}, and \textit{11 also increase the mean true-class margin}. The strongest examples are \textit{chameleon}, where \textsc{MAS} and \textsc{TAS} correct $2.50\times$ and $2.14\times$ as many predictions as they corrupt, and \textit{shape}, where \textsc{TAS} reaches a $5.53\times$ ratio.

The predicted failure regimes provide a useful negative control. In seven of the eight low-degree or bottlenecked pairs, $\text{WR} \le \text{RW}$, and all eight have negative mean $\Delta m_v$. The clearest example is \textit{cycle}: \textsc{MAS} and \textsc{TAS} introduce approximately four times as many errors as they correct, consistent with the lack of redundant short-path support identified by Proposition~\ref{prop:tas-edge-connectivity-certificate} and Corollary~\ref{cor:tas-cut-obstruction}.

Overall, the analysis shows that cascade rewiring does not merely raise aggregate graph homophily: in the regimes where it is designed to operate, the added shortcut neighborhoods predominantly \textit{correct existing errors, increase prediction margins, and introduce substantially fewer new errors}.

\subsection{Correlation Between Classification Accuracy and Structural Properties}\label{app:correlation}

Table~\ref{tab:correlation} suggests a correlation between the classification accuracy of tokenized sparse-attention GTs and graph structural properties, including label homophily, the average degree of the auxiliary graph, and the connectivity of the original graph. Notably, the correlations obtained from cascade models are clearer than those from the baseline GTs. Since these patterns are observed on the combined set of synthesized and benchmark graphs—and the synthesized graphs dominate in number—an important question is whether the same correlations persist independently within benchmark graphs.

Due to the imbalance in the homophily distribution of the benchmark graphs, performing separate regressions on the synthesized and benchmark graph sets is not feasible. Instead, we plot test accuracy against the input graph's label homophily and encode additional structural properties—the average degree of the input graph (original graph for NAG and VCR, and auxiliary graph for the cascade models), the connectivity of the original graph, and whether the graph is a DGL benchmark—using marker size, shape, and color, respectively. Figure~\ref{fig:correlation} illustrates the interaction among these structural factors, consistent with the regression results in Table~\ref{tab:correlation}; the patterns observed in the DGL benchmark graphs closely align with those of the synthesized graphs.

\subsection{Ablation Study}\label{ablation}

In this section, we present a parameter analysis of \textsc{CR-Graphormer} and justify the use of cascade frequency-based edge weights $W^\ast$ through additional classification results.

\begin{figure}[ht!]
    \centering
    \includegraphics[width=0.49\linewidth]{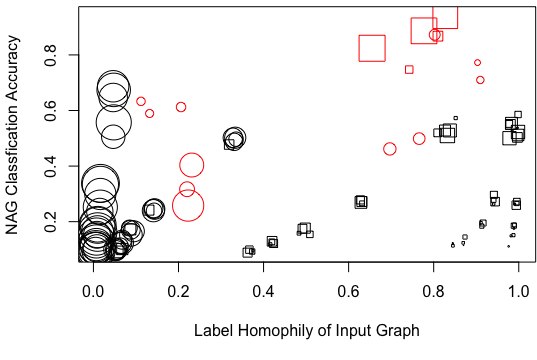}\hfill
    \includegraphics[width=0.49\linewidth]{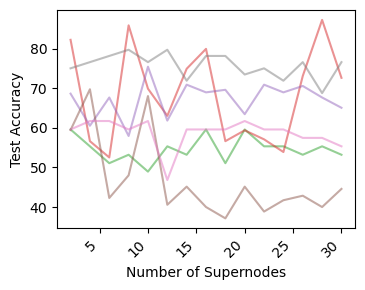}\\
    \includegraphics[width=0.49\linewidth]{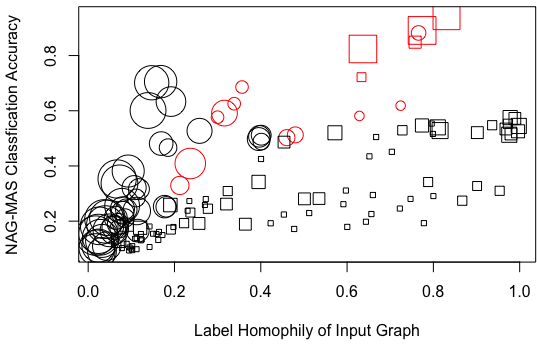}\hfill
    \includegraphics[width=0.49\linewidth]{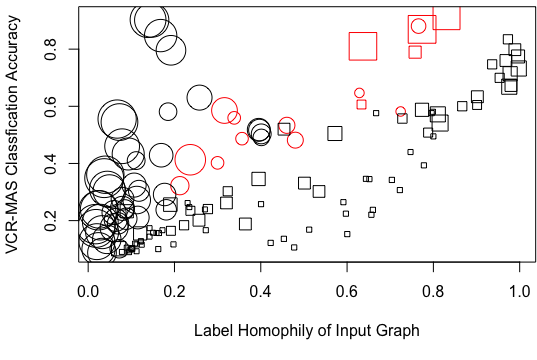}\\
    \includegraphics[width=0.49\linewidth]{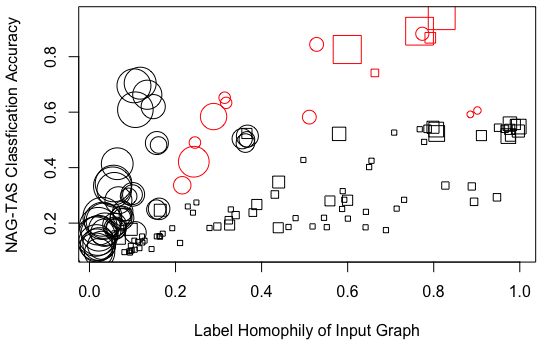}\hfill
    \includegraphics[width=0.49\linewidth]{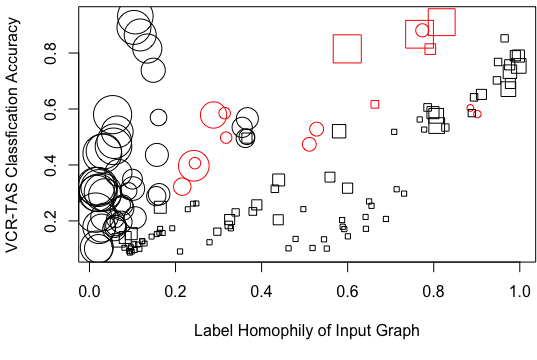}\\
    \includegraphics[width=0.49\linewidth]{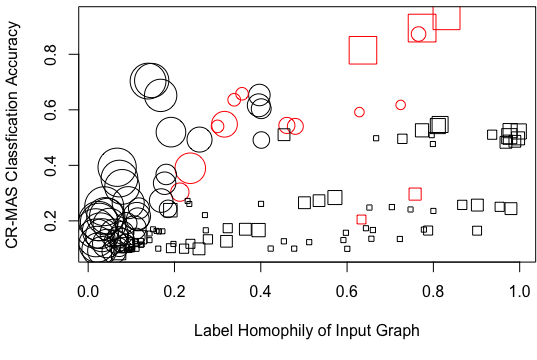}\hfill
    \includegraphics[width=0.49\linewidth]{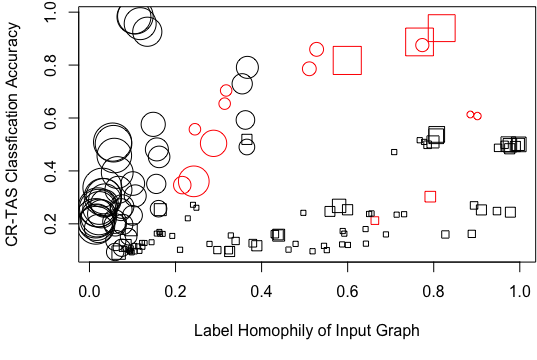}\\
    \vspace{0.3cm}
    \includegraphics[height=1.1cm]{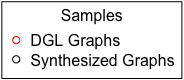} \hspace{0.75cm}
    \includegraphics[height=1.1cm]{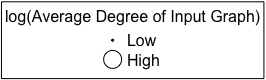} \hspace{0.75cm}
    \includegraphics[height=1.1cm]{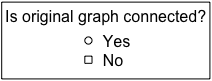}\\
    \vspace{0.3cm}
    \caption{Mean test accuracy vs. structural properties.}
    \label{fig:correlation}
\end{figure}

\subsubsection{Parameter Analysis for \textsc{CR-Graphormer}}

We perform a sensitivity analysis with respect to the main hyperparameters controlling the MAS and TAS cascades, using \textsc{CR-Graphormer} as the cascade model. In this experiment, we retain the same experimental settings as in previous sections (see \textbf{Experimental Setup} in Section~\ref{sec:empirical} and \textbf{Implementation Details} in Appendix~\ref{app:exp}), and vary one adjacency search parameter at a time, including the top-\(k\) parameter \(k\), walk length \(\ell\), number of starting neighbors \(r = \lvert R \rvert\), number of permutations \(P\), and the maximum threshold \(T\) with the threshold set \(L = \{1,2,\dots,T\}\). The resulting mean test accuracies are shown in Figures~\ref{fig:MASablation} and~\ref{fig:TASablation} on $10$ representative datasets
(heterophilic: \textit{chameleon, cornell, texas, wisconsin};
moderate-degree homophilic: \textit{community, shape};
low-degree homophilic: \textit{citeseer, cora, cycle, grid}). 

\begin{figure}[ht!]
    \centering
    \includegraphics[width=0.49\linewidth]{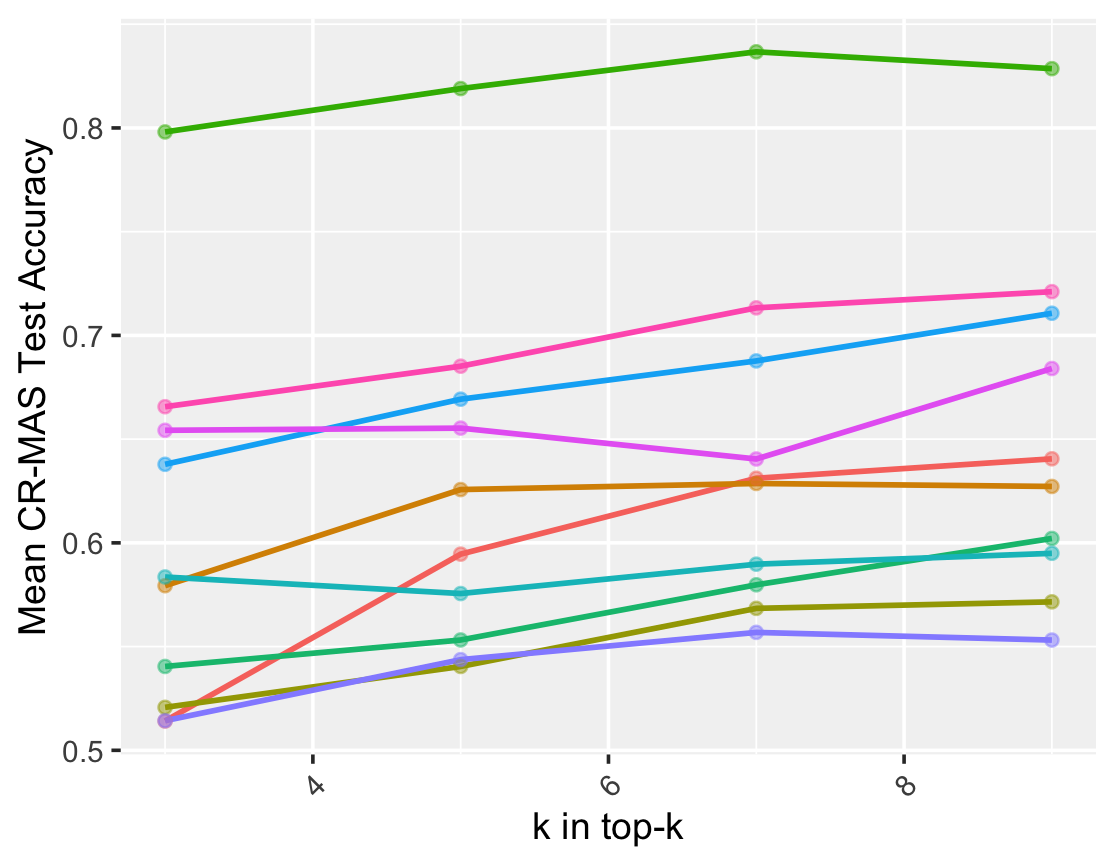}\hfill
    \includegraphics[width=0.49\linewidth]{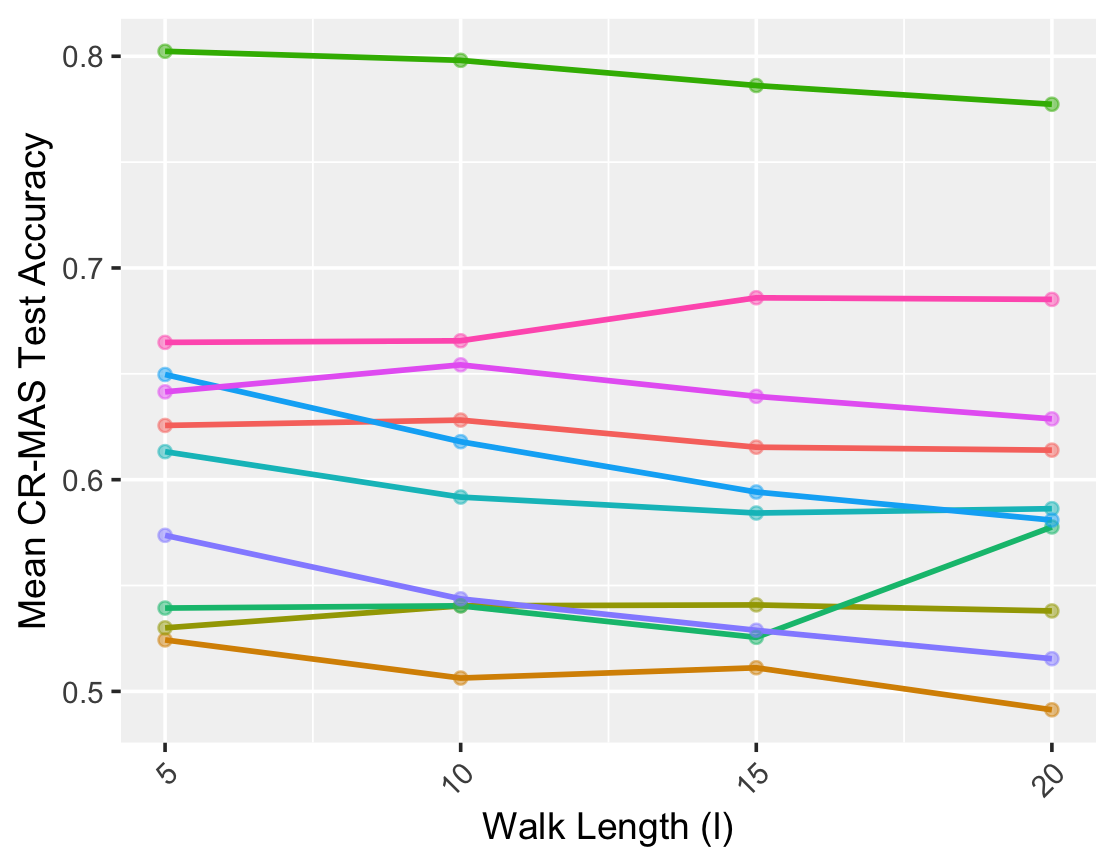}\\
    \includegraphics[width=0.49\linewidth]{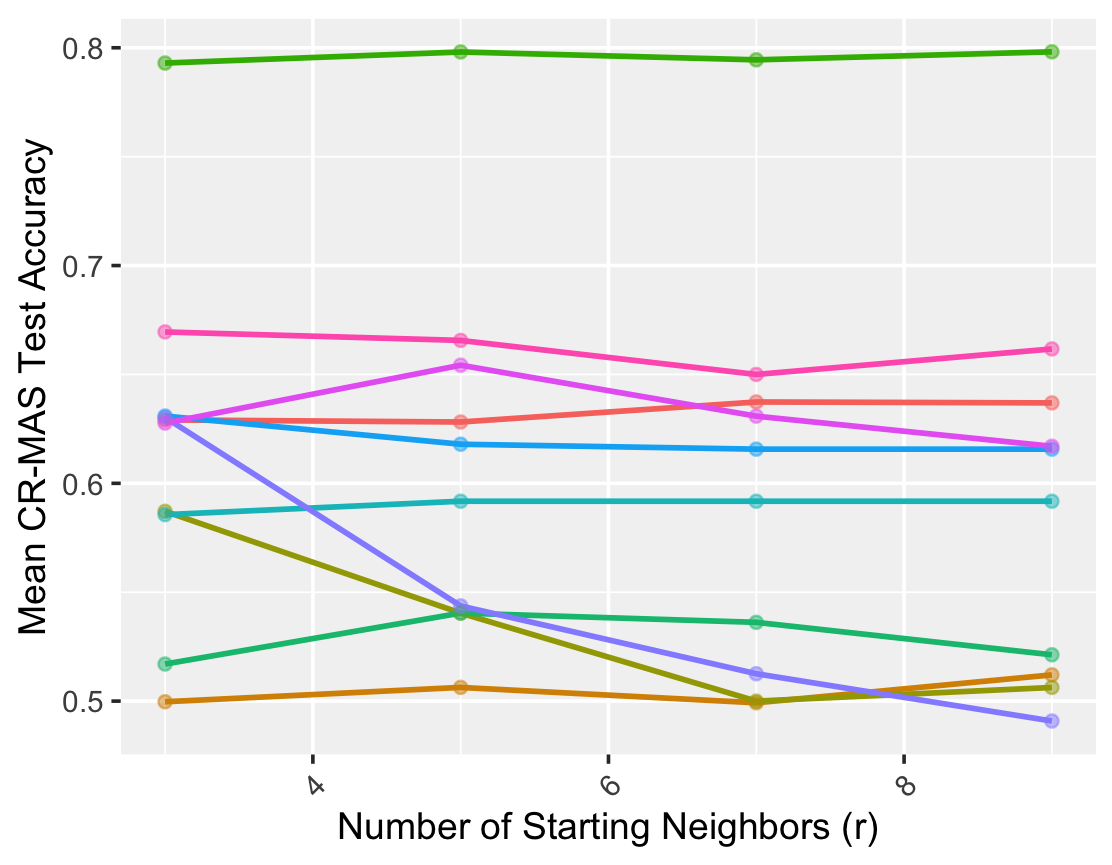}\hfill
    \includegraphics[width=0.49\linewidth]{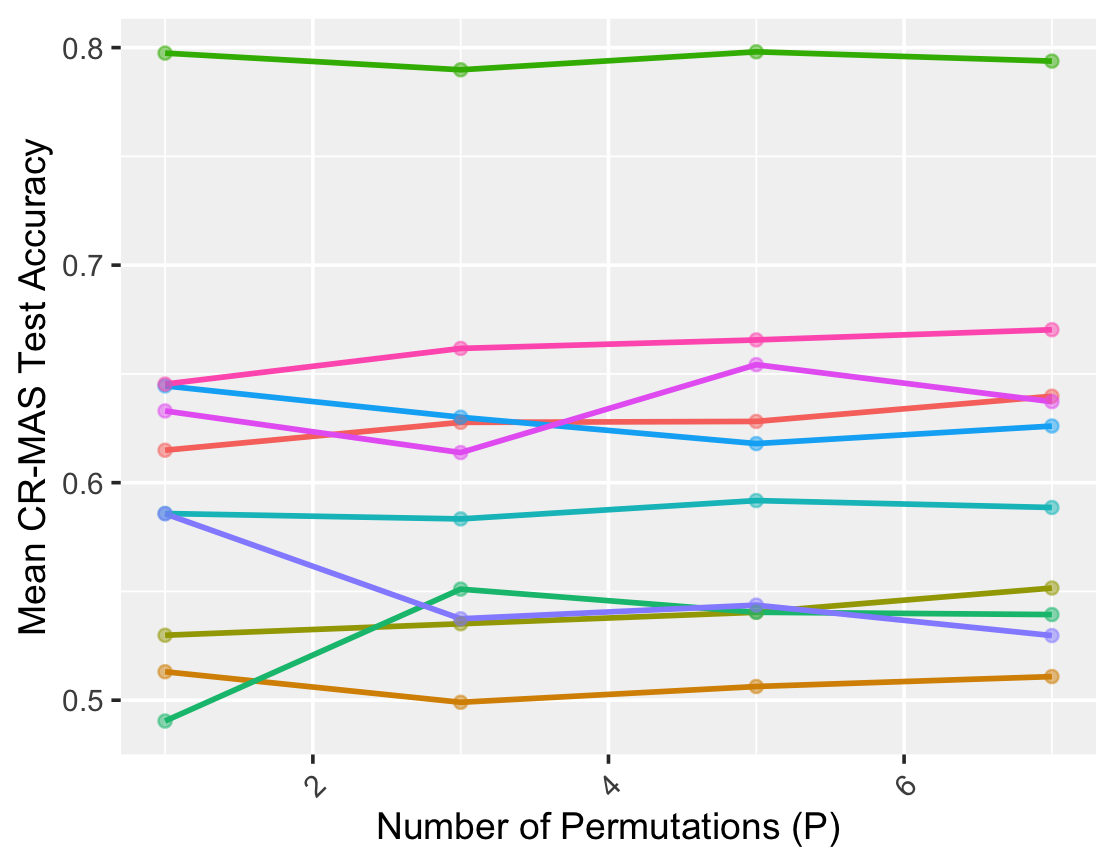}\\
    \vspace{0.35cm}
    \includegraphics[height=0.9cm]{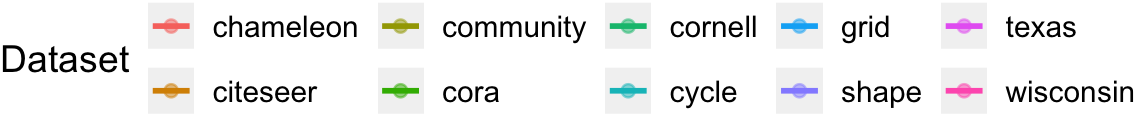}
    \vspace{0.35cm}
    \caption{Parameter Analysis for \textsc{CR-MAS}.}
    \label{fig:MASablation}
\end{figure}

\paragraph{CR-MAS (Figure~\ref{fig:MASablation}).}
\begin{itemize}[leftmargin=0.5cm, labelindent=0pt]
    \item \textbf{Increasing $k$ (auxiliary graph degree/token length) is generally beneficial.}
    Increasing $k$ ($3 \!\to\! 9$) improves mean accuracy on most datasets, often monotonically, with gains saturating around mid/large $k$.
    The trend is especially clear on harder heterophilic/low-degree cases (e.g., \textit{chameleon} rises from $\sim$0.52 to $\sim$0.64; \textit{grid} from $\sim$0.64 to $\sim$0.71), while other high-performing datasets (e.g., \textit{cora}) remain strong and relatively stable (around $\sim$$0.80$-$0.84$). 

    \item \textbf{Walk length $\ell$ exhibits diminishing returns (often negative) for MAS.}
    Longer cascades ($\ell\!=\!5\!\to\!20$) tend to \emph{reduce} accuracy for many datasets (notably \textit{citeseer}, \textit{grid}, and \textit{shape}), suggesting that MAS benefits most from a \emph{tighter} mesoscopic radius; extending too far likely introduces weaker/less relevant co-activations. This observation is consistent with Theorem \ref{thm:walklength}. A few datasets show mild robustness or slight recovery at $\ell=20$, but the dominant pattern is ``smaller $\ell$ is more stable'' for MAS.

    \item \textbf{Starting-neighborhood size $r$, when large, could degrade performance for MAS.}
    Varying $r$ ($3\!\to\!9$) is mostly mild, but some datasets degrade sharply with larger $r$ (e.g., \textit{shape} drops from $\sim$0.63 to $\sim$0.49; \textit{community} from $\sim$0.58 to $\sim$0.50). This indicates MAS is sensitive to noisy initialization when the initial active set is too wide. A safer choice is choosing $r \leq d_{avg}$, where $d_{avg}$ is the average degree in the graph.

    \item \textbf{Number of permutations $P$ is a second-order effect for MAS.}
    Changing $P$ ($1\!\to\!7$) yields relatively small shifts on most datasets, with only modest improvements at intermediate values (some datasets peak around $P\approx 3$--$5$). Overall, MAS appears fairly robust to $P$ once a small amount of resampling is used.
\end{itemize}

\begin{figure}[ht!]
    \centering
    \includegraphics[width=0.49\linewidth]{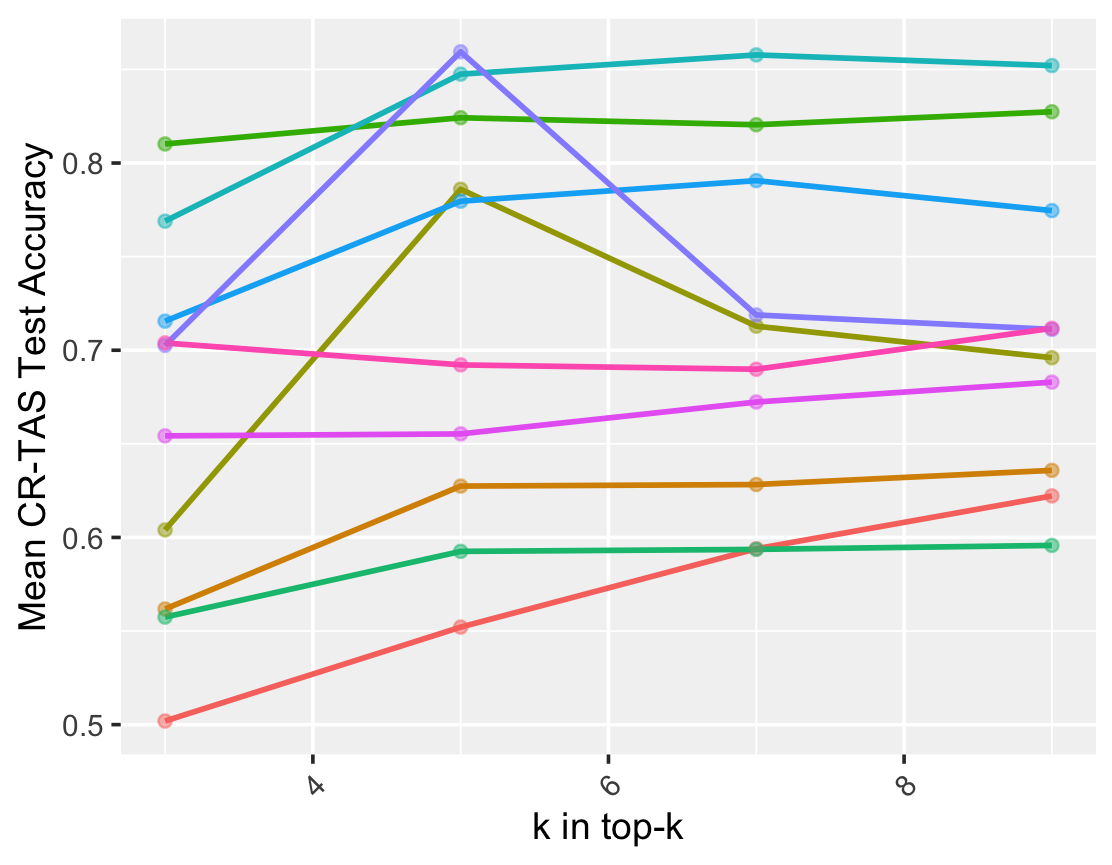}\hfill
    \includegraphics[width=0.49\linewidth]{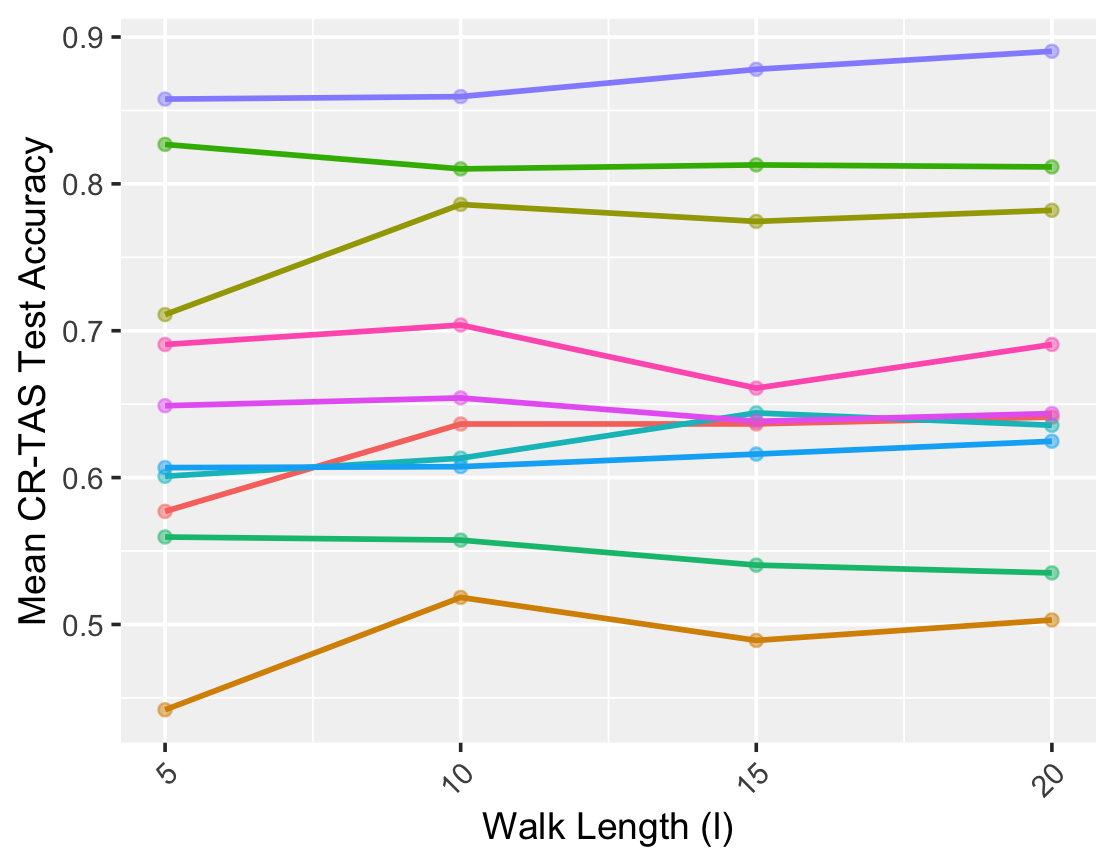}\\
    \includegraphics[width=0.49\linewidth]{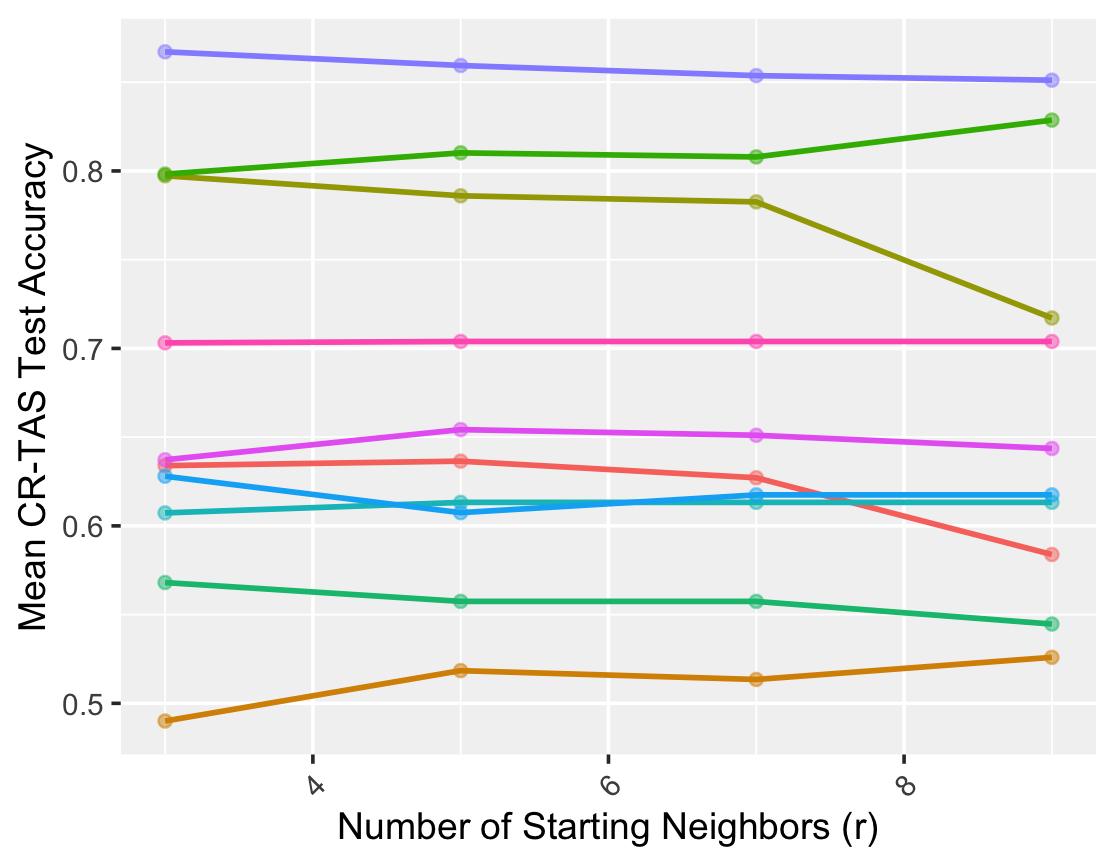}\hfill
    \includegraphics[width=0.49\linewidth]{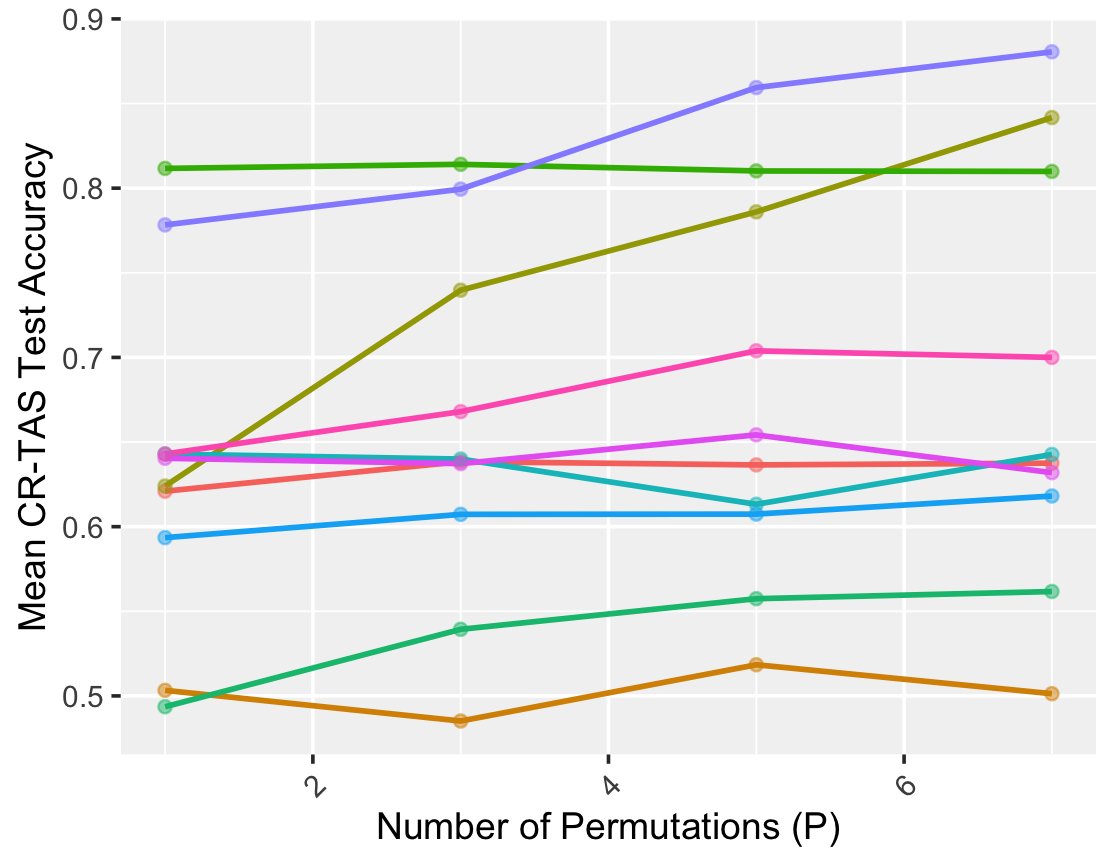}\\
    \includegraphics[width=0.49\linewidth]{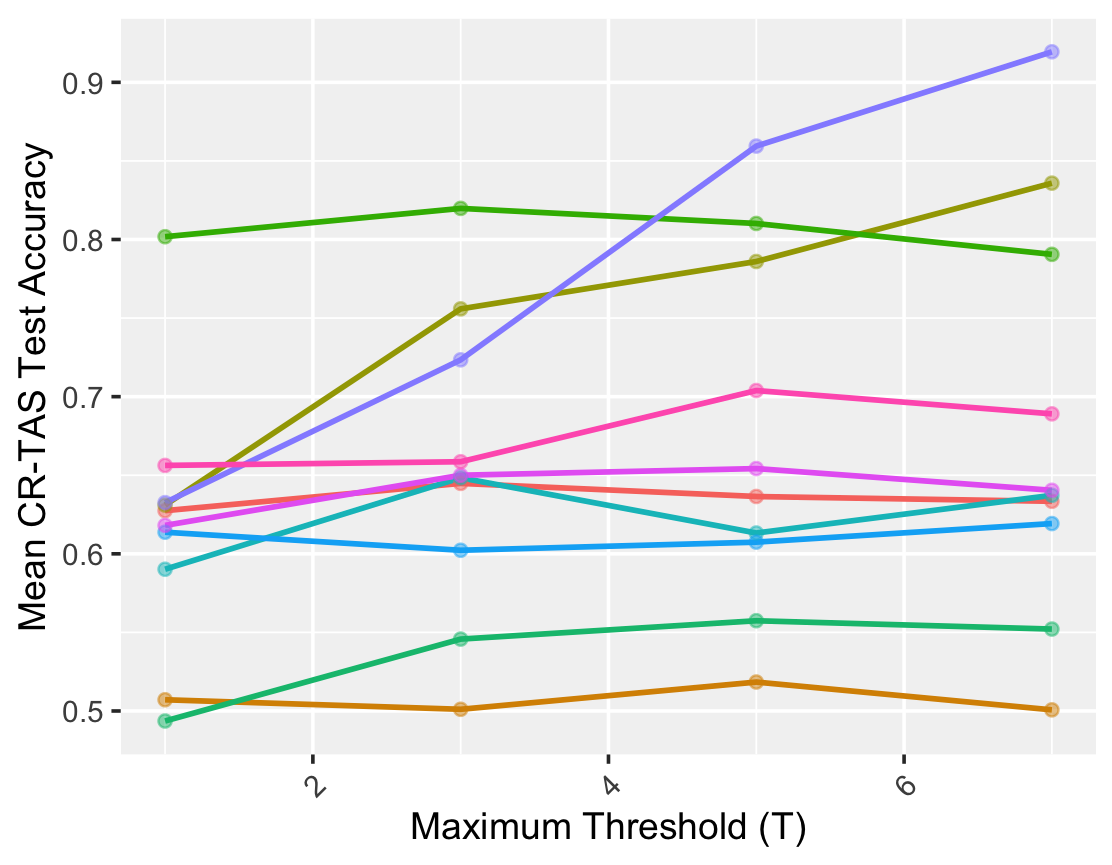}
    \includegraphics[width=0.15\linewidth]{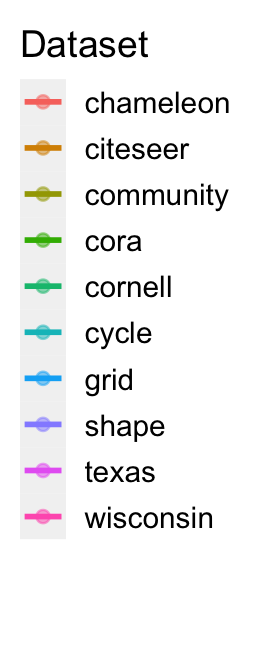}
    \caption{Parameter Analysis for \textsc{CR-TAS}.}
    \label{fig:TASablation}
\end{figure}

\newpage

\paragraph{CR-TAS (Figure~\ref{fig:TASablation}).}
\begin{itemize}[leftmargin=0.5cm, labelindent=0pt]
    \item \textbf{Top-$k$ is non-monotonic for several datasets.}
    Unlike MAS, TAS shows clear non-monotonicity: some datasets peak at intermediate $k$ (notably \textit{community} and \textit{shape}, which peak at $k\approx 5$ and then drop when $k$ is increased further),
    while heterophilic \textit{chameleon} improves steadily with larger $k$ ($\sim$0.50 $\to$ $\sim$0.62). This suggests TAS can over-include less useful neighbors if $k$ becomes too large, depending on the graph type. 

    \item \textbf{TAS is more tolerant of longer walk lengths.}
    Accuracy is typically stable or slightly increasing with $\ell$ (e.g., \textit{shape} increases from $\sim$0.86 to $\sim$0.89 as $\ell$ grows; \textit{chameleon} improves mainly from $\ell=5$ to $\ell=10$ and then plateaus).
    Compared to MAS, TAS does \emph{not} show the same widespread degradation as $\ell$ increases. This observation is consistent with thresholding preventing premature activation of weakly supported nodes. 

    \item \textbf{Starting-neighborhood size $r$ is mostly robust, but very large $r$ is generally not preferable.}
    Most datasets vary only slightly with $r$, but the largest $r$ setting causes noticeable drops for some (e.g., \textit{chameleon} and \textit{community}), implying that even under thresholding, an overly broad initialization can inject noise. A safer choice is to choose $r \leq d_{avg}$. 

    \item \textbf{Permutation count $P$ is more important under TAS than MAS.}
    Increasing $P$ yields pronounced gains for certain datasets---especially \textit{community} ($\sim$0.62 $\to$ $\sim$0.84) and \textit{shape} ($\sim$0.78 $\to$ $\sim$0.88), while others remain comparatively flat.
    This indicates TAS benefits more from additional resampling/permutation averaging, likely because its thresholded selection introduces more variance that is reduced with repeated runs. 

    \item \textbf{Maximum threshold $T$ (i.e., allowing larger reinforcement thresholds) yields greater gains.}
    Increasing $T$ produces the largest visible ablation effects: \textit{shape} rises dramatically ($\sim$0.63 $\to$ $\sim$0.92) and \textit{community} increases strongly ($\sim$0.63 $\to$ $\sim$0.84),
    whereas several other datasets change modestly and a few show slight saturation/decline at high $T$ (e.g., \textit{cora}).
    Overall, the figure suggests that enabling stronger multi-neighbor reinforcement (higher thresholds) can be crucial on some graphs, but is not uniformly beneficial across all datasets.
\end{itemize}

\subsubsection{The Use of $W^\ast$ in Cascaded Graph Learning}\label{weights}
Throughout the main body, \emph{cascade models} refer to models that receive the auxiliary graph $G^\ast$ together with cascade-frequency edge weights $W^\ast$ as input. We also evaluated the topology-only variant in which only $G^\ast$ is provided. Tables~\ref{tab:summary_unweighted} and~\ref{tab:numerical3} show that the effect of $W^\ast$ is architecture-dependent. For sparse-attention GTs, the topology alone already yields broad gains: VCR improves on 13 of 16 graphs and NAG-MAS on up to 6, with the largest single gain reaching $+33.31\%$ for VCR. For GNN-based architectures (GCN, GraphGPS), the topology-only variant produces fewer improvements (4 and 7 of 16, respectively, with gains up to $+4.30\%$), reflecting that GNN aggregation is more sensitive to degree imbalance and benefits more from weighting. Adding $W^\ast$ on top of $G^\ast$ extends the gains for all backbones: at least 10 graphs improve for MAS cascades (up to $+22.79\%$ over the topology-only variant) and at least 11 for TAS cascades (up to $+35.89\%$). Performance declines, when they occur, are small: the only notable drops are for VCR-MAS and VCR-TAS on \textit{cycle} and \textit{grid} ($-7.55\%$ to $-14.53\%$ relative to VCR with topology-only input) and VCR-MAS on \textit{community} ($-6.94\%$) and \textit{squirrel} ($-3.39\%$); all other declines are at most $1.46\%$.

\begin{table}[ht!]
\centering
\small\setlength{\tabcolsep}{8.4pt}
\caption{Number of the 16 DGL graphs whose test accuracy improves when using the auxiliary graph $G^\ast$, with or without cascade frequency-based edge weights $W^\ast$. Column 3 summarizes Table \ref{tab:numerical3}. Column 4 summarizes Table \ref{std1}. Column 5 summarizes the gains from Table \ref{std1} over Table \ref{tab:numerical3}.}
\begin{tabular}{lcccc}
\hline
\textbf{Baseline} & \shortstack{\textbf{Cascade} \\ \textbf{Mode}} & \shortstack{\textbf{Improvement from} \\ \textbf{$G^{*}$ Input}} & \shortstack{\textbf{Improvement from} \\ \textbf{Cascade Model}} & \shortstack{\textbf{Improvement of Cascade} \\ \textbf{Model over $G^{*}$ Input}} \\
\hline
GCN & MAS & 4 (up to \blue{\textbf{+3.05\%}}) & 8 (up to \red{\textbf{+17.11\%}}) & 14 (up to \green{\textbf{+16.60\%}}) \\
 & TAS & 3 (up to \blue{\textbf{+3.23\%}}) & 10 (up to \red{\textbf{+20.15\%}}) & 14 (up to \green{\textbf{+30.66\%}})\\
\hline
GraphGPS & MAS & 7 (up to \blue{\textbf{+4.29\%}}) & 8 (up to \red{\textbf{+8.15\%}}) & 10 (up to \green{\textbf{+7.02\%}})\\
 & TAS & 4 (up to \blue{\textbf{+2.45\%}}) & 9 (up to \red{\textbf{+9.52\%}}) & 11 (up to \green{\textbf{+12.63\%}})\\
\hline
NAG & MAS & 6 (up to \blue{\textbf{+19.18\%}}) & 11 (up to \red{\textbf{+19.77\%}}) & 15 (up to \green{\textbf{+19.15\%}})\\
 & TAS & 5 (up to \blue{\textbf{+18.90\%}}) & 13 (up to \red{\textbf{+34.54\%}}) & 14 (up to \green{\textbf{+35.89\%}})\\
\hline
VCR & MAS & 12 (up to \blue{\textbf{+32.81\%}}) & 13 (up to \red{\textbf{+32.74\%}}) & 11 (up to \green{\textbf{+22.79\%}})\\
 & TAS & 13 (up to \blue{\textbf{+33.31\%}}) & 13 (up to \red{\textbf{+32.33\%}}) & 11 (up to \green{\textbf{+29.65\%}})\\
\hline
\end{tabular}
\label{tab:summary_unweighted}
\end{table}

\begin{table}[ht!]
\centering
\tiny
\setlength{\tabcolsep}{6.5pt}
\caption{Test accuracy (mean $\pm$ std, in \%) across DGL datasets where non-baseline methods input auxiliary graphs into the baseline models without frequency-based edge weights. OOM indicates out-of-memory errors. Darker cell colors indicate better performance.}
\begin{tabular}{lcccccccc}
\hline
\textbf{Method} & \textbf{\textit{actor}} & \textbf{\textit{chameleon}} & \textbf{\textit{citeseer}} & \textbf{\textit{community}} & \textbf{\textit{computer}} & \textbf{\textit{cora}} & \textbf{\textit{cornell}} & \textbf{\textit{cycle}} \\
\hline
GCN & \cellcolor{myred!60} 29.59$\pm$0.91 & 57.30$\pm$2.67 & \cellcolor{myred!60} 75.70$\pm$1.41 & \cellcolor{myred!60} 55.63$\pm$1.92 & \cellcolor{myred!60} 90.30$\pm$0.66 & \cellcolor{myred!60} 87.96$\pm$0.83 & \cellcolor{myred!60} 45.11$\pm$6.20 & \cellcolor{myred!60} 60.73$\pm$6.59 \\
GCN-MAS & 28.27$\pm$0.88 & \cellcolor{myred!30} 59.12$\pm$2.37 & 70.86$\pm$1.75 & \cellcolor{myred!30} 50.39$\pm$2.54 & 88.25$\pm$0.74 & 83.57$\pm$1.05 & \cellcolor{myred!30} 44.26$\pm$5.15 & 59.22$\pm$2.69 \\
GCN-TAS & \cellcolor{myred!30} 28.31$\pm$0.72 & \cellcolor{myred!60} 59.99$\pm$2.31 & \cellcolor{myred!30} 73.24$\pm$1.28 & 47.60$\pm$2.17 & \cellcolor{myred!30} 88.93$\pm$0.83 & \cellcolor{myred!30} 86.26$\pm$0.89 & 41.28$\pm$6.74 & \cellcolor{myred!30} 59.27$\pm$2.73 \\
\hline
GraphGPS & 35.19$\pm$1.31 & \cellcolor{mygreen!60} 60.78$\pm$1.35 & \cellcolor{mygreen!60} 75.23$\pm$1.63 & \cellcolor{mygreen!30} 44.00$\pm$2.30 & OOM & \cellcolor{mygreen!60} 86.65$\pm$1.04 & \cellcolor{mygreen!30} 61.38$\pm$5.49 & 55.11$\pm$10.10 \\
GPS-MAS & \cellcolor{mygreen!30} 35.24$\pm$1.32 & \cellcolor{mygreen!30} 59.73$\pm$1.71 & 72.63$\pm$1.48 & \cellcolor{mygreen!60} 47.94$\pm$3.11 & OOM & 85.08$\pm$1.48 & \cellcolor{mygreen!60} 61.70$\pm$6.14 & \cellcolor{mygreen!30} 56.26$\pm$3.37 \\
GPS-TAS & \cellcolor{mygreen!60} 35.25$\pm$1.03 & 59.55$\pm$1.55 & \cellcolor{mygreen!30} 73.83$\pm$1.42 & 42.33$\pm$1.72 & OOM & \cellcolor{mygreen!30} 86.16$\pm$1.00 & 60.43$\pm$4.95 & \cellcolor{mygreen!60} 57.56$\pm$5.66 \\
\hline
NAG & 31.68$\pm$1.26 & 40.38$\pm$1.81 & \cellcolor{mycyan!60} 74.72$\pm$1.67 & \cellcolor{mycyan!30} 46.17$\pm$2.49 & \cellcolor{mycyan!30} 88.66$\pm$0.73 & \cellcolor{mycyan!60} 86.75$\pm$0.82 & \cellcolor{mycyan!60} 58.94$\pm$7.11 & \cellcolor{mycyan!60} 77.24$\pm$9.34 \\
NAG-MAS & \cellcolor{mycyan!60} 33.08$\pm$1.41 & \cellcolor{mycyan!60} 59.56$\pm$1.27 & 70.82$\pm$1.66 & \cellcolor{mycyan!60} 49.66$\pm$2.15 & 88.46$\pm$0.78 & 83.69$\pm$0.87 & \cellcolor{mycyan!30} 56.49$\pm$7.26 & 57.99$\pm$2.53 \\
NAG-TAS & \cellcolor{mycyan!30} 32.88$\pm$1.14 & \cellcolor{mycyan!30} 59.28$\pm$1.19 & \cellcolor{mycyan!30} 72.67$\pm$1.59 & 42.90$\pm$1.52 & \cellcolor{mycyan!60} 88.87$\pm$0.74 & \cellcolor{mycyan!30} 86.05$\pm$1.13 & 44.89$\pm$6.36 & \cellcolor{mycyan!30} 59.82$\pm$2.34 \\
\hline
VCR & 26.41$\pm$1.20 & 27.28$\pm$2.51 & \cellcolor{mypurple!60} 71.15$\pm$1.36 & 28.13$\pm$4.60 & 82.45$\pm$5.30 & \cellcolor{mypurple!60} 85.41$\pm$1.33 & 44.89$\pm$6.09 & \cellcolor{mypurple!60} 61.64$\pm$4.12 \\
VCR-MAS & \cellcolor{mypurple!30} 31.60$\pm$1.73 & \cellcolor{mypurple!30} 60.09$\pm$2.48 & 63.45$\pm$1.77 & \cellcolor{mypurple!60} 54.16$\pm$2.39 & \cellcolor{mypurple!30} 89.32$\pm$0.69 & 80.76$\pm$0.92 & \cellcolor{mypurple!60} 51.06$\pm$6.33 & \cellcolor{mypurple!30} 59.25$\pm$2.94 \\
VCR-TAS & \cellcolor{mypurple!60} 32.43$\pm$1.70 & \cellcolor{mypurple!60} 60.59$\pm$2.67 & \cellcolor{mypurple!30} 66.03$\pm$1.72 & \cellcolor{mypurple!30} 49.37$\pm$1.84 & \cellcolor{mypurple!60} 89.56$\pm$0.71 & \cellcolor{mypurple!30} 82.56$\pm$0.96 & \cellcolor{mypurple!30} 48.83$\pm$5.41 & 59.20$\pm$2.73 \\
\hline
\end{tabular}

\setlength{\tabcolsep}{6.25pt}
\begin{tabular}{lcccccccc}
\hline
\textbf{Method} & \textbf{\textit{grid}} & \textbf{\textit{photo}} & \textbf{\textit{pubmed}} & \textbf{\textit{shape}} & \textbf{\textit{squirrel}} & \textbf{\textit{texas}} & \textbf{\textit{wiki}} & \textbf{\textit{wisconsin}} \\
\hline
GCN & \cellcolor{myred!60} 59.34$\pm$5.45 & \cellcolor{myred!60} 93.90$\pm$0.74 & \cellcolor{myred!60} 87.16$\pm$0.35 & \cellcolor{myred!60} 89.46$\pm$11.87 & 39.10$\pm$1.43 & 54.57$\pm$7.12 & \cellcolor{myred!60} 83.89$\pm$0.73 & \cellcolor{myred!30} 51.41$\pm$4.81 \\
GCN-MAS & \cellcolor{myred!30} 58.19$\pm$1.76 & 91.75$\pm$1.09 & 83.78$\pm$0.42 & 54.69$\pm$4.47 & \cellcolor{myred!30} 40.49$\pm$3.42 & \cellcolor{myred!30} 56.91$\pm$8.25 & 79.72$\pm$0.95 & \cellcolor{myred!60} 54.45$\pm$5.72 \\
GCN-TAS & \cellcolor{myred!30} 58.19$\pm$1.76 & \cellcolor{myred!30} 92.44$\pm$0.83 & \cellcolor{myred!30} 84.69$\pm$0.38 & \cellcolor{myred!30} 58.91$\pm$7.93 & \cellcolor{myred!60} 42.33$\pm$1.83 & \cellcolor{myred!60} 57.02$\pm$6.04 & \cellcolor{myred!30} 81.34$\pm$0.74 & 51.09$\pm$4.03 \\
\hline
GraphGPS & \cellcolor{mygreen!60} 57.41$\pm$10.00 & \cellcolor{mygreen!60} 94.64$\pm$0.63 & OOM & \cellcolor{mygreen!30} 46.74$\pm$14.93 & \cellcolor{mygreen!60} 40.04$\pm$1.27 & 69.68$\pm$6.58 & OOM & 72.66$\pm$4.95 \\
GPS-MAS & 53.48$\pm$7.51 & \cellcolor{mygreen!30} 94.46$\pm$0.63 & OOM & \cellcolor{mygreen!60} 48.14$\pm$4.45 & \cellcolor{mygreen!30} 39.24$\pm$1.50 & \cellcolor{mygreen!60} 72.45$\pm$7.02 & OOM & \cellcolor{mygreen!60} 76.95$\pm$3.58 \\
GPS-TAS & \cellcolor{mygreen!30} 56.33$\pm$5.59 & 94.06$\pm$0.97 & OOM & 43.63$\pm$8.60 & 38.10$\pm$1.27 & \cellcolor{mygreen!30} 69.79$\pm$7.19 & OOM & \cellcolor{mygreen!30} 74.14$\pm$5.23 \\
\hline
NAG & \cellcolor{mycyan!60} 70.99$\pm$1.99 & \cellcolor{mycyan!60} 93.94$\pm$0.67 & 87.26$\pm$0.52 & \cellcolor{mycyan!60} 49.89$\pm$6.51 & 25.80$\pm$1.26 & \cellcolor{mycyan!60} 63.30$\pm$6.43 & \cellcolor{mycyan!60} 82.40$\pm$0.98 & \cellcolor{mycyan!30} 61.25$\pm$4.32 \\
NAG-MAS & \cellcolor{mycyan!30} 58.22$\pm$1.72 & 93.45$\pm$0.69 & \cellcolor{mycyan!30} 88.08$\pm$0.41 & 45.77$\pm$4.53 & \cellcolor{mycyan!30} 39.11$\pm$1.98 & \cellcolor{mycyan!30} 62.13$\pm$7.06 & \cellcolor{mycyan!30} 81.39$\pm$0.92 & \cellcolor{mycyan!60} 67.73$\pm$3.80 \\
NAG-TAS & 57.91$\pm$1.99 & \cellcolor{mycyan!30} 93.68$\pm$0.48 & \cellcolor{mycyan!60} 88.15$\pm$0.40 & \cellcolor{mycyan!30} 48.54$\pm$4.64 & \cellcolor{mycyan!60} 39.72$\pm$1.62 & 61.60$\pm$6.60 & 81.29$\pm$0.67 & 58.28$\pm$5.53 \\
\hline
VCR & 58.32$\pm$1.96 & 88.45$\pm$7.85 & 83.25$\pm$0.44 & \cellcolor{mypurple!30} 44.66$\pm$6.15 & 22.48$\pm$2.00 & 54.68$\pm$7.28 & 78.46$\pm$2.69 & 51.33$\pm$7.46 \\
VCR-MAS & \cellcolor{mypurple!30} 58.58$\pm$1.62 & \cellcolor{mypurple!30} 92.99$\pm$0.80 & \cellcolor{mypurple!30} 87.18$\pm$0.63 & 43.11$\pm$2.97 & \cellcolor{mypurple!30} 42.76$\pm$2.58 & \cellcolor{mypurple!60} 62.45$\pm$8.32 & \cellcolor{mypurple!30} 81.94$\pm$0.76 & \cellcolor{mypurple!60} 67.11$\pm$5.97 \\
VCR-TAS & \cellcolor{mypurple!60} 59.24$\pm$3.01 & \cellcolor{mypurple!60} 93.39$\pm$0.71 & \cellcolor{mypurple!60} 87.35$\pm$0.42 & \cellcolor{mypurple!60} 47.29$\pm$5.49 & \cellcolor{mypurple!60} 43.17$\pm$3.30 & \cellcolor{mypurple!30} 61.70$\pm$7.78 & \cellcolor{mypurple!60} 82.86$\pm$0.70 & \cellcolor{mypurple!30} 65.86$\pm$5.83 \\
\hline
\end{tabular}
\label{tab:numerical3}
\end{table}

These observations can be explained by the fact that the frequency of node activations under MAS and TAS captures the mesoscopic structural information beyond the raw graph topology. Nodes that are repeatedly co-activated with a given seed are typically embedded in the same cohesive cluster or dense region, while nodes in sparser regions or across weak cuts activate less frequently. Using these frequencies as edge weights biases the GNN and GT message-passing mechanisms toward structurally relevant neighbors, amplifying signals from densely connected clusters and downweighting less informative connections. In contrast, simply inputting the auxiliary graph with unit edge weights provides only topological information without distinguishing the relative importance of different edges, which limits the improvement in classification. Consequently, the cascade frequency-weighted graphs enable models to exploit higher-order connectivity and cluster cohesion, resulting in larger and more consistent gains across datasets and architectures.

\subsection{Comparison to Existing Rewiring Baselines}\label{app:rewiring}

The VCR block in Tables \ref{tab:summary} and \ref{tab:numerical1} already provides a controlled comparison between Personalized PageRank (PPR) and cascade-based token selection under identical model configurations and training protocols. Replacing PPR token lists with cascade-selected lists yields significant performance improvements, outperforming PPR on $13/16$ DGL benchmarks ($12/12$ graphs in target regimes, up to $+32.74$ points) and $6/7$ additional datasets in Appendix \ref{app:longrange} (up to +39.05\%). Additionally, broader PPR and diffusion-propagation methods are already represented in the paper through baselines such as PPRGo and GRAND+.

To further evaluate these mechanisms, we compare GCN cascade models (\textsc{GCN-MAS} and \textsc{GCN-TAS}) against the unrewired \textsc{GCN} baseline as well as strong structural rewiring methods, including \textsc{GCN+FoSR} \citep{karhadkar2023fosr} and \textsc{GCN+SDRF} \citep{topping2022understanding}, across all 16 DGL benchmarks under the experimental protocol described in Section~\ref{sec:empirical}.

\begin{table}[!ht]
\centering
\small
\setlength{\tabcolsep}{6pt}
\renewcommand{\arraystretch}{1.05}
\caption{Test accuracy (mean, in \%) of GCN cascade models compared with unrewired GCN and structural rewiring baselines \textsc{GCN+FoSR} and \textsc{GCN+SDRF} across all 16 DGL benchmarks. Bold text indicates the best performance per dataset.}
\label{tab:rewiring}
\begin{tabular}{lccccc}
\hline
\textbf{Dataset} & \textbf{GCN} & \textbf{GCN+FoSR} & \textbf{GCN+SDRF} & \textbf{GCN-MAS} & \textbf{GCN-TAS} \\
\hline
\textit{actor}     & 29.59 & 32.11 & 31.73 & 34.28 & \textbf{34.62} \\
\textit{chameleon} & 57.30 & 43.66 & 44.79 & 57.67 & \textbf{58.75} \\
\textit{citeseer}  & 75.70 & \textbf{76.35} & 71.91 & 74.71 & 75.46 \\
\textit{community} & \textbf{55.63} & 51.09 & 50.34 & 51.76 & 50.11 \\
\textit{computer}  & 90.30 & 83.78 & 80.11 & 89.34 & \textbf{90.87} \\
\textit{cora}      & \textbf{87.96} & 86.10 & 85.84 & 84.93 & 87.29 \\
\textit{cornell}   & 45.11 & 50.74 & 51.54 & \textbf{60.85} & 57.02 \\
\textit{cycle}     & 60.73 & \textbf{68.04} & 66.91 & 59.27 & 60.39 \\
\textit{grid}      & 59.34 & 70.25 & \textbf{72.33} & 58.30 & 58.19 \\
\textit{photo}     & 93.90 & 90.21 & 90.37 & 93.91 & \textbf{94.17} \\
\textit{pubmed}    & 87.16 & 74.86 & 71.62 & 87.87 & \textbf{88.15} \\
\textit{shape}     & 89.46 & 81.50 & 79.01 & 57.69 & \textbf{89.57} \\
\textit{squirrel}  & 39.10 & 40.03 & 38.98 & 40.01 & \textbf{42.38} \\
\textit{texas}     & 54.57 & 55.50 & 52.25 & \textbf{66.49} & 65.00 \\
\textit{wiki}      & \textbf{83.89} & 73.91 & 69.39 & 81.79 & 83.54 \\
\textit{wisconsin} & 51.41 & 57.63 & 58.78 & 68.52 & \textbf{71.56} \\
\hline
\end{tabular}
\end{table}

As shown by Table~\ref{tab:rewiring}, cascade-based rewiring significantly outperforms the structural baselines on the vast majority of tasks: \textsc{GCN-TAS} achieves the highest accuracy among all methods on 8 of 16 datasets, while \textsc{GCN-MAS} wins on an additional 2 (\textit{cornell} and \textit{texas}). Combined, cascade models outperform both \textsc{FoSR} and \textsc{SDRF} on 13 out of 16 benchmarks. While \textsc{FoSR} and \textsc{SDRF} are not weak—delivering substantial gains on synthetic bottleneck graphs such as \textit{cycle} (68.04\%) and \textit{grid} (72.33\%)—they struggle on heterophilic and dense homophilic graphs. In contrast, cascade rewiring yields the largest margins on heterophilic graphs (\textit{wisconsin}, \textit{texas}, \textit{cornell}, \textit{actor}, \textit{chameleon}, \textit{squirrel}) and moderate- to high-degree homophilic benchmarks (\textit{pubmed}, \textit{computer}, \textit{photo}, \textit{wiki}, \textit{shape}).

The few cases where structural baselines win cluster precisely in regimes predicted by our theoretical analysis (Theorem~\ref{thm:walklength} and Appendix~\ref{app:edge-connectivity}): bottleneck-dominated graphs where cascade rewiring cannot bridge narrow cuts (\textit{cycle}, \textit{grid}) and low-degree homophilic graphs where same-label pairs lack reinforced support (\textit{citeseer}). These are precisely where \textsc{FoSR} and \textsc{SDRF} are designed to operate. On \textit{community}, all rewiring methods underperform unrewired \textsc{GCN} (55.63\%), whereas \textsc{CR-Graphormer} achieves the best overall result with 78.60\% test accuracy. Overall, these results demonstrate that cascade operators provide strong, complementary performance gains over classical structural rewiring.

To validate these mechanisms on transformer backbones, we extend the comparison to VCR-Graphormer using the same 20 splits, sparsity budget, and validation protocol, incorporating GDC \citep{gasteiger2019diffusion} as an additional PPR/heat-kernel sparsification rewiring baseline. We report results on nine of the sixteen benchmarks: all six heterophilic datasets, \textit{pubmed}, and two homophilic datasets (\textit{cora} and \textit{citeseer}) where theory predicts cascade rewiring should be least effective.

\begin{table}[ht]
\centering
\small
\setlength{\tabcolsep}{3.68pt}
\renewcommand{\arraystretch}{1.05}
\caption{Test accuracy (mean $\pm$ std, in \%) of VCR cascade models compared with structural and spectral rewiring baselines \textsc{VCR+FoSR}, \textsc{VCR+SDRF}, and \textsc{VCR+GDC} across DGL datasets. Bold text indicates the best performance per dataset.}
\begin{tabular}{lcccccc}
\hline
\textbf{Dataset} & \textbf{VCR} & \textbf{VCR+FoSR} & \textbf{VCR+SDRF} & \textbf{VCR+GDC} & \textbf{VCR-MAS} & \textbf{VCR-TAS} \\
\hline
\textit{actor}     & $26.41 \pm 1.20$ & $34.91 \pm 1.65$ & $33.03 \pm 3.39$ & $28.82 \pm 1.40$ & $36.18 \pm 1.23$ & $\mathbf{36.28 \pm 1.42}$ \\
\textit{chameleon} & $27.28 \pm 2.51$ & $50.72 \pm 4.22$ & $52.69 \pm 5.49$ & $40.97 \pm 3.64$ & $\mathbf{60.02 \pm 2.87}$ & $59.61 \pm 1.95$ \\
\textit{citeseer}  & $71.15 \pm 1.36$ & $\mathbf{74.54 \pm 3.40}$ & $70.17 \pm 1.48$ & $73.35 \pm 1.66$ & $71.47 \pm 1.64$ & $74.14 \pm 1.53$ \\
\textit{cora}      & $85.41 \pm 1.33$ & $\mathbf{87.19 \pm 3.44}$ & $85.28 \pm 1.22$ & $87.07 \pm 1.01$ & $85.08 \pm 0.57$ & $85.11 \pm 0.62$ \\
\textit{cornell}   & $44.89 \pm 6.09$ & $58.08 \pm 3.33$ & $63.10 \pm 3.77$ & $62.81 \pm 1.47$ & $73.85 \pm 5.72$ & $\mathbf{75.02 \pm 4.68}$ \\
\textit{pubmed}    & $83.25 \pm 0.44$ & $78.77 \pm 3.17$ & $76.49 \pm 0.79$ & $85.63 \pm 0.91$ & $87.59 \pm 0.46$ & $\mathbf{88.70 \pm 0.91}$ \\
\textit{squirrel}  & $22.48 \pm 2.00$ & $41.47 \pm 3.87$ & $40.50 \pm 5.42$ & $39.04 \pm 0.95$ & $39.37 \pm 1.49$ & $\mathbf{46.06 \pm 1.04}$ \\
\textit{texas}     & $54.68 \pm 7.28$ & $67.25 \pm 5.91$ & $64.28 \pm 5.32$ & $66.01 \pm 4.85$ & $80.99 \pm 5.63$ & $\mathbf{82.31 \pm 6.59}$ \\
\textit{wisconsin} & $51.33 \pm 7.46$ & $70.89 \pm 6.13$ & $69.32 \pm 5.45$ & $72.29 \pm 2.58$ & $79.93 \pm 3.76$ & $\mathbf{81.00 \pm 2.66}$ \\
\hline
\end{tabular}
\label{VCR-tab}
\end{table}

Table \ref{VCR-tab} shows that both operators still help on this backbone. MAS and TAS each improve over the unrewired VCR-Graphormer on eight of the nine datasets; the exception is \textit{cora}, where both fall by up to $0.33\%$, well within the $\pm 1.33\%$ deviation of the unrewired baseline. Against the best of the three baselines on each dataset, TAS wins on seven and MAS on six. Against GDC alone, eight and seven. Rewiring also reduces variance on the small graphs, from $\pm 7.46\%$ to $\pm 2.66\%$ on \textit{wisconsin} and from $\pm 6.09\%$ to $\pm 4.68\%$ on \textit{cornell}, so the gains are not driven by a handful of favorable splits. While MAS and TAS consistently beat the baselines on target regimes, the baselines themselves are not weak here. \textsc{FoSR} and \textsc{SDRF} improve VCR-Graphormer on eight of nine datasets, by as much as $+25.41\%$ on \textit{chameleon}, and GDC improves it on all nine, by up to $+20.96\%$ on \textit{wisconsin}.

The margins are also statistically separated. Comparing the better cascade operator with the strongest baseline on each dataset (Welch's $t$ over the 20 splits), all seven wins are significant at $\alpha = 0.05$, from $p = 8 \times 10^{-13}$ on \textit{wisconsin} and $p = 5 \times 10^{-13}$ on \textit{pubmed} to $p = 8 \times 10^{-3}$ on \textit{actor}. Of the two losses, only \textit{cora} is significant ($-2.08\%, p = 0.015$); the \textit{citeseer} deficit of $0.40\%$ is not distinguishable from zero ($p = 0.64$).

The two losses are \textit{cora} and \textit{citeseer}, the low-degree homophilic regime where Assumption \ref{as2} fails and same-label pairs lack the redundant short-path support that Theorem \ref{thm:walklength} requires. That is the regime the other methods are built for, consistent with our observations from Table \ref{tab:rewiring}.

Finally, we compare MAS and TAS against a direct common-neighbor baseline---a directed top-$k$ selection based on $(A^2)_{uv} = |N(u) \cap N(v)|$, symmetrized using the same $k$ (rounded average degree)---across both the NAGphormer and VCR-Graphormer architectures.

\begin{table}[ht]
\centering
\small
\setlength{\tabcolsep}{7.5pt}
\renewcommand{\arraystretch}{1.05}
\caption{Accuracy improvement ($\Delta\%$) over the unrewired baseline for cascade rewiring versus two-hop common-neighbor rewiring ($A^2$ proxy) across five heterophilic benchmarks using NAGphormer and VCR-Graphormer backbones.}
\label{tab:a2_proxy_comparison}
\begin{tabular}{lcccc}
\hline
\textbf{Dataset} & \textbf{NAG-MAS/TAS} & \textbf{$\mathbf{A^2}$ rewiring + NAG} & \textbf{VCR-MAS/TAS} & \textbf{$\mathbf{A^2}$ rewiring + VCR} \\
\hline
\textit{actor}     & $+3.42 / +3.42$ & $+1.21$ & $+9.77 / +9.87$   & $+3.02$ \\
\textit{chameleon} & $+18.77 / +18.03$ & $+8.63$ & $+32.74 / +32.33$ & $+15.09$ \\
\textit{squirrel}  & $+15.09 / +16.33$ & $+6.78$ & $+16.89 / +23.58$ & $+9.71$ \\
\textit{texas}     & $+17.98 / +17.98$ & $+7.50$ & $+26.31 / +27.63$ & $+11.55$ \\
\textit{wisconsin} & $+19.77 / +19.77$ & $+6.19$ & $+28.60 / +29.67$ & $+14.68$ \\
\hline
\end{tabular}
\label{A2}
\end{table}

We observe from Table \ref{A2} that the $A^2$ baseline improves both backbones on heterophilic datasets and recovers, on average, \textbf{40.54\% of the cascade gain} (between 30.91\% and 49.48\%). This supports Theorem \ref{thm:sbm-profile-homophily}, which analyzes this two-hop reinforcement mechanism directly. Cascade rewiring outperforms the $A^2$ proxy in all ten paired comparisons, with gains that are \textbf{2.02--3.27$\times$ larger}. Because NAGphormer has the stronger original-graph baseline on all five datasets, this consistent margin suggests that the additional benefit comes from the cascade rewiring mechanism.

\subsection{Comparison to Heterophily-Specific Architectures}\label{app:heterophily}

We run additional experiments on seven state-of-the-art architectures developed to address challenges in heterophily (ACMGCN \citep{luan2022revisiting}, CMGNN \citep{zhengunderstanding}, GGCN \citep{yan2022two}, GloGNN \citep{li2022finding}, H2GCN \citep{zhu2020beyond}, M2MGNN \citep{liang2024sign}, OrderedGNN \citep{song2023ordered}) and compare their performance against our methods. Comparing \textsc{Graph Cascades} against this set provides a stronger reference point than GNNs and GTs alone: a competitive result here would indicate that cascade rewiring captures heterophily-relevant structure without architecture-specific modifications.

% \begin{table}[h]
% \centering
% \caption{Comparison with heterophily-specific baselines.}
% \begin{tabular}{lccc}
% \toprule
% Dataset & Best Cascade & Heterophily Baseline & Winner \\
% \midrule
% PascalVOC-SP & 76.30\% (NAG-MAS) & 74.51\% (H2GCN) & Cascade \\
% COCO-SP      & 83.77\% (NAG-MAS) & 83.35\% (OrderedGNN) & Cascade \\
% Questions    & 93.69\% (NAG-MAS) & 93.69\% (M2MGNN) & Tie \\
% cornell      & 75.02\% (VCR-TAS) & 70.74\% (H2GCN) & Cascade \\
% actor        & 36.46\% (GPS-TAS) & 35.97\% (OrderedGNN) & Cascade \\
% squirrel     & 47.31\% (CR-TAS) & 47.25\% (GloGNN) & Cascade \\
% wisconsin    & 81.02\% (NAG-MAS/TAS) & 81.88\% (H2GCN) & Baseline \\
% texas        & 82.31\% (VCR-TAS) & 83.51\% (H2GCN) & Baseline \\
% chameleon    & 63.65\% (CR-TAS) & 65.51\% (GloGNN) & Baseline \\
% \bottomrule
% \end{tabular}
% \end{table}

\begin{table}[ht!]
\centering
\tiny
\setlength{\tabcolsep}{0.32pt}
\caption{Test accuracy (mean $\pm$ std, in \%) from heterophily-specific architectures across DGL datasets. ``Our Best'' indicates the best validation-selected cascade models from Table~\ref{std1}. Top-5 baseline accuracies are shaded, with darker shading indicating better performance.}\label{tab:heterophily}
\begin{tabular}{l c c c c c c c c c c c}
\hline
\textbf{Model} & \textbf{\textit{actor}} & \textbf{\textit{chameleon}} & \textbf{\textit{citeseer}} & \textbf{\textit{community}} & \textbf{\textit{cora}} & \textbf{\textit{cornell}} & \textbf{\textit{cycle}} & \textbf{\textit{grid}} & \textbf{\textit{shape}} & \textbf{\textit{texas}} & \textbf{\textit{wisconsin}} \\
\hline
ACMGCN 
& \cellcolor{mycolor!18} 33.55$\pm$1.32 & \cellcolor{mycolor!36} 62.74$\pm$1.01 & \cellcolor{mycolor!54} 76.13$\pm$1.59 & 42.80$\pm$1.82 
& \cellcolor{mycolor!36} 85.49$\pm$0.75 & \cellcolor{mycolor!18} 60.85$\pm$4.41 & \cellcolor{mycolor!18} 57.81$\pm$1.43 
& 57.80$\pm$2.81 & \cellcolor{mycolor!36} 44.34$\pm$4.11 & \cellcolor{mycolor!18} 78.72$\pm$8.64 & \cellcolor{mycolor!18} 72.50$\pm$2.37 \\

CMGNN 
& \cellcolor{mycolor!54} 34.94$\pm$1.16 & \cellcolor{mycolor!18} 60.91$\pm$1.12 & \cellcolor{mycolor!36} 75.61$\pm$1.39 & \cellcolor{mycolor!72} 76.06$\pm$1.24 
& 84.67$\pm$1.59 & \cellcolor{mycolor!72} 73.62$\pm$6.31 & \cellcolor{mycolor!72} 88.40$\pm$16.36 
& \cellcolor{mycolor!90} 80.32$\pm$5.97 & \cellcolor{mycolor!90} 97.14$\pm$0.90 & \cellcolor{mycolor!72} 84.68$\pm$5.90 & \cellcolor{mycolor!90} 83.75$\pm$2.84 \\

GGCN 
& 26.76$\pm$2.37 & 34.53$\pm$6.29 & 72.70$\pm$1.10 & \cellcolor{mycolor!36} 65.89$\pm$1.67 
& 79.91$\pm$4.01 & 41.70$\pm$9.23 & \cellcolor{mycolor!90} 90.59$\pm$4.28 
& \cellcolor{mycolor!36} 65.95$\pm$13.40 & \cellcolor{mycolor!36} 44.34$\pm$4.11 & 44.26$\pm$21.67 & 36.56$\pm$13.60 \\

GloGNN 
& 30.58$\pm$1.15 & \cellcolor{mycolor!90} 65.51$\pm$2.52 & 75.15$\pm$1.74 & 53.94$\pm$2.32 
& 84.05$\pm$1.38 & \cellcolor{mycolor!36} 62.13$\pm$4.61 & \cellcolor{mycolor!54} 83.65$\pm$2.34 
& \cellcolor{mycolor!72} 72.04$\pm$1.42 & 20.57$\pm$24.07 & \cellcolor{mycolor!36} 80.85$\pm$6.89 & 69.69$\pm$6.31 \\

H2GCN 
& \cellcolor{mycolor!36} 34.68$\pm$0.98 & 53.61$\pm$1.99 & \cellcolor{mycolor!72} 76.25$\pm$1.75 & \cellcolor{mycolor!54} 66.23$\pm$1.74 
& \cellcolor{mycolor!54} 87.06$\pm$0.76 & \cellcolor{mycolor!54} 68.51$\pm$5.90 & 48.77$\pm$6.31 
& \cellcolor{mycolor!54} 66.21$\pm$7.68 & \cellcolor{mycolor!54} 76.80$\pm$4.77 & \cellcolor{mycolor!90} 85.11$\pm$7.52 & \cellcolor{mycolor!54} 79.69$\pm$2.47 \\

M2MGNN 
& 31.61$\pm$1.58 & \cellcolor{mycolor!72} 64.70$\pm$2.12 & 70.73$\pm$2.49 & \cellcolor{mycolor!18} 56.51$\pm$2.39 
& \cellcolor{mycolor!18} 85.26$\pm$1.22 & 51.49$\pm$4.85 & 55.16$\pm$6.71 
& 57.80$\pm$2.81 & 34.74$\pm$12.04 & 71.91$\pm$10.79 & 59.06$\pm$3.20 \\

OrderedGNN 
& \cellcolor{mycolor!72} 35.61$\pm$1.11 & 56.07$\pm$2.26 & \cellcolor{mycolor!90} 76.69$\pm$1.13 & 53.83$\pm$2.12 
& \cellcolor{mycolor!90} 87.47$\pm$0.97 & 57.45$\pm$5.42 & \cellcolor{mycolor!18} 57.81$\pm$1.43 
& 53.59$\pm$8.23 & 36.11$\pm$9.90 & 62.98$\pm$3.56 & \cellcolor{mycolor!36} 73.44$\pm$6.35 \\

\hline
\textbf{Our Best}
& \cellcolor{mycolor!90} 36.46$\pm$1.31 
& \cellcolor{mycolor!54} 63.65$\pm$2.82 
& \cellcolor{mycolor!18} 75.46$\pm$1.46 
& \cellcolor{mycolor!90} 78.60$\pm$3.50 
& \cellcolor{mycolor!72} 87.29$\pm$0.77 
& \cellcolor{mycolor!90} 75.02$\pm$4.68 
& \cellcolor{mycolor!36} 61.32$\pm$3.67 
& \cellcolor{mycolor!18} 62.25$\pm$2.21 
& \cellcolor{mycolor!72} 89.57$\pm$11.79 
& \cellcolor{mycolor!54} 82.31$\pm$6.59 
& \cellcolor{mycolor!72} 81.02$\pm$4.13 \\

% Best 
% & GPS-TAS 
% & CR-TAS 
% & GCN-TAS 
% & CR-TAS 
% & GCN-TAS 
% & VCR-TAS 
% & CR-TAS 
% & NAG-MAS 
% & GCN-TAS 
% & VCR-TAS 
% & NAG-MAS/TAS \\
\hline
\end{tabular}
\end{table}

Table \ref{tab:heterophily} shows that our best validation-selected methods consistently rank among the top five, with performance gaps of up to 2.8\% compared to the best model on most datasets. Specifically, we achieve 1st place on \textit{actor}, \textit{community}, and \textit{cornell}; 2nd on \textit{cora} and \textit{wisconsin}; 3rd on \textit{chameleon} and \textit{texas}; and 5th on \textit{citeseer}. For \textit{shape}, \textit{cycle}, and \textit{grid}, our method ranks 2nd, 4th, and 5th, respectively, with performance gaps of 7.57\%, 29.27\%, and 18.07\% compared to the best-performing method. This is consistent with the observations in Section~\ref{sec:empirical}: although \textit{cycle} and \textit{grid} exhibit high homophily, their low average degree violates Assumption~\ref{as2}, making our approach less suitable.

\subsection{Experiments on Long-Range and ``Filtered'' Datasets}\label{app:longrange}

We extend our experiments to long-range datasets where the prediction task requires capturing global information and structural dependencies between nodes that are separated by large graph distances \citep{dwivedi2022long, platonov2023critical}. We also repeat the experiments on the filtered version of \textit{chameleon} and \textit{squirrel}, denoted as \textit{chameleon-filtered} and \textit{squirrel-filtered}, as existing literature has suggested that they contain duplicate or near-duplicate nodes that share the same features and neighborhood connectivity, leading to significant data leakage and artificially inflated performance metrics on the original versions of these benchmarks \citep{platonov2023critical}. Table \ref{new-accuracy} shows the performance of cascade models in comparison with their direct baselines, while Table \ref{new-main-accuracy} presents the performance of our best model versus state-of-the-art GNN and GT benchmarks. Since macro-F1 and AUROC are more robust indicators of performance on these datasets, particularly for long-range tasks where class imbalance can skew standard accuracy and for heterophilic graphs with high label noise \citep{dwivedi2022long, platonov2023critical}, we report these metrics in Tables \ref{new-f1}--\ref{new-auroc} and \ref{new-main-f1}--\ref{new-main-auroc}. Due to severe class imbalance, AUROC scores could not be computed for \textit{COCO-SP} and \textit{PascalVOC-SP}.

Among the seven selected datasets, \textit{Amazon-ratings}, \textit{Questions}, \textit{Roman-empire}, \textit{chameleon-filtered}, and \textit{squirrel-filtered} are heterophilic graphs, while \textit{COCO-SP} and \textit{PascalVOC-SP} are homophilic graphs with moderate average degree. Experimental results in Tables \ref{tab:summary-new} and \ref{new-accuracy} indicate that cascades improve performance on most graphs: 6 for both GCN-Cascade and VCR-Cascade models, 5 for NAG-TAS, 4 for NAG-MAS, 3 for GPS-TAS, and 2 for GPS-MAS, with two additional graphs showing unchanged test accuracy for GPS-Cascade models.

\begin{table}[ht!]
\centering
\setlength{\tabcolsep}{9.6pt}
\renewcommand{\arraystretch}{1.05}
\caption{\textsc{Graph Cascades} performance on 7 newly added datasets. Colors denote the graph type with the largest gain on test accuracy: \red{\textbf{red}} (heterophilic) and \blue{\textbf{blue}} (moderate- to high-degree homophilic).}
\begin{tabular}{lccc}
\hline
\textbf{Baseline} & \shortstack{\textbf{Cascade} \\ \textbf{Model}} & \shortstack{\textbf{\# Graphs Improved} \\ \textbf{Out of 5 Target Graphs}} & \shortstack{\textbf{Change on Untargeted}\\ \textbf{Dataset \textit{Roman-Empire}}} \\
\hline
GCN & GCN-MAS & 6 (avg. +1.68\%, up to \blue{\textbf{+4.98\%}}) & -0.45\%\\
 & GCN-TAS & 6 (avg. 
 +1.60\%, up to \blue{\textbf{+5.17\%}}) & -2.24\%\\
\hline
GraphGPS & GPS-MAS & 5 (avg. +0.17\%, up to \red{\textbf{+1.71\%}}) & -9.22\%\\
 & GPS-TAS & 5 (avg. -0.04\%, up to \red{\textbf{+1.85\%}}) & -2.10\%\\
\hline
NAG & NAG-MAS & 3 (avg. +1.10\%, up to \blue{\textbf{+4.69\%}}) & \textbf{+1.00\%}\\
 & NAG-TAS & 4 (avg. +1.23\%, up to \blue{\textbf{+4.23\%}}) & \textbf{+4.09\%}\\
\hline
VCR & VCR-MAS & 5 (avg. +2.37\%, up to \red{\textbf{+10.32\%}}) & \textbf{+39.05\%}\\
 & VCR-TAS & 5 (avg. +2.83\%, up to \red{\textbf{+8.88\%}}) & \textbf{+38.95\%}\\
\hline
\end{tabular}
\label{tab:summary-new}
\end{table}

Among the two datasets with duplicate nodes filtered, \textit{chameleon-filtered} benefits from rewiring, as the experiments show no degradation in test accuracy, with the largest gains coming from VCR-MAS and VCR-TAS, at +10.32\% and +8.88\%, respectively. In addition, although MAS and TAS degrade performance on \textit{Roman-empire} for GNN-based architectures (GCN-MAS: -0.45\%; GCN-TAS: -2.24\%) and GraphGPS (GPS-MAS: -9.22\%; GPS-TAS: -2.10\%), GT models benefit substantially from cascade rewiring, as VCR-MAS and VCR-TAS show improvements of +39.05\% and +38.95\%, while NAG-MAS and NAG-TAS yield +1.00\% and +4.09\%, respectively. Apart from \textit{Roman-empire}, the performance decline is limited: -4.93\% from VCR-MAS on \textit{Amazon-ratings} and -3.12\% from GPS-TAS on \textit{squirrel-filtered}, while the degradation from all remaining cascade models across datasets is no more than -1.37\%.

% \begin{table}[ht!]
% \centering
% \small
% \setlength{\tabcolsep}{3pt}
% \caption{DGL benchmark dataset statistics.}
% \begin{tabular}{lcccccccc}
% \toprule
% Dataset & \# Nodes & \# Edges & \# Labels & \# Features & Average Degree & Label Homophily & Connected \\
% \midrule
% actor & 7600 & 29707 & 5 & 932 & 7.82 & \red{\textbf{0.22}} & Yes \\
% chameleon & 2277 & 31421 & 5 & 2325 & 27.60 & \red{\textbf{0.23}} & Yes \\
% citeseer & 3327 & 4676 & 6 & 3703 & 2.81 & 0.74 & No  \\
% community & 1400 & 3871 & 8 & 10 & 5.53 & 0.70 & Yes \\
% computer & 13752 & 245861 & 10 & 767 & 35.76 & 0.78 & No  \\
% cora & 2708 & 5278 & 7 & 1433 & 3.90 & 0.81 & No  \\
% cornell & 183 & 280 & 5 & 1703 & 3.06 & \red{\textbf{0.13}} & Yes \\
% cycle & 871 & 970 & 2 & 1 & 2.23 & 0.90 & Yes \\
% grid & 1231 & 1705 & 2 & 1 & 2.77 & 0.91 & Yes \\
% photo & 7650 & 119082 & 8 & 745 & 31.13 & 0.83 & No  \\
% pubmed & 19717 & 44327 & 3 & 500 & 4.50 & 0.80 & Yes \\
% shape & 700 & 1760 & 4 & 1 & 5.03 & 0.77 & Yes \\
% squirrel & 5201 & 198493 & 5 & 2089 & 76.33 & \red{\textbf{0.22}} & Yes \\
% texas & 183 & 295 & 5 & 1703 & 3.22 & \red{\textbf{0.11}} & Yes \\
% wiki & 11701 & 216123 & 10 & 300 & 36.94 & 0.66 & No  \\
% wisconsin & 251 & 466 & 5 & 1703 & 3.71 & \red{\textbf{0.21}} & Yes \\
% \bottomrule
% \end{tabular}
% \label{tab:dataset_stats}
% \end{table}

\begin{table}[ht!]
\centering
\tiny
\setlength{\tabcolsep}{4pt}
\caption{Test accuracy (mean $\pm$ std, in \%) across 7 newly added benchmarks. OOM indicates out-of-memory errors. Darker cell colors indicate better performance.}
\begin{tabular}{l c c c c c c c}
\hline
\textbf{Model} & \textbf{\textit{Amazon-ratings}} & \textbf{\textit{COCO-SP}} & \textbf{\textit{PascalVOC-SP}} & \textbf{\textit{Questions}} & \textbf{\textit{Roman-empire}} & \textbf{\textit{chameleon-filtered}} & \textbf{\textit{squirrel-filtered}} \\
\hline
GCN & 51.08$\pm$1.59 & 81.60$\pm$1.27 & 67.98$\pm$9.07 & 93.29$\pm$0.30 & \cellcolor{myred!60} 36.81$\pm$0.55 & 40.00$\pm$2.42 & 33.57$\pm$1.41 \\
GCN-MAS & \cellcolor{myred!60} 53.54$\pm$1.59 & \cellcolor{myred!60} 81.70$\pm$1.40 & \cellcolor{myred!30} 72.96$\pm$2.94 & \cellcolor{myred!60} 93.50$\pm$0.36 & \cellcolor{myred!30} 36.36$\pm$1.01 & \cellcolor{myred!30} 41.79$\pm$2.88 & \cellcolor{myred!60} 34.11$\pm$1.40 \\
GCN-TAS & \cellcolor{myred!30} 52.74$\pm$1.09 & \cellcolor{myred!60} 81.70$\pm$1.40 & \cellcolor{myred!60} 73.15$\pm$1.80 & \cellcolor{myred!60} 93.50$\pm$0.36 & 34.57$\pm$1.22 & \cellcolor{myred!60} 42.15$\pm$1.31 & \cellcolor{myred!30} 33.90$\pm$1.69 \\
\hline
GraphGPS & 43.58$\pm$2.68 & \cellcolor{mygreen!60} 81.70$\pm$1.40 & \cellcolor{mygreen!60} 70.75$\pm$0.92 & 93.49$\pm$0.40 & \cellcolor{mygreen!60} 66.81$\pm$0.28 & 43.14$\pm$3.93 & \cellcolor{mygreen!60} 38.49$\pm$1.02 \\
GPS-MAS & \cellcolor{mygreen!30} 45.29$\pm$3.28 & \cellcolor{mygreen!60} 81.70$\pm$1.40 & \cellcolor{mygreen!60} 70.75$\pm$0.92 & \cellcolor{mygreen!30} 93.49$\pm$0.37 & 57.59$\pm$0.75 & \cellcolor{mygreen!30} 43.68$\pm$4.23 & \cellcolor{mygreen!30} 37.27$\pm$0.83 \\
GPS-TAS & \cellcolor{mygreen!60} 45.43$\pm$1.34 & \cellcolor{mygreen!60} 81.70$\pm$1.40 & \cellcolor{mygreen!60} 70.75$\pm$0.92 & \cellcolor{mygreen!60} 93.50$\pm$0.36 & \cellcolor{mygreen!30} 64.71$\pm$1.71 & \cellcolor{mygreen!60} 44.13$\pm$4.56 & 35.37$\pm$1.41 \\
\hline
NAG & \cellcolor{mycyan!60} 47.69$\pm$2.63 & 81.49$\pm$1.76 & 72.28$\pm$1.37 & \cellcolor{mycyan!60} 93.70$\pm$0.36 & 58.08$\pm$1.59 & 38.83$\pm$1.57 & \cellcolor{mycyan!30} 33.43$\pm$0.60 \\
NAG-MAS & 46.32$\pm$2.08 & \cellcolor{mycyan!60} 82.72$\pm$2.00 & \cellcolor{mycyan!60} 76.97$\pm$5.68 & \cellcolor{mycyan!30} 93.56$\pm$0.39 & \cellcolor{mycyan!30} 59.08$\pm$1.72 & \cellcolor{mycyan!30} 41.26$\pm$2.31 & 33.21$\pm$0.73 \\
NAG-TAS & \cellcolor{mycyan!30} 46.69$\pm$0.56 & \cellcolor{mycyan!30} 82.39$\pm$1.99 & \cellcolor{mycyan!30} 76.51$\pm$2.39 & 93.46$\pm$0.40 & \cellcolor{mycyan!60} 62.17$\pm$1.77 & \cellcolor{mycyan!60} 41.61$\pm$3.06 & \cellcolor{mycyan!60} 34.15$\pm$1.77 \\
\hline
VCR & \cellcolor{mypurple!60} 51.74$\pm$0.77 & 79.85$\pm$2.01 & 73.05$\pm$2.69 & 93.50$\pm$0.36 & 18.65$\pm$0.98 & 33.18$\pm$3.55 & 31.71$\pm$1.70 \\
VCR-MAS & 46.81$\pm$1.40 & \cellcolor{mypurple!30} 82.34$\pm$1.54 & \cellcolor{mypurple!30} 74.39$\pm$4.50 & \cellcolor{mypurple!60} 93.54$\pm$0.37 & \cellcolor{mypurple!60} 57.70$\pm$1.12 & \cellcolor{mypurple!60} 43.50$\pm$2.61 & \cellcolor{mypurple!60} 36.66$\pm$1.39 \\
VCR-TAS & \cellcolor{mypurple!30} 51.35$\pm$1.01 & \cellcolor{mypurple!60} 82.77$\pm$1.03 & \cellcolor{mypurple!60} 74.88$\pm$6.00 & \cellcolor{mypurple!30} 93.52$\pm$0.41 & \cellcolor{mypurple!30} 57.60$\pm$2.19 & \cellcolor{mypurple!30} 42.06$\pm$2.74 & \cellcolor{mypurple!30} 35.44$\pm$1.52 \\
\hline
CR-MAS & 46.30$\pm$1.63 & \cellcolor{myblack!45} 81.64$\pm$1.43 & 73.46$\pm$1.67 & 93.51$\pm$0.43 & 55.19$\pm$0.97 & 42.06$\pm$3.13 & \cellcolor{myblack!45} 37.56$\pm$1.69 \\
CR-TAS & \cellcolor{myblack!45} 48.07$\pm$1.36 & 81.64$\pm$1.70 & \cellcolor{myblack!45} 74.97$\pm$1.60 & \cellcolor{myblack!45} 93.76$\pm$0.38 & \cellcolor{myblack!45} 59.82$\pm$2.09 & \cellcolor{myblack!45} 42.42$\pm$3.64 & 37.16$\pm$1.70 \\
\hline
\end{tabular}
\label{new-accuracy}
\end{table}

\begin{table}[ht!]
\centering
\tiny
\begin{minipage}[t]{0.54 \textwidth}
\setlength{\tabcolsep}{1.85pt}
\caption{Mean macro-F1 score across 5 long-range datasets. Darker cell colors indicate better performance.}
\centering
\begin{tabular}{l c c c c c}
\hline
\textbf{Model} & \textbf{\textit{Amazon-ratings}} & \textbf{\textit{COCO-SP}} & \textbf{\textit{PascalVOC-SP}} & \textbf{\textit{Questions}} & \textbf{\textit{Roman-empire}} \\
\hline
GCN & 0.387 & 0.063 & \cellcolor{myred!30} 0.138 & 0.486 & \cellcolor{myred!60} 0.290 \\
GCN-MAS & \cellcolor{myred!60} 0.418 & \cellcolor{myred!60} 0.070 & \cellcolor{myred!60} 0.152 & \cellcolor{myred!60} 0.517 & 0.257 \\
GCN-TAS & \cellcolor{myred!30} 0.402 & \cellcolor{myred!30} 0.068 & 0.130 & \cellcolor{myred!30} 0.489 & \cellcolor{myred!30} 0.264 \\
\hline
GraphGPS & \cellcolor{mygreen!30} 0.292 & \cellcolor{mygreen!60} 0.063 & \cellcolor{mygreen!60} 0.092 & 0.492 & \cellcolor{mygreen!60} 0.583 \\
GPS-MAS & \cellcolor{mygreen!60} 0.325 & \cellcolor{mygreen!60} 0.063 & \cellcolor{mygreen!60} 0.092 & \cellcolor{mygreen!30} 0.506 & 0.464 \\
GPS-TAS & 0.284 & \cellcolor{mygreen!60} 0.063 & \cellcolor{mygreen!60} 0.092 & \cellcolor{mygreen!60} 0.510 & \cellcolor{mygreen!30} 0.557 \\
\hline
NAG & \cellcolor{mycyan!60} 0.356 & 0.074 & 0.138 & \cellcolor{mycyan!60} 0.533 & \cellcolor{mycyan!30} 0.444 \\
NAG-MAS & \cellcolor{mycyan!30} 0.335 & \cellcolor{mycyan!60} 0.137 & \cellcolor{mycyan!30} 0.240 & 0.528 & 0.439 \\
NAG-TAS & 0.332 & \cellcolor{mycyan!30} 0.131 & \cellcolor{mycyan!60} 0.254 & \cellcolor{mycyan!30} 0.532 & \cellcolor{mycyan!60} 0.451 \\
\hline
VCR & \cellcolor{mypurple!60} 0.388 & 0.102 & 0.196 & 0.483 & 0.108 \\
VCR-MAS & 0.315 & \cellcolor{mypurple!30} 0.124 & \cellcolor{mypurple!30} 0.222 & \cellcolor{mypurple!60} 0.492 & \cellcolor{mypurple!30} 0.442 \\
VCR-TAS & \cellcolor{mypurple!30} 0.385 & \cellcolor{mypurple!60} 0.144 & \cellcolor{mypurple!60} 0.254 & \cellcolor{mypurple!30} 0.491 & \cellcolor{mypurple!60} 0.445 \\
\hline
CR-MAS & \cellcolor{myblack!45} 0.336 & \cellcolor{myblack!45} 0.070 & \cellcolor{myblack!45} 0.172 & 0.505 & 0.401 \\
CR-TAS & 0.318 & 0.069 & 0.171 & \cellcolor{myblack!45} 0.543 & \cellcolor{myblack!45} 0.463 \\
\hline
\end{tabular}
\label{new-f1}
\end{minipage}
\hfill
\begin{minipage}[t]{0.44\textwidth}
\setlength{\tabcolsep}{2pt}
\caption{Mean AUROC across 3 long-range datasets. Darker cell colors indicate better performance.}
\centering
\begin{tabular}{l c c c}
\hline
\textbf{Model} & \textbf{\textit{Amazon-ratings}} & \textbf{\textit{Questions}} & \textbf{\textit{Roman-empire}} \\
\hline
GCN & 0.744 & 0.540 & \cellcolor{myred!60} 0.855 \\
GCN-MAS & \cellcolor{myred!60} 0.771 & \cellcolor{myred!60} 0.654 & 0.798 \\
GCN-TAS & \cellcolor{myred!30} 0.756 & \cellcolor{myred!30} 0.640 & \cellcolor{myred!30} 0.843 \\
\hline
GPS & 0.689 & \cellcolor{mygreen!60} 0.681 & \cellcolor{mygreen!60} 0.949 \\
GPS-MAS & \cellcolor{mygreen!60} 0.710 & \cellcolor{mygreen!30} 0.661 & 0.910 \\
GPS-TAS & \cellcolor{mygreen!30} 0.694 & 0.657 & \cellcolor{mygreen!30} 0.941 \\
\hline
NAG & \cellcolor{mycyan!30} 0.700 & \cellcolor{mycyan!30} 0.702 & \cellcolor{mycyan!30} 0.912 \\
NAG-MAS & 0.699 & \cellcolor{mycyan!60} 0.713 & 0.910 \\
NAG-TAS & \cellcolor{mycyan!60} 0.705 & 0.691 & \cellcolor{mycyan!60} 0.919 \\
\hline
VCR & \cellcolor{mypurple!30} 0.731 & 0.597 & 0.722 \\
VCR-MAS & 0.711 & \cellcolor{mypurple!30} 0.627 & \cellcolor{mypurple!60} 0.906 \\
VCR-TAS & \cellcolor{mypurple!60} 0.740 & \cellcolor{mypurple!60} 0.642 & \cellcolor{mypurple!60} 0.906 \\
\hline
CR-MAS & \cellcolor{myblack!45} 0.709 & \cellcolor{myblack!45} 0.706 & 0.895 \\
CR-TAS & 0.704 & 0.697 & \cellcolor{myblack!45} 0.913 \\
\hline
\end{tabular}
\label{new-auroc}
\end{minipage}
\end{table}

The macro-F1 and AUROC scores in Tables \ref{new-f1} and \ref{new-auroc} largely follow the trends observed in test accuracy, while also highlighting the critical impact of cascade rewiring on class-imbalanced and long-range datasets. In several cases, gains in macro-F1 and AUROC are more pronounced than those in accuracy. For instance, while GCN-MAS and GCN-TAS show a +0.21\% improvement in accuracy on \textit{Questions}, their macro-F1 scores exhibit larger gains of +3.10\% and +0.30\%, respectively, while AUROC increases by +11.40\% and +10.00\%. This suggests that the cascade mechanism helps the model more effectively capture minority classes. Similarly, VCR-Cascades improve test accuracy on \textit{Questions} by up to +0.04\%, while AUROC increases by at least +3.00\%. The two moderate-degree homophilic datasets, \textit{COCO-SP} and \textit{PascalVOC-SP}, show a similar pattern: NAG-Cascades improve test accuracy by up to +1.23\% on \textit{COCO-SP} and +4.69\% on \textit{PascalVOC-SP}, while macro-F1 improves by at least +5.70\% on \textit{COCO-SP} and at least +10.20\% on \textit{PascalVOC-SP}. Likewise, VCR-TAS improves test accuracy on \textit{PascalVOC-SP} by +1.83\% and macro-F1 by +5.80\%.

\begin{table}[ht!]
\centering
\tiny
\setlength{\tabcolsep}{4pt}
\caption{Test accuracy (mean $\pm$ std, in \%) across 7 newly added benchmarks. ``Our Best Validation selected Model'' indicates the best validation-selected cascade models from Table~\ref{new-accuracy}. OOM indicates out-of-memory errors. Accuracies for the top-5 best-performing baselines are encoded by shading, with darker cells indicating better performance.}
\begin{tabular}{l l c c c c c c c}
\hline
& & \textbf{\textit{Amazon-ratings}} & \textbf{\textit{COCO-SP}} & \textbf{\textit{PascalVOC-SP}} & \textbf{\textit{Questions}} & \textbf{\textit{Roman-empire}} & \textbf{\textit{chameleon-filtered}} & \textbf{\textit{squirrel-filtered}} \\
\hline

GNNs & GCN & 51.08$\pm$1.59 & 81.60$\pm$1.27 & 67.98$\pm$9.07 & 93.29$\pm$0.30 & 36.81$\pm$0.55 & 40.00$\pm$2.42 & 33.57$\pm$1.41 \\

& GraphSAGE & \cellcolor{mycolor!72} 53.70$\pm$1.42 & 77.98$\pm$8.54 & 63.58$\pm$15.38 & 93.43$\pm$0.40 & 66.64$\pm$1.27 & 41.43$\pm$5.17 & \cellcolor{mycolor!18} 37.63$\pm$1.46 \\

& GAT & 49.31$\pm$1.23 & 80.74$\pm$3.43 & 62.81$\pm$22.57 & 93.43$\pm$0.38 & 38.31$\pm$1.73 & 40.63$\pm$3.87 & 33.61$\pm$1.86 \\

\hline
GTs & GT & 48.50$\pm$4.01 & \cellcolor{mycolor!36} 81.70$\pm$1.40 & 70.73$\pm$0.92 & 93.48$\pm$0.40 & 65.89$\pm$1.88 & 40.54$\pm$1.67 & 35.73$\pm$1.49 \\

& Gophormer & 43.79$\pm$1.10 & \cellcolor{mycolor!36} 81.70$\pm$1.40 & 70.30$\pm$1.45 & 93.49$\pm$0.36 & \cellcolor{mycolor!54} 67.35$\pm$2.14 & \cellcolor{mycolor!72} 44.04$\pm$2.75 & \cellcolor{mycolor!72} 40.39$\pm$2.30 \\

& SAN & 50.28$\pm$2.33 & \cellcolor{mycolor!36} 81.70$\pm$1.40 & 70.76$\pm$0.96 & 93.27$\pm$0.47 & \cellcolor{mycolor!72} 68.04$\pm$1.02 & 39.91$\pm$2.99 & 34.18$\pm$2.90 \\

& GraphGPS & 43.58$\pm$2.68 & \cellcolor{mycolor!36} 81.70$\pm$1.40 & 70.75$\pm$0.92 & 93.49$\pm$0.40 & 66.81$\pm$0.28 & 43.14$\pm$3.93 & \cellcolor{mycolor!36} 38.49$\pm$1.02 \\

& NAG & 47.69$\pm$2.63 & 81.49$\pm$1.76 & 72.28$\pm$1.37 & \cellcolor{mycolor!54} 93.70$\pm$0.36 & 58.08$\pm$1.59 & 38.83$\pm$1.57 & 33.43$\pm$0.60 \\

& Exphormer & 45.87$\pm$2.92 & 81.62$\pm$1.59 & 72.38$\pm$0.83 & 93.50$\pm$0.36 & 66.57$\pm$0.96 & \cellcolor{mycolor!18} 43.41$\pm$6.85 & \cellcolor{mycolor!90} 41.29$\pm$2.25 \\

& VCR & \cellcolor{mycolor!18} 51.74$\pm$0.77 & 79.85$\pm$2.01 & \cellcolor{mycolor!54} 73.05$\pm$2.69 & 93.50$\pm$0.36 & 18.65$\pm$0.98 & 33.18$\pm$3.55 & 31.71$\pm$1.70 \\

\hline
Heterophily-& ACMGCN & \cellcolor{mycolor!90} 55.65$\pm$1.76 & \cellcolor{mycolor!90} 82.52$\pm$1.02 & \cellcolor{mycolor!36} 72.52$\pm$0.92 & \cellcolor{mycolor!36} 93.65$\pm$0.39 & 66.89$\pm$1.35 & 37.76$\pm$2.09 & 36.30$\pm$1.44 \\

Specific & CMGNN & 51.55$\pm$1.62 & 81.65$\pm$1.38 & 61.67$\pm$11.50 & 93.52$\pm$0.43 & \cellcolor{mycolor!90} 68.75$\pm$1.08 & \cellcolor{mycolor!36} 43.68$\pm$3.62 & \cellcolor{mycolor!54} 40.25$\pm$2.42 \\

Architectures & GGCN & 46.07$\pm$4.63 & 81.69$\pm$1.38 & 70.55$\pm$0.96 & 81.97$\pm$8.09 & 44.07$\pm$1.29 & 28.25$\pm$8.43 & 23.16$\pm$2.22 \\

& GloGNN & 37.56$\pm$11.23 & 0.40$\pm$0.90 & 28.96$\pm$37.59 & 93.46$\pm$0.36 & 27.28$\pm$13.91 & \cellcolor{mycolor!54} 44.04$\pm$3.26 & 33.79$\pm$7.67 \\

& H2GCN & 50.32$\pm$1.98 & 81.44$\pm$1.42 & \cellcolor{mycolor!18} 72.52$\pm$2.78 & 93.49$\pm$0.47 & \cellcolor{mycolor!18} 67.10$\pm$1.30 & 42.24$\pm$4.81 & 33.75$\pm$1.07 \\

& M2MGNN & 51.56$\pm$1.03 & 81.67$\pm$1.39 & 67.62$\pm$8.04 & \cellcolor{mycolor!72} 93.76$\pm$0.49 & 49.85$\pm$1.07 & 40.63$\pm$2.51 & 32.57$\pm$1.45 \\

& OrderedGNN & \cellcolor{mycolor!54} 53.62$\pm$2.90 & \cellcolor{mycolor!54} 82.29$\pm$2.41 & \cellcolor{mycolor!72} 75.88$\pm$8.66 & \cellcolor{mycolor!18} 93.61$\pm$0.38 & \cellcolor{mycolor!36} 67.14$\pm$1.16 & 40.36$\pm$3.40 & 34.61$\pm$2.17 \\

\hline
\textbf{Our} & \textbf{Best} & GCN-MAS & VCR-MAS & NAG-MAS & CR-TAS & GPS-TAS & GPS-TAS & CR-TAS \\
\textbf{Validated} & \textbf{Model} & \cellcolor{mycolor!36} 53.54$\pm$1.59 & \cellcolor{mycolor!72} 82.34$\pm$1.54 & \cellcolor{mycolor!90} 76.97$\pm$5.68 & \cellcolor{mycolor!90} 93.76$\pm$0.38 & 64.71$\pm$1.71 & \cellcolor{mycolor!90} 44.13$\pm$4.56 & 37.16$\pm$1.70 \\

\hline
\end{tabular}
\label{new-main-accuracy}
\end{table}

\begin{table}[ht!]
\centering
\tiny
\begin{minipage}[t]{0.505\textwidth}
\setlength{\tabcolsep}{0.41pt}
\caption{Mean macro-F1 score across 5 long-range datasets. ``Our Best'' indicates the best validation-selected cascade models from Table~\ref{new-accuracy}. Darker cell colors indicate better performance.}
\centering
\begin{tabular}{l c c c c c}
\hline
 & \textbf{\textit{Amazon-ratings}} & \textbf{\textit{COCO-SP}} & \textbf{\textit{PascalVOC-SP}} & \textbf{\textit{Questions}} & \textbf{\textit{Roman-empire}} \\
\hline

GCN
& 0.387 & 0.063 & 0.138 & 0.486 & 0.290 \\

GraphSAGE
& \cellcolor{mycolor!18} 0.402 & 0.062 & 0.090 & 0.490 & 0.570 \\

GAT
& 0.350 & 0.067 & 0.119 & 0.485 & 0.302 \\

\hline
GT
& 0.343 & 0.063 & 0.092 & 0.485 & 0.582 \\

Gophormer
& 0.258 & 0.063 & 0.101 & 0.499 & \cellcolor{mycolor!72} 0.596 \\

SAN
& 0.390 & 0.063 & 0.093 & \cellcolor{mycolor!18} 0.524 & \cellcolor{mycolor!54} 0.593 \\

GraphGPS
& 0.292 & 0.063 & 0.092 & 0.492 & \cellcolor{mycolor!18} 0.583 \\

NAG
& 0.356 & 0.074 & 0.138 & \cellcolor{mycolor!72} 0.533 & 0.444 \\

Exphormer
& 0.286 & 0.077 & 0.146 & 0.483 & 0.564 \\

VCR
& 0.388 & \cellcolor{mycolor!36} 0.102 & \cellcolor{mycolor!54} 0.196 & 0.483 & 0.108 \\

\hline
ACMGCN
& \cellcolor{mycolor!72} 0.429 & \cellcolor{mycolor!72} 0.126 & \cellcolor{mycolor!36} 0.177 & \cellcolor{mycolor!36} 0.526 & 0.574 \\

CMGNN
& 0.376 & 0.065 & 0.115 & 0.494 & \cellcolor{mycolor!90} 0.609 \\

GGCN
& 0.332 & 0.063 & 0.093 & 0.514 & 0.360 \\

GloGNN
& 0.143 & 0.001 & 0.039 & 0.485 & 0.084 \\

H2GCN
& 0.358 & \cellcolor{mycolor!18} 0.078 & \cellcolor{mycolor!18} 0.166 & 0.505 & \cellcolor{mycolor!54} 0.593 \\

M2MGNN
& \cellcolor{mycolor!54} 0.423 & 0.063 & 0.104 & \cellcolor{mycolor!54} 0.532 & 0.396 \\

OrderedGNN
& \cellcolor{mycolor!90} 0.445 & \cellcolor{mycolor!90} 0.144 & \cellcolor{mycolor!72} 0.241 & 0.516 & 0.571 \\

\hline
\textbf{Our Best}
& \cellcolor{mycolor!36} 0.418 & \cellcolor{mycolor!54} 0.124 & \cellcolor{mycolor!90} 0.254 & \cellcolor{mycolor!90} 0.543 & 0.557 \\

\hline
\end{tabular}
\label{new-main-f1}
\end{minipage}
\hfill
\begin{minipage}[t]{0.48\textwidth}
\setlength{\tabcolsep}{2pt}
\caption{Mean AUROC across 3 long-range datasets. ``Our Best'' indicates the best validation-selected cascade models from Table~\ref{new-accuracy}. Darker cell colors indicate better performance.}
\centering
\begin{tabular}{l c c c}
\hline
 & \textbf{\textit{Amazon-ratings}} & \textbf{\textit{Questions}} & \textbf{\textit{Roman-empire}} \\
\hline

GCN & 0.744 & 0.540 & 0.855 \\

GraphSAGE & \cellcolor{mycolor!18} 0.748 & 0.654 & 0.941 \\

GAT & 0.730 & 0.671 & 0.852 \\

\hline
GT & 0.716 & \cellcolor{mycolor!18} 0.687 & 0.940 \\

Gophormer & 0.678 & \cellcolor{mycolor!36} 0.695 & \cellcolor{mycolor!90} 0.953 \\

SAN & 0.739 & 0.663 & \cellcolor{mycolor!72} 0.951 \\

GPS & 0.689 & 0.681 & \cellcolor{mycolor!54} 0.949 \\

NAG & 0.700 & \cellcolor{mycolor!54} 0.702 & 0.912 \\

Exphormer & 0.700 & \cellcolor{mycolor!72} 0.703 & \cellcolor{mycolor!36} 0.947 \\

VCR & 0.731 & 0.597 & 0.722 \\

\hline
ACMGCN & \cellcolor{mycolor!36} 0.759 & 0.642 & 0.936 \\

CMGNN & 0.746 & 0.656 & \cellcolor{mycolor!36} 0.947 \\

GGCN & 0.682 & 0.494 & 0.849 \\

GloGNN & 0.591 & 0.664 & 0.698 \\

H2GCN & 0.733 & 0.626 & 0.943 \\

M2MGNN & \cellcolor{mycolor!54} 0.760 & 0.684 & 0.877 \\

OrderedGNN & \cellcolor{mycolor!90} 0.772 & 0.626 & 0.941 \\

\hline
\textbf{Our Best} & \cellcolor{mycolor!72} 0.771 & \cellcolor{mycolor!90} 0.713 & 0.941 \\

\hline
\end{tabular}
\label{new-main-auroc}
\end{minipage}
\end{table}

As shown in Table \ref{new-main-accuracy}, our best validation-selected method achieves the best performance on three datasets (\textit{PascalVOC-SP}, \textit{Questions}, and \textit{chameleon-filtered}) compared to 17 GNN and GT architectures. On \textit{COCO-SP}, \textit{Amazon-ratings}, \textit{squirrel-filtered}, and \textit{Roman-empire}, our best validation selected method ranks 2nd, 4th, 6th, and 11th, respectively, with moderate gaps of 0.18\%, 2.11\%, 3.73\%, and 4.04\% from the top-performing methods. Overall, our approach ranks within the top six on all 6 target datasets (heterophilic or moderate- to high-degree homophilic), whereas \textit{Roman-empire} belongs to the bottlenecked/regular graph regime where strong performance is not expected by our theory. Our strong performance is followed by OrderedGNN, which achieves top-five performance on 5 out of 7 datasets but with slightly weaker results, and then ACMGCN and Gophormer, each achieving top-five performance on 4 datasets. The macro-F1 and AUROC scores reported in Tables \ref{new-main-f1} and \ref{new-main-auroc} again show consistent results with test accuracy in Table \ref{new-main-accuracy}.

\begin{remark}[Scope on Long-Range Benchmarks]
    Long-range graph benchmarks~\citep{dwivedi2022long, platonov2023critical} target a regime where label-relevant signals must propagate across graph bottlenecks --- long chains, narrow cuts, or tree-like structures with limited multi-path connectivity. This is complementary to the regime targeted by \textsc{Graph Cascades}: as established in Theorem~\ref{thm:walklength}, threshold activation certifies $\kappa$ edge-disjoint short paths from the seed, and such redundancy is especially scarce in bottlenecked structures. We therefore expect cascade rewiring to be most effective on long-range datasets whose informative dependencies are carried by cohesive mesoscopic regions (e.g., \textit{PascalVOC-SP} and \textit{COCO-SP}), and less effective on chain-like or grid-like structures (e.g., \textit{Roman-empire}). The results above are consistent with this prediction: cascades improve macro-F1 substantially on \textit{PascalVOC-SP} and \textit{COCO-SP}, remain competitive on \textit{Amazon-ratings} and \textit{Questions}, and show inconsistent gains on \textit{Roman-empire}, where the method's structural assumptions are not satisfied. We report these results not as uniform improvements but as a calibration of the regimes where the mechanism applies. 
\end{remark}

\subsection{Broad Backbone Evaluations}\label{10-backbones}

To broaden the within-backbone evaluation, we ran six additional compatible backbones (GAT, GT, Gophormer, ACMGCN, H2GCN, and GloGNN) on target regimes (heterophilic and moderate- to high-degree homophilic) under the same protocol for cascade models, bringing the total from four to ten (in addition to GCN, GraphGPS, NAGphormer, and VCR-Graphormer in Section \ref{sec:empirical}). We selected models that can use the cascade weights $W^\ast$ directly without architectural changes. Models that cannot consume edge weights are still within scope: Appendix \ref{weights} shows that the unweighted graph $G^\ast$ can also improve sparse-attention GTs, and \textsc{CR-Graphormer} is designed for this setting. However, these models do not permit the same controlled weighted comparison.

\begin{table}[ht!]
\centering
\small
\caption{Number of selected baselines (out of 10: GCN, GraphGPS, NAGphormer, VCR-Graphormer, GAT, GT, Gophormer, ACMGCN, H2GCN, and GloGNN) where \textsc{Graph Cascades} improves test accuracy per dataset.}
\setlength{\tabcolsep}{4pt}
\begin{tabular}{lcccccc}
\toprule
& \multicolumn{3}{c}{\textbf{MAS}} & \multicolumn{3}{c}{\textbf{TAS}} \\
\cmidrule(lr){2-4} \cmidrule(lr){5-7}
\textbf{Dataset} & \shortstack{\textbf{\# Models} \\ \textbf{Improved}} & \shortstack{\textbf{Average} \\ \textbf{Change}} & \shortstack{\textbf{Max} \\ \textbf{Change}} & \shortstack{\textbf{\# Models} \\ \textbf{Improved}} & \shortstack{\textbf{Average} \\ \textbf{Change}} & \shortstack{\textbf{Max} \\ \textbf{Change}} \\
\midrule
\textit{actor}              & 9 & +2.44\% & +9.77\% & 9 & +2.53\% & +9.87\%  \\
\textit{cornell}            & 9 & +7.62\% & +28.96\% & 8 & +7.95\% & +30.13\% \\
\textit{wisconsin}          & 7 & +7.97\% & +28.60\% & 8 & +8.83\% & +29.67\% \\
\textit{chameleon-filtered} & 8 & +1.92\% & +12.89\% & 7 & +1.62\% & +9.30\%  \\
\textit{COCO-SP}            & 7 & +0.27\% & +0.94\% & 8 & +0.37\% & +1.59\%  \\
\textit{squirrel-filtered}  & 6 & +0.87\% & +3.73\%  & 8 & +0.55\% & +3.34\%  \\
\textit{Questions}          & 6 & +0.64\% & +3.71\% & 6 & +0.78\% & +3.69\%  \\
\textit{shape}              & 5 & -2.34\% & +25.14\% & 6 & +9.96\% & +39.34\% \\
\textit{texas}              & 4 & +3.01\% & +26.31\% & 5 & +3.98\% & +27.63\% \\
\bottomrule
\end{tabular}
\label{tab:cascade_summary}
\end{table}

Table \ref{tab:cascade_summary} shows that, across 10 backbones, both cascade operators—MAS and TAS—demonstrate robust and consistent performance gains on the majority of evaluated datasets, notably on \textit{actor, cornell, wisconsin, chameleon-filtered, COCO-SP,} and \textit{squirrel-filtered}. The gains are most pronounced on \textit{cornell} (up to $+7.95\%$ on average and $+30.13\%$ in individual gain), \textit{wisconsin} (up to $+8.83\%$ on average and $+29.67\%$ in individual gain), \textit{shape} (up to $+9.96\%$ on average and $+39.34\%$ in individual gain), and \textit{texas} (up to $+3.98\%$ on average and $+27.63\%$ in individual gain). Although TAS outperforms MAS in most of the evaluations, the magnitude of performance gains remains consistent per dataset across both contagion operators.

%\newpage

\section*{NeurIPS Paper Checklist}

\begin{enumerate}

\item {\bf Claims}
    \item[] Question: Do the main claims made in the abstract and introduction accurately reflect the paper's contributions and scope?
    \item[] Answer: \answerYes{} % Replace by \answerYes{}, \answerNo{}, or \answerNA{}.
    \item[] Justification: All claims made in the abstract and introduction are justified with theoretical results and empirical experiments. 

\item {\bf Limitations}
    \item[] Question: Does the paper discuss the limitations of the work performed by the authors?
    \item[] Answer: \answerYes{} % Replace by \answerYes{}, \answerNo{}, or \answerNA{}.
    \item[] Justification: A paragraph towards the end of main text discusses the limitations of our approach. 

\item {\bf Theory assumptions and proofs}
    \item[] Question: For each theoretical result, does the paper provide the full set of assumptions and a complete (and correct) proof?
    \item[] Answer: \answerYes{} % Replace by \answerYes{}, \answerNo{}, or \answerNA{}.
    \item[] Justification: All the necessary assumptions and proofs are provided. 

    \item {\bf Experimental result reproducibility}
    \item[] Question: Does the paper fully disclose all the information needed to reproduce the main experimental results of the paper to the extent that it affects the main claims and/or conclusions of the paper (regardless of whether the code and data are provided or not)?
    \item[] Answer: \answerYes{} % Replace by \answerYes{}, \answerNo{}, or \answerNA{}.
    \item[] Justification: The paper discloses the training, validations and test splits together with hyperparameters used with a detailed ablation study. 

\item {\bf Open access to data and code}
    \item[] Question: Does the paper provide open access to the data and code, with sufficient instructions to faithfully reproduce the main experimental results, as described in supplemental material?
    \item[] Answer: \answerYes{} % Replace by \answerYes{}, \answerNo{}, or \answerNA{}.
    \item[] Justification: We provide a  GitHub URL containing the source code. All datasets are publicly available. 

\item {\bf Experimental setting/details}
    \item[] Question: Does the paper specify all the training and test details (e.g., data splits, hyperparameters, how they were chosen, type of optimizer) necessary to understand the results?
    \item[] Answer: \answerYes{} % Replace by \answerYes{}, \answerNo{}, or \answerNA{}.
    \item[] Justification: All the necessary details together with the code(url) is provided in the paper. 

\item {\bf Experiment statistical significance}
    \item[] Question: Does the paper report error bars suitably and correctly defined or other appropriate information about the statistical significance of the experiments?
    \item[] Answer: \answerYes{} % Replace by \answerYes{}, \answerNo{}, or \answerNA{}.
    \item[] Justification: We report standard deviations and conduct regression analysis (Table \ref{tab:correlation}) to show that our rewiring indeed increases homophily in networks. 

\item {\bf Experiments compute resources}
    \item[] Question: For each experiment, does the paper provide sufficient information on the computer resources (type of compute workers, memory, time of execution) needed to reproduce the experiments?
    \item[] Answer: \answerYes{} % Replace by \answerYes{}, \answerNo{}, or \answerNA{}.
    \item[] Justification: In the appendix, we provide sufficient information about compute resources and hyperparameters. 
    
\item {\bf Code of ethics}
    \item[] Question: Does the research conducted in the paper conform, in every respect, with the NeurIPS Code of Ethics \url{https://neurips.cc/public/EthicsGuidelines}?
    \item[] Answer: \answerYes{} % Replace by \answerYes{}, \answerNo{}, or \answerNA{}.
    \item[] Justification: The research conducted in the paper  conforms to the NeurIPS Code of Ethics. 

\item {\bf Broader impacts}
    \item[] Question: Does the paper discuss both potential positive societal impacts and negative societal impacts of the work performed?
    \item[] Answer: \answerNo{} % Replace by \answerYes{}, \answerNo{}, or \answerNA{}.
    \item[] Justification: The paper studies a rewiring architecture in Graph Transformers which can be used as a black box in many applications. To the best of our knowledge we are not aware of direct societal harms of our work.  
    
\item {\bf Safeguards}
    \item[] Question: Does the paper describe safeguards that have been put in place for responsible release of data or models that have a high risk for misuse (e.g., pre-trained language models, image generators, or scraped datasets)?
    \item[] Answer: \answerNA{} % Replace by \answerYes{}, \answerNo{}, or \answerNA{}.
    \item[] Justification: To the best of our knowledge, the paper poses no such risks. 

\item {\bf Licenses for existing assets}
    \item[] Question: Are the creators or original owners of assets (e.g., code, data, models), used in the paper, properly credited and are the license and terms of use explicitly mentioned and properly respected?
    \item[] Answer: \answerYes{} % Replace by \answerYes{}, \answerNo{}, or \answerNA{}.
    \item[] Justification: We cite the original paper or library that produced the datasets or code. 

\item {\bf New assets}
    \item[] Question: Are new assets introduced in the paper well documented and is the documentation provided alongside the assets?
    \item[] Answer: \answerYes{} % Replace by \answerYes{}, \answerNo{}, or \answerNA{}.
    \item[] Justification: We release our code using a GitHub URL for replication purposes. 

\item {\bf Crowdsourcing and research with human subjects}
    \item[] Question: For crowdsourcing experiments and research with human subjects, does the paper include the full text of instructions given to participants and screenshots, if applicable, as well as details about compensation (if any)? 
    \item[] Answer: \answerNA{} % Replace by \answerYes{}, \answerNo{}, or \answerNA{}.
    \item[] Justification: The paper does not involve crowdsourcing or research with human subjects.
    \item[] Guidelines:
    \begin{itemize}
        \item The answer \answerNA{} means that the paper does not involve crowdsourcing or research with human subjects.
        \item Including this information in the supplemental material is fine, but if the main contribution of the paper involves human subjects, then as much detail as possible should be included in the main paper. 
        \item According to the NeurIPS Code of Ethics, workers involved in data collection, curation, or other labor should be paid at least the minimum wage in the country of the data collector. 
    \end{itemize}

\item {\bf Institutional review board (IRB) approvals or equivalent for research with human subjects}
    \item[] Question: Does the paper describe potential risks incurred by study participants, whether such risks were disclosed to the subjects, and whether Institutional Review Board (IRB) approvals (or an equivalent approval/review based on the requirements of your country or institution) were obtained?
    \item[] Answer: \answerNA{} % Replace by \answerYes{}, \answerNo{}, or \answerNA{}.
    \item[] Justification: The paper does not involve crowdsourcing nor research with human subjects.

\item {\bf Declaration of LLM usage}
    \item[] Question: Does the paper describe the usage of LLMs if it is an important, original, or non-standard component of the core methods in this research? Note that if the LLM is used only for writing, editing, or formatting purposes and does \emph{not} impact the core methodology, scientific rigor, or originality of the research, declaration is not required.
    %this research? 
    \item[] Answer: \answerNA{} % Replace by \answerYes{}, \answerNo{}, or \answerNA{}.
    \item[] Justification: LLMs are used only for writing and grammar checking purposes.

\end{enumerate}

\end{document}